\documentclass{article}

\usepackage{iclr2027_conference,times}

\usepackage[utf8]{inputenc}
\usepackage[T1]{fontenc}
\usepackage{amsmath,amsthm,amssymb,mathtools}
\usepackage{amsfonts}
\usepackage{algorithm}
\usepackage{algpseudocode}
\usepackage{booktabs}
\usepackage{enumitem}
\usepackage{graphicx}
\usepackage{nicefrac}
\usepackage{microtype}
\usepackage{xcolor}
\usepackage{hyperref}
\usepackage{url}

\newcommand{\Irr}{\mathrm{Irr}}
\DeclareMathOperator{\range}{range}

\newtheorem{theorem}{Theorem}
\newtheorem{lemma}{Lemma}
\newtheorem*{remark}{Remark}
\newtheorem{example}[theorem]{Example}
\newtheorem{assumption}{Assumption}

\theoremstyle{plain}
\newtheorem{proposition}[theorem]{Proposition}
\newtheorem{corollary}[theorem]{Corollary}
\theoremstyle{definition}
\newtheorem{definition}[theorem]{Definition}
\theoremstyle{remark}

\DeclareMathOperator{\Top}{Top}
\newcommand{\Vol}{\mathrm{Vol}}

\newcommand{\Supp}{\mathrm{Supp}}
\DeclareMathOperator{\Null}{null}

\makeatletter
\def\@oddhead{%
  \vbox{%
    \hbox to \textwidth{ \small \hfill  \hfill}%
    \vskip 4pt % Adds space between the text and the line
    \hrule height 0.5pt width \textwidth % Draws the horizontal line
  }%
}
\def\@evenhead{\@oddhead}
\makeatother

\title{DePICT: Decision-Preserving Interface for Constrained Downstream Tasks}

\author{
\textbf{Utkarsh Grover}\textsuperscript{1}
\quad
\textbf{Ravi Ranjan}\textsuperscript{2}
\quad
\textbf{Agoritsa Polyzou}\textsuperscript{2}
\quad
\textbf{Wyatt Mackey}\textsuperscript{3}
\\
\textbf{J. Morris Chang}\textsuperscript{1}
\quad
\textbf{Leonardo Bobadilla}\textsuperscript{2}
\quad
\textbf{Xiaomin Lin}\textsuperscript{1}
\\[0.6em]
\small
\textsuperscript{1}University of South Florida
\qquad
\textsuperscript{2}Florida International University
\\
\textsuperscript{3}DEVCOM Army Research Lab
\\[0.4em]
\footnotesize
\texttt{\{utkarshgrover, chang5, xlin2\}@usf.edu}
\\
\texttt{\{rkuma031, apolyzou, jabobadi\}@fiu.edu}
\\
\texttt{wyatt.t.mackey.civ@army.mil}
}

\iclrfinalcopy

\begin{document}

\maketitle

\begin{abstract}

%We study the fundamental question of which primitive context variables are necessary for solving a constrained downstream task.  In our setting, a high-dimensional primitive context parameterizes a constrained optimization problem, while a downstream agent observes only a selected interface. For locally regular constrained programs, we give a sharp geometric characterization of the interface directions that preserve the local optimum: a context direction is relevant exactly when perturbing 
A constrained optimization problem may involve a parameter in its objective and in its active constraint, and yet the final decision may be insensitive to small changes in that parameter. This raises a fundamental question: \textit{which inputs does a decision-making system truly depend on?} Building on this question, we introduce \textsc{DePICT}, a procedure that constructs decision-preserving interfaces by ranking context directions with the optimizer's solution-sensitivity and aggregating them across an operating regime. We study this problem in a high-dimensional setting where primitive context parameterizes a constrained task and the downstream agent observes only a selected subset of context directions. For locally regular constrained programs, we derive a Karush–Kuhn–Tucker (KKT) based characterization of when a context direction/ parameter is optimizer-relevant. Our analysis shows that appearing in the active optimization problem does not necessarily mean that a variable affects the final decision. Some context directions can change the KKT conditions while leaving the optimal solution unchanged because their effect is absorbed by the dual variables. DePICT is designed to remove exactly these directions. In a controlled diagnosis, it recovers the decision-relevant interface exactly and reduces linear-predictor regret to $0.009$, compared with $0.475$ for the strongest competing baseline. 
%It also performs well on controlled KKT-regime and real solver-mediated tasks.
\end{abstract}
\vspace{-0.2cm}
\section{Introduction}
\label{sec:introduction}
\vspace{-0.2cm}
%Constrained decision systems routinely receive far more context than their downstream optimizer needs. A routing \cite{kwon2020pomo,wu2021learning}, planning \cite{schrittwieser2020mastering,hafner2023mastering}, or allocation \cite{zhang2020learning} agent observes a high-dimensional state, yet only a few context directions can change the optimal feasible decision. Discarding the rest

Adding sensing, communication, and processing capacity increases the cost of a decision-making system, yet not every observed input can change its downstream decision. Routing systems may consume traffic, weather, terrain, and demand signals \citep{kwon2020pomo,wu2021learning}; planning agents may
combine many state variables and sensing channels \citep{schrittwieser2020mastering,hafner2023mastering}; and allocation systems may process hundreds of contextual features \citep{zhang2020learning}. These inputs are not free: they may need to be sensed, transmitted, stored, maintained, licensed, and repeatedly processed throughout deployment
\citep{verma2024machine,lan2022progressive,staab2024principle,longpre2024large}. Yet a constrained downstream optimizer may be insensitive to many of them: changing those inputs does not change the optimal feasible decision. Carrying such context through the entire pipeline adds cost without changing what the system ultimately does.

The challenge is deciding which inputs can be removed without changing the decision. This is not the same as asking which variables help prediction  \citep{tibshirani1996lasso,peng2005mrmr,sharma2023pfmm}. 
A feature may strongly predict an intermediate quantity, correlate with useful signals  \citep{peng2005mrmr,ren2020prediction,sharma2023pfmm}, 
or even appear explicitly in the optimization problem while leaving the optimum unchanged \citep{fiacco1976sensitivity,bonnans2000perturbation}. 
Conversely, removing a variable that does move the optimum can destroy information that no downstream learner can reconstruct. The question is therefore more basic than feature importance:  \emph{which primitive context directions can actually change the constrained optimum?}

This motivates \textbf{DePICT}, which constructs a decision-preserving interface directly from the downstream objective and constraints. DePICT evaluates local solution sensitivities at representative contexts, ranks primitive directions by how strongly they move the optimum, and aggregates these sensitivities across an operating regime. It requires neither labeled context decision pairs nor a pretrained predictor; instead, it assumes access to a differentiable parameterization of the downstream objective and constraints, as formalized in Section \ref{sec:method}.  The resulting interface is fixed before the downstream predictor is trained. Consequently, DePICT identifies context required by the optimizer rather than
context preferred by a particular predictor architecture.

This perspective separates DePICT from predictive feature selection and
decision-focused learning. Feature-selection methods ask which variables are useful for prediction
\citep{tibshirani1996lasso,peng2005mrmr,sharma2023pfmm}, while Decision Focused Learning asks how predictive models should be trained to produce high-quality downstream decisions
\citep{elmachtoub2022spo,mandi2022rank,wang2021mdpfeatures}. DePICT instead asks a prior structural question: which primitive context directions can change the constrained optimum at all? Extended related work is provided in Appendix~\ref{app:related_work}. 
%It therefore determines what information should be retained before predictor training, complementing methods that subsequently learn how to use that information.

Together, these results turn interface selection from an empirical
preprocessing choice into a property of the downstream optimization problem. Our main contribution is an end-to-end treatment of decision-preserving interface design for constrained downstream tasks.
\textbf{(i) Definition.} A primitive context direction is \emph{irreducible} if perturbing it moves the constrained optimum to first order. \textbf{(ii) Characterization.} We give an exact test for irreducibility via KKT sensitivity. It shows that a symbolically present direction can still be irrelevant, either because it never reaches the active KKT system (Layer~1) or because the multipliers absorb its effect (Layer~2). \textbf{(iii) Consequence.} Removing an irreducible direction induces a first-order bias that no differentiable predictor on the reduced interface can remove. \textbf{(iv) Stability.} Interface rankings are stable across the operating regime, with a certified radius within which the top-$m$ interface is unchanged. \textbf{(v) Validation.} We test each result on matched benchmarks that satisfy our assumptions (controlled benchmarks and the Layer-2 diagnostic of Section~\ref{sec:separation}) and on benchmarks that violate them.

%\begin{figure*}[t]
%    \centering
%    \includegraphics[width=0.95\linewidth]{Images/DEPICT.pdf}
%    %\vspace{-0.3cm}
%    \caption{DePICT constructs the information interface from the downstream decision.}
%    \label{fig:Main}
    %\vspace{-0.4cm}
%\end{figure*}

%\input{Sections/Related_Work}
\vspace{-0.2cm}
\section{Decision-Preserving Interfaces}
\label{sec:method}
\vspace{-0.1cm}
\subsection{Problem Formulation and Decision Relevance}
\label{sec:problem}
\vspace{-0.1cm}
 We now formalize the question: \textit{which primitive context directions can actually change a constrained decision?}
Let $c\in\mathbb{R}^{n_1}$ denote the primitive context and
$y(c)\in\mathbb{R}^{n_y}$ the task parameters induced by that context.
We write $ \theta^\star(c):=S(y(c)),$ where the downstream solver $S$ returns a local minimizer of
\begin{equation}
    \min_{\theta\in\mathbb{R}^{n_\theta}}
    f(\theta;y(c))
    \quad
    \mathrm{s.t.}
    \quad
    m_i(\theta;y(c))\geq0,
    \qquad i=1,\ldots,K .
    \label{eq:main_problem}
\end{equation}
%Thus, $c\mapsto y(c)\mapsto\theta^\star(c)$ is the oracle decision map introduced above. 
For compactness, we write
$f(\theta;c)$ and $m_i(\theta;c)$ below, absorbing the dependence
$c\mapsto y(c)$ into the notation.
Our theoretical analysis is local. We consider an operating regime
$\mathcal M\subseteq\mathbb{R}^{n_1}$ on which the objective and constraint structure are fixed and the same active set $\mathcal A$ remains locally binding throughout the regime. We make the following standard regularity assumption.

\begin{assumption}[Regular operating regime]
\label{assump:regularity}
For every $c\in\mathcal M$, the selected primal-dual solution
$(\theta^\star(c),\lambda^\star(c))$ satisfies LICQ, second-order sufficient
conditions (SOSC), and strict complementarity (SC), and the relevant objective and
constraint functions are twice continuously differentiable.
\end{assumption}

These conditions ensure that the local primal--dual solution branch is
well defined and differentiable with respect to $c$
\citep{fiacco1976sensitivity,robinson1980strongly,bonnans2000perturbation}.
They also make the active-set KKT system below nonsingular. Appendix~\ref{app:jacobian-derivation} states the corresponding KKT residual formally and gives the implicit-function derivation used throughout the paper. This differentiability gives a direct notion of decision relevance. For a
small perturbation $\delta$,
\begin{equation}
    \theta^\star(c+\delta)
    =
    \theta^\star(c)
    +
    J(c)\delta
    +
    o(\|\delta\|),
    \qquad
    J(c):=
    \frac{\partial\theta^\star(c)}{\partial c}.
    \label{eq:solution_jacobian}
\end{equation}
The $j$th column $
    v_j(c):=\operatorname{col}_j(J(c))
$ is the first-order motion of the optimal decision caused by perturbing
primitive coordinate $c_j$. We call $j$ \emph{irreducible at $c$} when
$v_j(c)\neq0$. Thus, a context direction is decision-relevant precisely when
changing it moves the local optimum to first order.

Because relevance can vary with $c$, the corresponding regime-level interface
is
\begin{equation}
    \Irr(\mathcal M)
    :=
    \bigcup_{c\in\mathcal M}
    \left\{
        j\in\{1,\ldots,n_1\}:v_j(c)\neq0
    \right\}.
    \label{eq:irr_set}
\end{equation}
This is the first-order decision-preserving interface: every coordinate in
$\Irr(\mathcal M)$ changes the optimum somewhere in the operating regime. DePICT approximates this regime-level object from representative contexts in $\mathcal M$; those contexts sample the regime but do not define it.

%DePICT approximates this regime-level object from representative contexts in $\mathcal M$; those contexts sample the regime but do not define it.

\vspace{-0.1cm}
\subsection{Deciding Relevance}
\vspace{-0.1cm}
Equations~\eqref{eq:solution_jacobian}-\eqref{eq:irr_set} define relevance
through $J(c)$ but do not say how to obtain it from
\eqref{eq:main_problem} or when a column $v_j(c)$ vanishes. Both are decided
by the active KKT system. We therefore define the active-set Lagrangian
$L(\theta,\lambda;c):=f(\theta;c)-\sum_{i\in\mathcal A}\lambda_i m_i(\theta;c)$,
where the sign follows from our convention $m_i(\theta;c)\geq0$, and at
$(\theta^\star(c),\lambda^\star(c),c)$ define
\begin{align}
A_{\mathcal A}
&:=
\bigl[\nabla_\theta m_i(\theta^\star;c)^\top\bigr]_{i\in\mathcal A},
&
H_L
&:=
\nabla_{\theta\theta}^2
L(\theta^\star,\lambda^\star;c),
\label{eq:kkt_blocks_main_1}\\
C_{\mathcal A}
&:=
\bigl[\nabla_c m_i(\theta^\star;c)^\top\bigr]_{i\in\mathcal A},
&
G_{\mathcal A}
&:=
\sum_{i\in\mathcal A}
\lambda_i^\star
\nabla_{\theta c}^2m_i(\theta^\star;c)
-
\nabla_{\theta c}^2f(\theta^\star;c).
\label{eq:kkt_blocks_main_2}
\end{align}
The brackets stack rows indexed by $i\in\mathcal A$, so that
$A_{\mathcal A}\in\mathbb{R}^{|\mathcal A|\times n_\theta}$ and
$C_{\mathcal A}\in\mathbb{R}^{|\mathcal A|\times n_1}$, while
$G_{\mathcal A}\in\mathbb{R}^{n_\theta\times n_1}$.
Context enters through exactly two channels: $C_{\mathcal A}$ records how it
shifts the active margins, and $G_{\mathcal A}$ records how it perturbs
stationarity. Together with the active constraint-normal space
$\operatorname{range}(A_{\mathcal A}(c)^\top)$, these decide relevance
completely.

\begin{theorem}[KKT characterization of optimizer relevance]
\label{thm:kkt-characterization}
Under Assumption~\ref{assump:regularity}, for every
$c\in\mathcal M$ and primitive coordinate $j$,
\begin{equation}
v_j(c)=0
\quad\Longleftrightarrow\quad
(C_{\mathcal A}(c))_{\cdot j}=0
\;\;\text{and}\;\;
(G_{\mathcal A}(c))_{\cdot j}
\in
\operatorname{range}\!\left(A_{\mathcal A}(c)^\top\right).
\label{eq:annihilation_main}
\end{equation}
\end{theorem}

A coordinate is thus irrelevant exactly when it moves no active margin and
its stationarity effect lies along the constraint normals, where multiplier
adjustment absorbs it without primal motion; appearing in the active KKT
system is necessary for relevance but not sufficient.

The theorem follows from the system that computes $J(c)$. Differentiating
the stationarity and active-constraint equations with respect to $c$ gives
\begin{equation}
\underbrace{
\begin{bmatrix}
H_L(c) & -A_{\mathcal A}(c)^\top\\
A_{\mathcal A}(c) & 0
\end{bmatrix}}_{K_{\mathcal A}(c)}
\begin{bmatrix}
J(c)\\
\Lambda_c(c)
\end{bmatrix}
=
\begin{bmatrix}
G_{\mathcal A}(c)\\
-C_{\mathcal A}(c)
\end{bmatrix},
\label{eq:policy_jacobian_saddle}
\end{equation}
where $\Lambda_c(c):=\partial\lambda^\star(c)/\partial c$.
Under Assumption~\ref{assump:regularity}, $K_{\mathcal A}(c)$ is nonsingular,
so the upper block of the unique solution is exactly the solution-sensitivity
Jacobian $J(c)$; reading off column $j$ of
\eqref{eq:policy_jacobian_saddle} gives \eqref{eq:annihilation_main}
directly (Appendix~\ref{app:separation}). This saddle system is the general
computational form. When $H_L(c)$ and
$S_{\mathcal A}(c):=A_{\mathcal A}(c)H_L(c)^{-1}A_{\mathcal A}(c)^\top$ are
additionally nonsingular, the same solution admits the closed form
\begin{equation}
    \boxed{
    J(c)
    =
    P^\perp(c)G_{\mathcal A}(c)
    -
    \Phi(c)C_{\mathcal A}(c)
    }
    \label{eq:policy_jacobian_formula}
\end{equation}
with
\begin{equation}
    \Phi
    :=
    H_L^{-1}A_{\mathcal A}^\top S_{\mathcal A}^{-1},
    \qquad
    P^\perp
    :=
    H_L^{-1}
    -
    \Phi A_{\mathcal A}H_L^{-1},
    \label{eq:projection_blocks_main}
\end{equation}
which exposes the geometry (Appendix~\ref{app:jacobian-derivation}). The two
terms are the two channels of Theorem~\ref{thm:kkt-characterization}:
$-\Phi C_{\mathcal A}$ is the primal motion that tracks a shifted active
margin, and $P^\perp G_{\mathcal A}$ is the stationarity perturbation the
multipliers cannot absorb. Since $P^\perp(c)u=0$ exactly when
$u\in\operatorname{range}(A_{\mathcal A}(c)^\top)$, the range condition in
\eqref{eq:annihilation_main} reads $P^\perp(c)(G_{\mathcal A}(c))_{\cdot j}=0$
under these conditions. A coordinate may therefore enter the active problem
through $C_{\mathcal A}$ or $G_{\mathcal A}$ and still satisfy $v_j(c)=0$.
\vspace{-0.1cm}
\subsection{Symbolic Support vs.\ Optimizer Relevance}
\label{sec:separation}
\vspace{-0.1cm}
%Equation~\eqref{eq:policy_jacobian_formula} shows that context can enter the local KKT system without necessarily moving the primal solution. A coordinate may perturb an active constraint or the stationarity equation, yet its effect can still be absorbed by the active-set geometry. 

Theorem~\ref{thm:kkt-characterization} decides relevance at the active KKT
level. We now show that symbolic participation, whether in the formulation or
in the active KKT system itself, can overstate it.

\begin{definition}[Symbolic supports]
\label{def:supp}
The \emph{full symbolic support} is
\[
\Supp_{\mathrm{full}}
:=
\{j : c_j \text{ appears symbolically in } f \text{ or some } m_i\}.
\]
The \emph{symbolic active-set support} over an operating regime
$\mathcal M$ is
\begin{equation}
\Supp(\mathcal M)
:=
\bigl\{
j\in\{1,\ldots,n_1\}:
\exists\,c\in\mathcal M
\text{ such that }
(C_{\mathcal A}(c))_{\cdot j}\neq0
\ \text{or}\
(G_{\mathcal A}(c))_{\cdot j}\neq0
\bigr\}.
\label{eq:supp_main}
\end{equation}
\end{definition}

The first set records whether a coordinate appears anywhere in the formulation;
the second retains only coordinates that reach the active KKT sensitivity
system through an active-margin shift $C_{\mathcal A}$ or a stationarity shift
$G_{\mathcal A}$. By Theorem~\ref{thm:kkt-characterization}, reaching that
system is necessary for relevance but not sufficient, so neither set is
guaranteed to coincide with the decision-preserving interface $\Irr(\mathcal M)$.

\begin{corollary}[Symbolic support upper-bounds optimizer relevance]
\label{cor:symbolic-upper-bound}
Under Assumption~\ref{assump:regularity},
\[
\Irr(\mathcal M)
\subseteq
\Supp(\mathcal M)
\subseteq
\Supp_{\mathrm{full}},
\]
and both inclusions may be strict.
\end{corollary}

Corollary~\ref{cor:symbolic-upper-bound} reveals two distinct ways in which
symbolic inspection can overestimate decision relevance.

\paragraph{Layer 1: outer gap
$\Supp_{\mathrm{full}}\setminus\Supp(\mathcal M)$.}
A coordinate appears in the optimization problem but never reaches the local
sensitivity system:
\[
(C_{\mathcal A})_{\cdot j}
\equiv
(G_{\mathcal A})_{\cdot j}
\equiv0
\qquad\text{on }\mathcal M.
\]
Such a coordinate is symbolically present but already irrelevant at the active KKT level.

\paragraph{Layer 2: inner gap
$\Supp(\mathcal M)\setminus\Irr(\mathcal M)$.}
More subtly, a coordinate can reach the active KKT system and still leave the
optimum unchanged. In this case,
\[
(C_{\mathcal A})_{\cdot j}\equiv0,
\qquad
(G_{\mathcal A})_{\cdot j}
\in\operatorname{range}(A_{\mathcal A}^\top)
\quad\text{on }\mathcal M,
\]
with $(G_{\mathcal A})_{\cdot j}\not\equiv0$ somewhere in the regime.
The coordinate is therefore symbolically active but KKT-annihilated.

\begin{example}[Layer-2 annihilation]
\label{ex:scaling}
Consider a binding budget constraint with a positive calibration factor $c_j$,
\[
m(\theta;c)
=
c_j\,g(\theta,c_{-j}),
\qquad
g(\theta,c_{-j})
:=
B(c_{-j})-w(c_{-j})^\top\theta,
\qquad
c_j>0,
\]
and suppose the objective $f$ is independent of $c_j$.
The coordinate $c_j$ appears explicitly in the constraint,
$\nabla_\theta m$, and the active KKT system. Yet
\[
\{\theta:c_jg(\theta,c_{-j})\ge0\}
=
\{\theta:g(\theta,c_{-j})\ge0\},
\]
so changing $c_j$ does not change the feasible set or the optimum.

The KKT characterization recovers this invariance directly. Because the
constraint is binding, $ g(\theta^\star,c_{-j})=0,$
and hence $(C_{\mathcal A})_{\cdot j}=0$. Moreover,
\[
\frac{\partial^2m}{\partial\theta\,\partial c_j}
=
\nabla_\theta g
=
c_j^{-1}\nabla_\theta m
\in
\operatorname{range}(A_{\mathcal A}^\top).
\]
The stationarity perturbation therefore lies entirely in the active
constraint-normal space. The multiplier correspondingly rescales as
$c_j^{-1}$, absorbing the perturbation without primal motion.
\end{example}

Example~\ref{ex:scaling} gives the simplest instance of the inner gap. The same
mechanism extends beyond a single rescaled constraint.

\begin{proposition}[Sufficient structures for Layer-2 annihilation]
\label{prop:structural-main}
Under Assumption~\ref{assump:regularity}, fix a coordinate $j$ such that
$(C_{\mathcal A})_{\cdot j}\equiv0$ and
$\nabla^2_{\theta c_j}f(\theta^\star(c);c)\equiv0$ on $\mathcal M$.
Then $j\notin\Irr(\mathcal M)$ in either of the following cases:
\begin{enumerate}[label=\textup{(\Roman*)},leftmargin=2em]

    \item \textbf{Constraint scaling.}
    Each active constraint either does not depend on $c_j$, or has the form
    \[
    m_i(\theta;c)
    =
    c_j\,g_i(\theta,c_{-j}),
    \qquad c_j>0 .
    \]

    \item \textbf{Proportional mixed partials.}
    There exist scalars $\beta_i(c)$ such that
    \[
    \frac{\partial^2m_i}
    {\partial\theta\,\partial c_j}
    (\theta^\star(c);c)
    =
    \beta_i(c)\,
    \nabla_\theta m_i(\theta^\star(c);c)
    \qquad
    \forall i\in\mathcal A,\;\forall c\in\mathcal M .
    \]
\end{enumerate}
\end{proposition}

Both cases route the stationarity perturbation entirely through active
constraint normals, so multiplier adjustment absorbs it. Case~\textup{(I)} is
a structured special case of~\textup{(II)}: $\beta_i=0$ for constraints
independent of $c_j$, while $\beta_i=c_j^{-1}$ for scaled constraints.
The single-constraint instance of~\textup{(I)} is the mechanism used in our
controlled Layer-2 diagnostic. Theorem~\ref{thm:structural-app} further
shows that changing constraint \emph{shape}, rather than scale, restores
irreducibility. The inner gap disappears in a boundary case.

\begin{corollary}[Affine active-constraint collapse]
\label{cor:affine-main}
Under Assumption~\ref{assump:regularity}, suppose all active constraints are
affine in $(\theta,c)$ and
\[
\nabla^2_{\theta c}f(\theta^\star(c);c)\equiv0
\qquad\text{on }\mathcal M .
\]
Then $G_{\mathcal A}\equiv0$ and
\[
\Irr(\mathcal M)
=
\Supp(\mathcal M)
=
\bigl\{
j:(C_{\mathcal A})_{\cdot j}\not\equiv0
\bigr\}.
\]
\end{corollary}

Thus, when active constraints are affine and the objective contributes no mixed
context-decision curvature, optimizer relevance reduces to active-margin
dependence and symbolic active-set inspection is sufficient.
A contextual linear objective gives a complementary and familiar boundary.
Consider an LP with fixed constraints and
$f(\theta;c)=y(c)^\top\theta$. At a nondegenerate vertex,
$A_{\mathcal A}$ is square and nonsingular, so
$
\operatorname{range}(A_{\mathcal A}^\top)
=
\mathbb{R}^{n_\theta}.
$
For a context coordinate that perturbs only the cost,
$(C_{\mathcal A})_{\cdot j}=0$ and
$
(G_{\mathcal A})_{\cdot j}
=
-\frac{\partial y(c)}{\partial c_j}
\in
\operatorname{range}(A_{\mathcal A}^\top).
$
Theorem~\ref{thm:kkt-characterization} therefore gives $v_j(c)=0$ while the
active basis remains unchanged. 
%In this sense, the Layer-2 characterization recovers the familiar piecewise-constant solution map of linear programs that motivates surrogate decision-focused losses such as SPO+.
The distinction above tells us when symbolic participation does and does not
translate into optimizer relevance. It leaves a complementary question:
what happens if an interface removes a coordinate that \emph{is} irreducible?
Section~\ref{sec:erm_bias} shows that the resulting first-order error
cannot be repaired by downstream learning. Full statements and proofs for the
separation results are given in Appendix~\ref{app:separation}.
\vspace{-0.1cm}
\subsection{Interface Misspecification and Stability}
\label{sec:erm_bias}
\label{sec:stability}
\vspace{-0.1cm}
Building on the distinction established in Section~\ref{sec:separation}, we ask two questions: what is lost when an irreducible direction is removed, and how stable are relevance judgments across the operating regime?

\paragraph{Removing an irreducible direction.}
Let
$\psi\in C^1(\mathbb{R}^{n_1},\mathbb{R}^{d})$
be an interface map and
$h\in C^1(\mathbb{R}^{d},\mathbb{R}^{n_\theta})$
a downstream decision map.

\begin{theorem}[Annihilation of an irreducible direction]
\label{thm:annihilation-main}
For any $c\in\mathcal M$ and any
$v\in\ker(D\psi(c))$,
\[
D(h\circ\psi)(c)v=0.
\]
Consequently, if primitive direction $j$ is irreducible at $c$,
$J(c)e_j\neq0$, and the interface removes that direction,
$D\psi(c)e_j=0$, then
\begin{equation}
D(h\circ\psi)(c)e_j=0
\qquad\text{while}\qquad
J(c)e_j\neq0.
\label{eq:interface_mismatch}
\end{equation}
Thus no differentiable downstream decision map through $\psi$ can match the
oracle decision map to first order in direction $e_j$ at $c$.
\end{theorem}

In a predict-then-optimize pipeline, this applies to
$h=S\circ g_\phi$ whenever the predictor--solver composition is locally
differentiable.

\begin{corollary}[Structural first-order bias under misspecified interfaces]
\label{cor:erm-main}
Let
\[
\mathcal F_\psi
:=
\left\{
h\circ\psi:
h\in C^1(\mathbb{R}^{d},\mathbb{R}^{n_\theta})
\right\}.
\]
If there exist $c\in\mathcal M$ and $j$ such that
$J(c)e_j\neq0$ and $D\psi(c)e_j=0$, then every
$\widehat{\theta}\in\mathcal F_\psi$ incurs a nonzero first-order mismatch at
$c$. In particular, empirical risk minimization within $\mathcal F_\psi$
cannot remove this local structural bias.
\end{corollary}

Appendix~\ref{app:erm_bias} gives the proofs and a quantitative local
first-order lower bound.

The preceding result makes interface selection consequential, but
irreducibility is defined through the local Jacobian $J(c)$ and may vary across
$\mathcal M$. Practical interface construction therefore also requires
rankings that remain stable under nearby context perturbations and filtering
rules that retain directions whose irreducibility persists across the regime.

\begin{definition}[Column norm profile and ranking gap]
\label{def:col-norm}
For $c\in\mathcal M$, define
\[
\Irr(c):=\{j:v_j(c)\neq0\},
\qquad
\rho(c):=\bigl(\|v_j(c)\|_2\bigr)_{j=1}^{n_1},
\]
where $v_j(c)$ is defined in Section~\ref{sec:problem}.
Let
\[
\rho_{(1)}(c)\geq
\rho_{(2)}(c)\geq\cdots\geq
\rho_{(n_1)}(c)\geq0
\]
denote the ordered column norms, and define the $m$th ranking gap
\[
\gamma_m(c)
:=
\rho_{(m)}(c)-\rho_{(m+1)}(c).
\]
Whenever $\gamma_m(c)>0$, define
\[
\Top_m(c)
:=
\{j:\rho_j(c)\geq\rho_{(m)}(c)\},
\qquad
|\Top_m(c)|=m.
\]
\end{definition}

\begin{proposition}[Stability of Jacobian rankings]
\label{prop:stability-main}
Under Assumption~\ref{assump:regularity}, let
$\mathcal K\subseteq\mathcal M$ be compact.

\begin{enumerate}[label=\textup{(\roman*)},leftmargin=2em]

    \item The map $c\mapsto\rho(c)$ is locally Lipschitz on $\mathcal M$.
    In particular, there exists $L_{\mathcal K}>0$ such that
    \[
    \|\rho(c)-\rho(c')\|_\infty
    \leq
    L_{\mathcal K}\|c-c'\|,
    \qquad
    \forall\,c,c'\in\mathcal K.
    \]
    The constant $L_{\mathcal K}$ depends on smoothness bounds for $f$ and
    $\{m_i\}_{i=1}^K$, the conditioning of the active-set KKT matrix
    $K_{\mathcal A}(c)$, and the geometry of $\mathcal K$; an explicit
    bound is given in Appendix~\ref{app:stability}.

    \item If $c_0\in\mathcal K$ satisfies $\gamma_m(c_0)>0$, then
    \[
    \Top_m(c)=\Top_m(c_0)
    \qquad
    \forall\,c\in\mathcal K
    \text{ with }
    \|c-c_0\|
    <
    \frac{\gamma_m(c_0)}{2L_{\mathcal K}}.
    \]
    Thus the certified local stability radius is controlled by the ratio
    $\gamma_m(c_0)/L_{\mathcal K}$.

    \item For $\alpha\in(0,1]$, suppose
    $0<\Vol(\mathcal K)<\infty$ and define
    \[
    \Irr_\alpha(\mathcal K)
    :=
    \left\{
    j:
    \frac{
    \Vol\{c\in\mathcal K:j\in\Irr(c)\}
    }{
    \Vol(\mathcal K)
    }
    \geq\alpha
    \right\}.
    \]
    Then
    \[
    \Irr_1(\mathcal K)
    \subseteq
    \Irr_\alpha(\mathcal K)
    \subseteq
    \Irr(\mathcal M).
    \]
    The family is monotone in $\alpha$:
    if $\alpha'\geq\alpha$, then
    $\Irr_{\alpha'}(\mathcal K)\subseteq\Irr_\alpha(\mathcal K)$.
\end{enumerate}
\end{proposition}

Proposition~\ref{prop:stability-main} connects local sensitivity to practical
interface construction. A positive ranking gap certifies a neighborhood in
which the top-$m$ ranking is unchanged, while frequency filtering retains
directions whose irreducibility persists across the operating regime.
Additional stability results and proofs are given in
Appendix~\ref{app:stability}. Algorithm~\ref{alg:depict-main} summarizes the resulting pipeline.

\begin{algorithm}[t]
\caption{\textsc{DePICT}: Decision-Preserving Interface Construction
(\textsc{Top-}$m$)}
\label{alg:depict-main}
\small
\begin{algorithmic}[1]

\Require Representative contexts $\{c_t\}_{t=1}^{T}\subset\mathcal M$;
constrained program~\eqref{eq:main_problem};
interface budget $m$
\Ensure Interface $\widehat{\mathcal I}_m\subseteq[n_1]$

\For{$c_t\in\{c_1,\ldots,c_T\}$}

    \State Solve~\eqref{eq:main_problem} at $c_t$ for
    $(\theta^\star(c_t),\lambda^\star(c_t))$;
    \quad
    $\mathcal A_t\gets
    \{i:m_i(\theta^\star(c_t);c_t)=0\}$

    \State Evaluate
    $A_{\mathcal A_t}(c_t)$,
    $H_L(c_t)$,
    $C_{\mathcal A_t}(c_t)$, and
    $G_{\mathcal A_t}(c_t)$
    from Eqs.~\eqref{eq:kkt_blocks_main_1}--\eqref{eq:kkt_blocks_main_2}

    \State
    $S_{\mathcal A_t}(c_t)
    \gets
    A_{\mathcal A_t}(c_t)H_L(c_t)^{-1}A_{\mathcal A_t}(c_t)^\top$;
    \quad
    $\Phi(c_t)
    \gets
    H_L(c_t)^{-1}A_{\mathcal A_t}(c_t)^\top
    S_{\mathcal A_t}(c_t)^{-1}$

    \State
    $P^\perp(c_t)
    \gets
    H_L(c_t)^{-1}
    -\Phi(c_t)A_{\mathcal A_t}(c_t)H_L(c_t)^{-1}$;
    \quad
    $J(c_t)
    \gets
    P^\perp(c_t)G_{\mathcal A_t}(c_t)
    -\Phi(c_t)C_{\mathcal A_t}(c_t)$

    \State
    $\rho(c_t)
    \gets
    \bigl(\|v_j(c_t)\|_2\bigr)_{j=1}^{n_1}$

\EndFor

\State
$\bar\rho
\gets
T^{-1}\sum_{t=1}^{T}\rho(c_t)$;
\quad
$\widehat{\mathcal I}_m
\gets
\operatorname{Top}_m(\bar\rho)$

\State \Return $\widehat{\mathcal I}_m$

\end{algorithmic}
\end{algorithm}

\vspace{-0.2cm}
\section{Experiments}
\label{sec:experiments}
\vspace{-0.2cm}
\begin{figure*}[t]
\centering
\includegraphics[width=\textwidth]{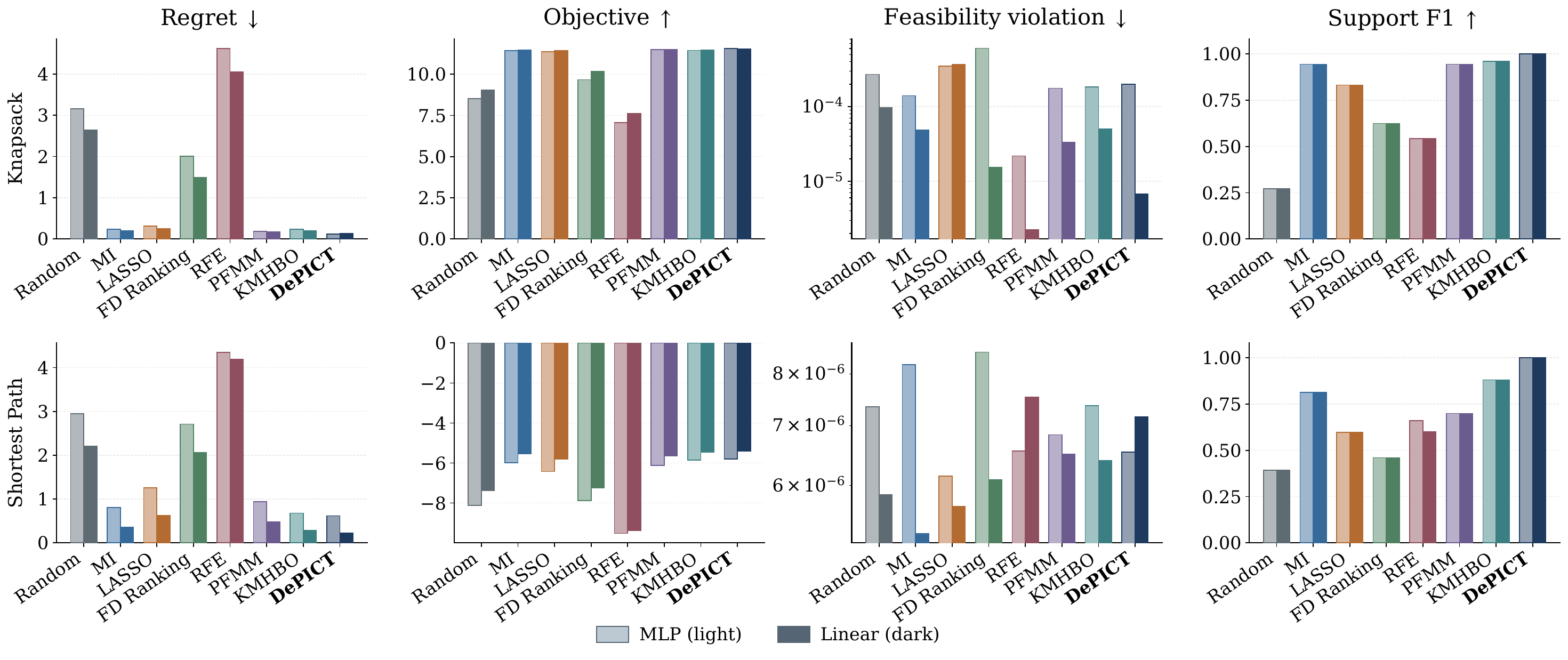}
\caption{
\textbf{Controlled results.}
Knapsack (top row) and shortest path (bottom row) under matched MLP and linear predictors. All the values are mean over 10 seeds.
%Columns report regret, objective value, feasibility violation, and support F1; arrows indicate the preferred direction. Full numerical results are reported in Appendix~Table~\ref{tab:controlled_synthetic_appendix}.
}
\label{fig:controlled_synthetic_main}
\vspace{-0.4cm}
\end{figure*}
%\vspace{-0.1cm}
%\subsection{Experiment design}
%\vspace{-0.1cm}
We evaluate DePICT in five stages. We first test interface recovery and decision quality on controlled benchmarks, then examine transfer to solver-mediated edge-regime tasks. We next isolate the Layer-2 annihilation mechanism predicted by our theory, test interface stability under context perturbations, and finally evaluate runtime scaling with context dimension. Across all benchmarks, interfaces are selected before predictor training and evaluated under matched downstream protocols. Full experimental details are given in Appendix~\ref{app:exp_details}.

\vspace{-0.1cm}
\subsection{Controlled validation}
\label{sec:exp_controlled}
\vspace{-0.1cm}

We use controlled knapsack and shortest-path benchmarks with known parameter-map supports to test whether DePICT recovers the relevant coordinates and preserves decision quality. The knapsack task follows \citet{mandi2020smart}, adapted to include both relevant and nuisance context
directions, while shortest path provides a controlled routing problem on a fixed grid with known cost-relevant coordinates. For shortest path, DePICT scores the interface using a regularized continuous flow formulation rather
than differentiating the discrete argmin, and downstream decisions are evaluated with the same regularized flow problem (Appendix~\ref{app:exp_details}). Controlled knapsack retains the LP boundary structure discussed in Section~\ref{sec:separation}, whereas shortest path uses a small strongly convex quadratic regularizer to obtain a unique, differentiable flow solution for both sensitivity analysis and evaluation.

\vspace{-0.1cm}
\paragraph{Results.}
\vspace{-0.1cm}
Figure~\ref{fig:controlled_synthetic_main} summarizes the controlled regime results. DePICT recovers the oracle interface exactly on both benchmarks (\(F1=1.000\)) and achieves the best regret and objective value among reduced-interface methods. On shortest path, it also obtains the highest path match. The advantage is clearest under the linear predictor, where interface errors are harder to mask; under the nonlinear predictor (MLP), gaps narrow but DePICT remains best or competitive.

\begin{figure*}[t]
\centering
\includegraphics[width=\textwidth]{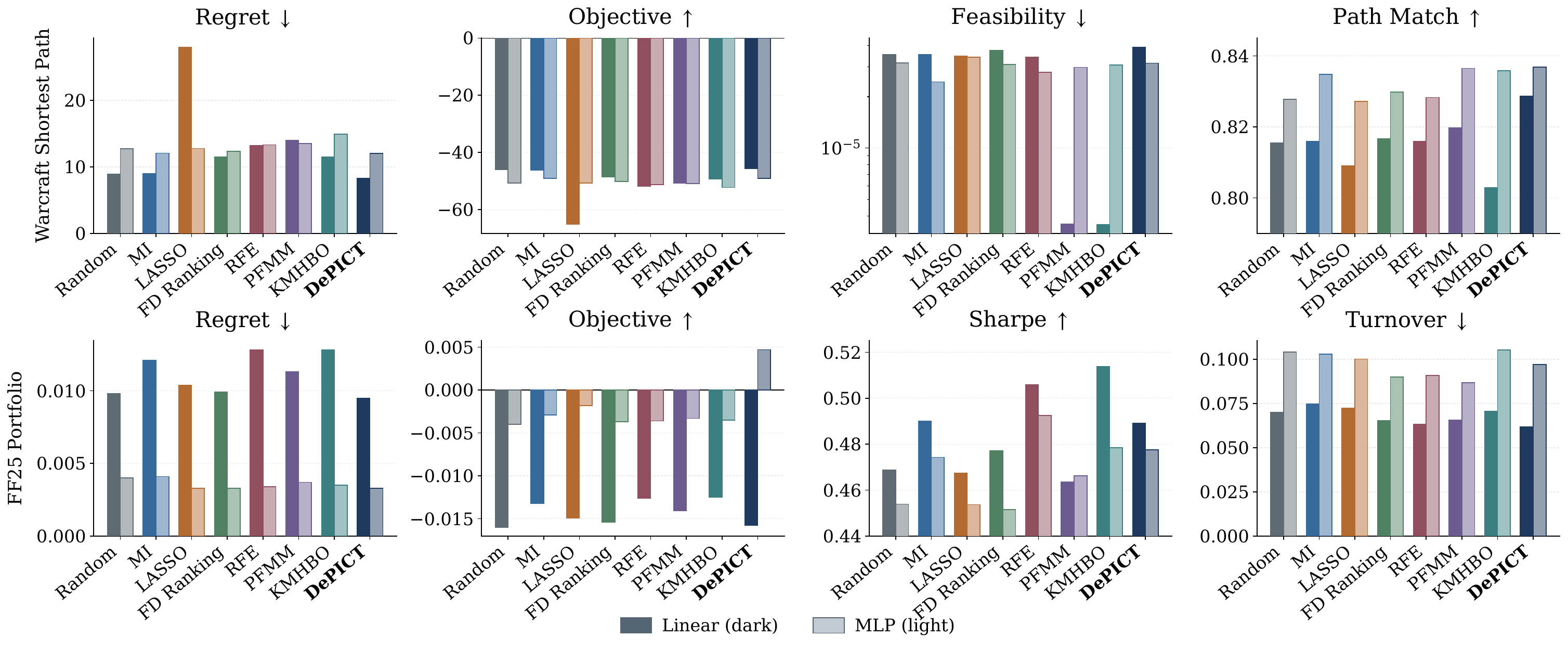}
\vspace{-0.6cm}
\caption{
\textbf{Edge-regime solver-mediated tasks.}
Warcraft shortest path (top row) and FF25 portfolio allocation (bottom row) under matched linear and MLP predictors. All the values are mean over 10 seeds.
%Warcraft, FF25 columns report regret, objective value, feasibility violation, and path match; FF25 columns report common-reference regret, common-reference objective, Sharpe, and turnover. Full numerical results are reported in Appendix~Table~\ref{tab:warcraft_ff25_main}.
}
\label{fig:edge_regime}
\vspace{-0.4cm}
\end{figure*}

\vspace{-0.1cm}
\subsection{Real and Edge-Regime Benchmarks}
\label{sec:exp_edge}
\vspace{-0.1cm}
We use Warcraft and FF25 as edge-regime benchmarks to test whether DePICT
transfers beyond the assumptions of controlled setting. Warcraft is derived from a combinatorial shortest-path benchmark~\citep{vlastelica2020blackbox}, but we use a strongly convex regularized flow QP for both interface scoring and downstream evaluation. FF25 evaluates constrained portfolio allocation on Fama–French panels ~\citep{frenchdatalibrary}, where no oracle support is available. We therefore treat performance on Warcraft and FF25 as evidence of empirical transfer, not as a formal guarantee for nonsmooth or discrete decision maps.

\vspace{-0.1cm}
\paragraph{Results.}
\vspace{-0.1cm}
Figure~\ref{fig:edge_regime} summarizes edge-regime
performance. On Warcraft, DePICT achieves the lowest regret under both predictor families, with a larger margin under the linear predictor, where interface errors are harder to mask. Under the nonlinear predictor (MLP), selector gaps nearly collapse, consistent with nonlinear compensation among selected variables. This pattern is consistent with the first-order theory: Theorem~\ref{thm:annihilation-main} identifies a
local first-order obstruction, while higher-capacity predictors may partially reconstruct missing directions through higher-order correlations among retained variables. Characterizing the precise conditions under which such compensation fails is an important direction; the present results indicate that interface quality is most visible when predictor capacity is limited relative to context dimension. On FF25, DePICT combines the lowest common-reference regret and turnover under the linear predictor with competitive Sharpe, preserving portfolio quality while reducing trading variation.

\vspace{-0.1cm}
\subsection{Layer-2 annihilation diagnostic.}
\vspace{-0.1cm}
The Layer-2 diagnostic tests the central distinction predicted by our theory: can DePICT distinguish context variables that enter the active KKT system from
those that actually move the primal optimum? Following Proposition~\ref{prop:structural-main}\textup{(I)}, these variables enter the active constraints as positive scaling factors, allowing their effect to be absorbed by the multipliers without moving the primal solution.

Figure~\ref{fig:knapsack_layer2_main} first verifies this mechanism on
controlled knapsack, where the irreducible interface is known exactly.
DePICT achieves perfect irreducible recall and support F1 while assigning zero false inclusion to both Layer-1 and Layer-2 nuisance variables. Competing methods retain substantially more Layer-2 variables, including mutual information (FPR \(0.700\)), FD ranking and PFMM (\(0.500\)), and KMHBO (\(0.650\)). This separation translates to downstream performance: DePICT reduces linear-predictor regret to \(0.009\), compared with \(0.475\) for the strongest non-DePICT baseline. The FF25 diagnostic shows the same separation in portfolio allocation. DePICT again assigns zero false inclusion to both nuisance classes, while Layer-2 FPR ranges from \(0.100\) for RFE to \(1.000\) for mutual information and PFMM. DePICT also achieves the lowest regret (\(0.0054\)) and turnover (\(0.0468\)), with a Sharpe ratio of \(0.550\). Together, these results show that distinguishing symbolic participation from optimizer relevance affects not only support recovery in the controlled setting, but also downstream decisions. Full results are reported in Appendix~Table~\ref{tab:knapsack_layer2}.

\begin{figure*}[t]
\centering
\includegraphics[width=\textwidth]{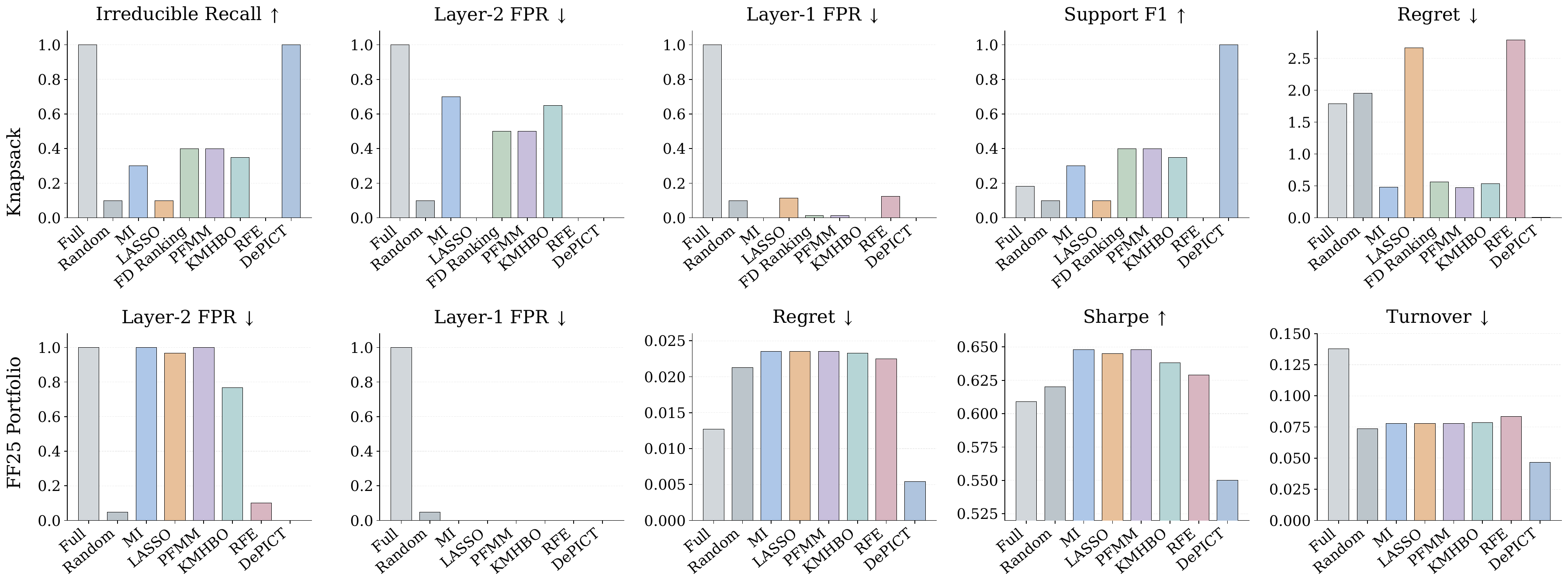}
\caption{\textbf{Layer-2 annihilation diagnostic.} Controlled knapsack (top) and FF25 portfolio (bottom) compare recovery of irreducible directions and Layer-2 nuisance variables. Mean over 10 seeds.}
\label{fig:knapsack_layer2_main}
\vspace{-0.4cm}
\end{figure*}
\vspace{-0.4cm}

\vspace{-0.1cm}
\subsection{Stability under perturbations}
\label{sec:exp_stability}
\vspace{-0.1cm}

The stability study tests whether DePICT's selected interface is robust to small changes in the operating context, rather than being an artifact of a single nominal point. We perturb a reference context $c_0$ and ask whether the top-$m$ interface remains similar and preserves downstream decision quality as the perturbation radius increases. Figure~\ref{fig:knapsack_stability_ablation} reports exact top-$m$ recovery,
Jaccard overlap, and regret across interface budgets
$m\in\{20,40,60\}$ as $\epsilon$ increases.

Exact top-$m$ recovery decreases with larger perturbations, consistent with
the ranking-gap result in Section~\ref{sec:stability}. Nevertheless, the
selected interfaces remain highly similar: for $m=60$, Jaccard overlap stays
above \(0.979\) even at \(\epsilon=5\times10^{-2}\), while regret remains
close to zero. Thus, exact ranking identity is local, but small context
perturbations can still preserve a highly similar interface and downstream
decision quality.

\begin{figure*}[t]
\centering
\includegraphics[width=\textwidth]{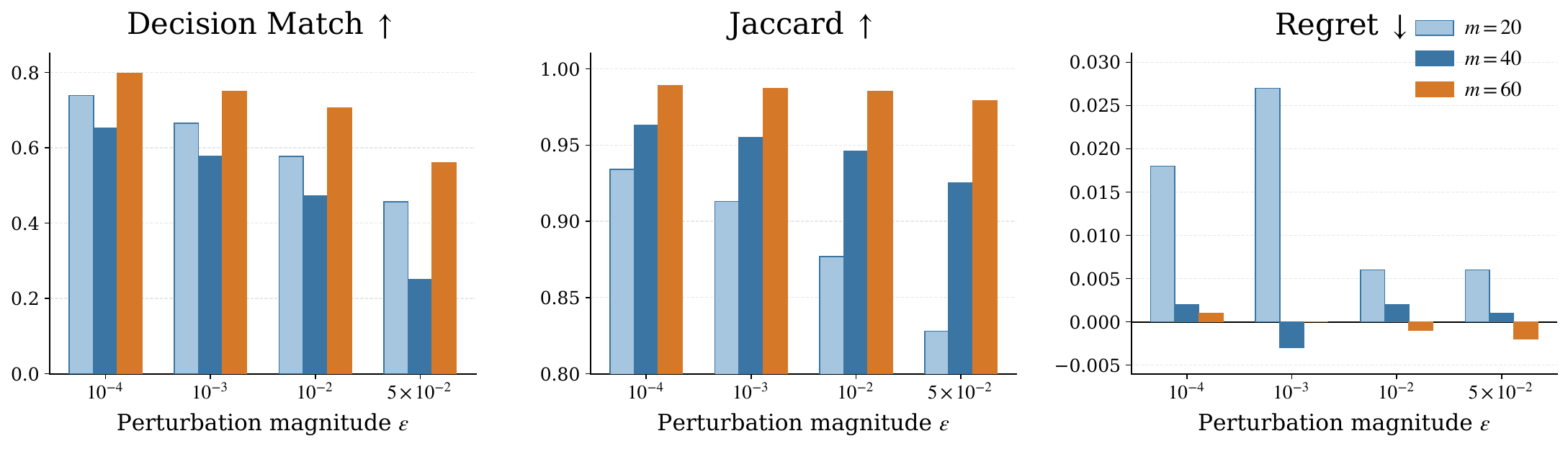}
\vspace{-0.6cm}
\caption{
\textbf{Perturbation stability on controlled knapsack.}
Decision match, Jaccard overlap, and regret are reported across interface budgets $m\in\{20,40,60\}$ as the perturbation magnitude $\epsilon$ increases. %Decision match decreases with larger perturbations, while Jaccard overlap remains high and regret stays close to zero. This supports the spectral-gap view: exact set identity is a local certificate, while high-overlap interfaces can remain decision-stable even when individual rankings change.
}
\label{fig:knapsack_stability_ablation}
\vspace{-0.4cm}
\end{figure*}

\vspace{-0.1cm}
\subsection{Runtime Scaling with Context Dimension}
\label{sec:runtime_scaling}
\vspace{-0.1cm}
The runtime study tests whether DePICT remains practical as the primitive context dimension grows, rather than becoming increasingly expensive with the number of observed inputs. We therefore vary $n_1$ while holding the optimizer
dimension fixed at $n_\theta=20$ and using $T=200$ representative contexts.  

Figure~\ref{fig:runtime_scaling} compares DePICT with LASSO, mutual information (MI), finite-difference (FD) ranking, and recursive feature elimination (RFE). DePICT requires approximately $5.2$-$5.4$\,s across the entire sweep from $n_1=200$ to $n_1=2000$. In contrast, LASSO increases from $1.79$\,s to $29.17$\,s, MI from $10.07$\,s to $108.05$\,s, FD ranking from $37.47$\,s to $372.78$\,s, and RFE from $1.98$\,s to $6.95$\,s. Thus, although some lightweight selectors are faster at small context dimension, DePICT shows substantially weaker dependence on the ambient context dimension over the tested range. This behavior follows from the structure of the sensitivity computation. At each representative context, DePICT solves a KKT system of size $n_\theta+|\mathcal A|$, which is determined by the downstream optimization problem rather than the ambient context dimension $n_1$. Increasing $n_1$
adds more sensitivity directions, but these can be evaluated using the same KKT factorization. 
As a result, when the optimizer size is fixed, the dominant linear-system cost changes little even as the number of context variables grows.

\begin{figure}[t]
    \centering
    \includegraphics[width=\linewidth]{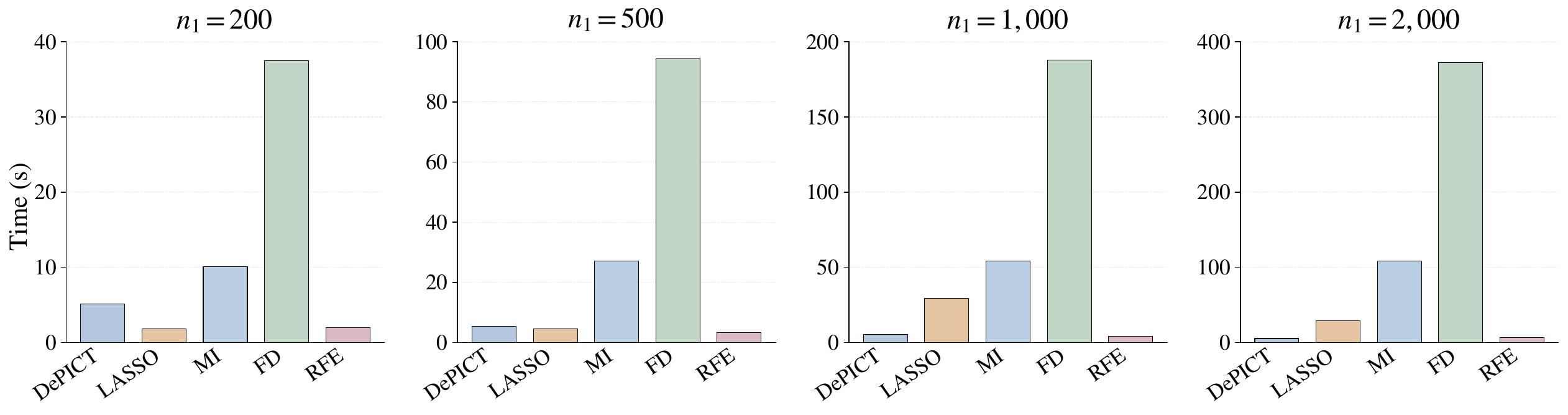}
    \caption{\textbf{Interface-selection runtime as context dimension increases.}
    Time comparison across different methods with increasing context dimension for Controlled Knapsack dataset.}
    \label{fig:runtime_scaling}
    \vspace{-0.4cm}
\end{figure}

% \section{Limitations and Future Work}
% \label{sec:limitations}

% Our theory is local and first-order: it assumes a stable active set and KKT regularity. It does not give exact guarantees at nonsmooth boundaries, for discrete decision maps, or under frequent active-set changes. Warcraft therefore serves as an out-of-regime stress test, evaluating whether the interface principle remains useful in a nonsmooth, solver-mediated pipeline.
% DePICT also has second-order computational cost. Policy-Jacobian scoring uses the Lagrangian Hessian, active-constraint Jacobian, and linearized KKT solves. The closed form involving \(H_L^{-1}\) and \(S_{\mathcal A}^{-1}\) is analytic; implementations need not form these inverses explicitly. Scaling to very large optimizers will require matrix-free or amortized sensitivity estimators.

%Finally, our statistical analysis is deliberately limited. We prove that discarding an irreducible direction creates structural first-order bias, and
%the appendix gives a conditional population-level risk floor. A complete learning theory connecting interface sufficiency, sample complexity, predictor class, Jacobian-estimation error, and active-set uncertainty remains open.

%Future work should extend decision-preserving interfaces to piecewise-smooth
%and nonsmooth optimization, develop finite-sample guarantees for correctly
%specified interfaces, and design scalable Jacobian-estimation procedures for
%large constrained, sequential, and multi-stage decision systems.
\vspace{-0.2cm}
\section{Conclusion and Limitations}
\label{sec:conclusion}
\vspace{-0.2cm}

DePICT recovers the decision-relevant interface exactly on both controlled benchmarks and directly validates the symbolic--optimizer gap: in the Layer-2 diagnostic, it reduces linear-predictor regret from $0.475$ to $0.009$ while retaining all irreducible coordinates and excluding both nuisance classes. The selected interfaces also remain stable under context perturbations and transfer effectively to Warcraft and FF25. Together, these results support the central principle of DePICT: interface design should be governed by whether
context moves the downstream optimizer, rather than by prediction utility or symbolic occurrence alone.

Our guarantees require stable active sets and KKT regularity, leaving
nonsmooth, discrete, and changing-active-set regimes for future work.
DePICT also relies on second-order sensitivity information and linearized KKT solves, motivating matrix-free or amortized estimators for larger optimization problems. Finally, the present analysis is local and first-order; extending it to finite perturbations and higher-order guarantees remains open. More broadly, DePICT determines which context the optimizer needs, while decision-focused learning determines how to use that context.
%First, our guarantees require stable active sets and KKT regularity; extending the theory to piecewise-smooth, nonsmooth, and discrete decision maps is an important direction. Second, DePICT currently relies on second-order sensitivity information through Lagrangian Hessians, active constraint Jacobians, and linearized KKT solves; scaling to very large optimizers calls for matrix free or amortized sensitivity estimators. Third, our analysis is first-order and local; developing finite sample and higher-order guarantees remains open. The broader implication is a natural division of labor: DePICT identifies structurally necessary context directions before training, and DFL methods optimize how to use them treating interface design and predictor training as separate, complementary problems.

%\input{Sections/Conclusion}

%\input{Sections/AIUse}
\section*{Acknowledgements}
This work was supported by CRADA22-0034-J003 and and NVIDIA Academic Grant Program.
\bibliography{mybibliography}
\bibliographystyle{iclr2027_conference}

\section*{Appendix}
\section{Related Work}
\label{app:related_work}
\paragraph{Decision-focused learning.}
Decision-focused learning trains predictive models according to the quality of the downstream decision rather than a surrogate prediction loss. This view appears in predict-and-optimize losses such as SPO/SPO+, ranking-based DFL, sequential decision settings, feasibility-aware formulations, and solver-free training for linear optimization problems \cite{elmachtoub2022spo,mandi2022rank,wang2021mdpfeatures,mandi2025feasibility,berden2025solverfree}. These works primarily ask how to train predictors so that predicted parameters lead to high-quality decisions. Our question is different: before training any predictor, which components of the context must be preserved so that the downstream constrained optimizer can still recover the correct local optimum?

\paragraph{Differentiable optimization layers.}
The solution-sensitivity Jacobian of Section~\ref{sec:problem} is obtained by
implicit differentiation of the KKT system, the same mechanism that underlies
differentiable optimization layers
\citep{gould2016differentiating,amos2017optnet,agrawal2019differentiable,blondel2022efficient}
and their use in task-based learning \citep{donti2017task}. Those works differentiate
\emph{through} a solver in order to train an upstream model; the Jacobian is an
intermediate of backpropagation. DePICT uses the same object for a different question:
which primitive coordinates have an identically zero column, when a column vanishes
despite symbolic participation (Theorem~\ref{thm:kkt-characterization}), and hence
which context can be discarded before any training. The Layer-2 and LP boundary cases
of Section~\ref{sec:separation} are also the regime in which solver gradients vanish,
which motivates surrogate losses, blackbox and perturbed estimators, and solver-free
training \citep{elmachtoub2022smart,mandi2020smart,vlastelica2020blackbox,berthet2020learning,berden2025solverfree};
DePICT treats that vanishing as information about relevance rather than as an obstacle
to training.

\paragraph{Task-specific data and feature relevance.}
The closest conceptual neighbor is \emph{What Data Enables Optimal Decisions?}, which characterizes which data directions are sufficient for recovering optimal decisions in linear optimization under cost uncertainty \cite{bennouna2025whatdata,ning2018data}. We share its task-specific notion of relevance, but study a different object. Rather than dataset sufficiency for linear programs, we characterize primitive context directions that preserve the local solution map of a nonlinear constrained problem~\cite{gupta2016pac}. This requires local KKT sensitivity rather than global polyhedral geometry~\cite{fiacco1976sensitivity}. Which leads us to task-aware feature selection and prediction-focused representation learning \cite{sharma2023pfmm,ren2020prediction}. These methods identify features useful for prediction, clustering, or latent-variable modeling. In contrast, our relevance notion is optimizer-induced: a context coordinate is relevant only if perturbing it changes the constrained optimum. This distinction matters because a feature may appear symbolically in constraints or objectives while being annihilated by the KKT geometry of the active set~\cite{creese2022geometric}.
\paragraph{Feature-selection baselines.}
We compare against standard feature-selection methods spanning multiple paradigms, including 
mutual-information screening \cite{peng2005mrmr}, LASSO \cite{tibshirani1996lasso}, 
recursive feature elimination (RFE) \cite{kohavi1997rfe,guyon2002rfe}, 
perturbation-based finite-difference (FD) ranking a removal-style attribution baseline \cite{lin2023removal}, and random subset selection. 
We further include task-aware feature selection via prediction-focused mixture models (PFMM) 
\cite{sharma2023pfmm}, and black-box metaheuristic approaches such as KMHBO 
\cite{ye2026kmhbo}. 
These methods select features based on prediction utility, sparsity, or heuristic or local perturbation signals, 
and do not account for constraint geometry. 
In contrast, our approach defines relevance through KKT sensitivity of the downstream optimizer, 
identifying context directions that are necessary to preserve the constrained optimum.

\paragraph{Positioning.}
Taken together, prior work explains how to train predictors for downstream optimization, how to reason about task-sufficient data in linear decision problems, and how parametric optima vary under perturbations. What remains missing is a characterization of which primitive context directions are genuinely necessary to preserve the local optimum of a constrained downstream task, and when symbolic participation overestimates optimizer relevance. DePICT fills this gap by introducing decision-preserving interfaces, characterizing them through local KKT sensitivity, and proving that discarding an irreducible direction creates structural first-order error that downstream learning cannot repair.

\begin{figure*}[t]
    \centering
    \includegraphics[width=0.95\linewidth]{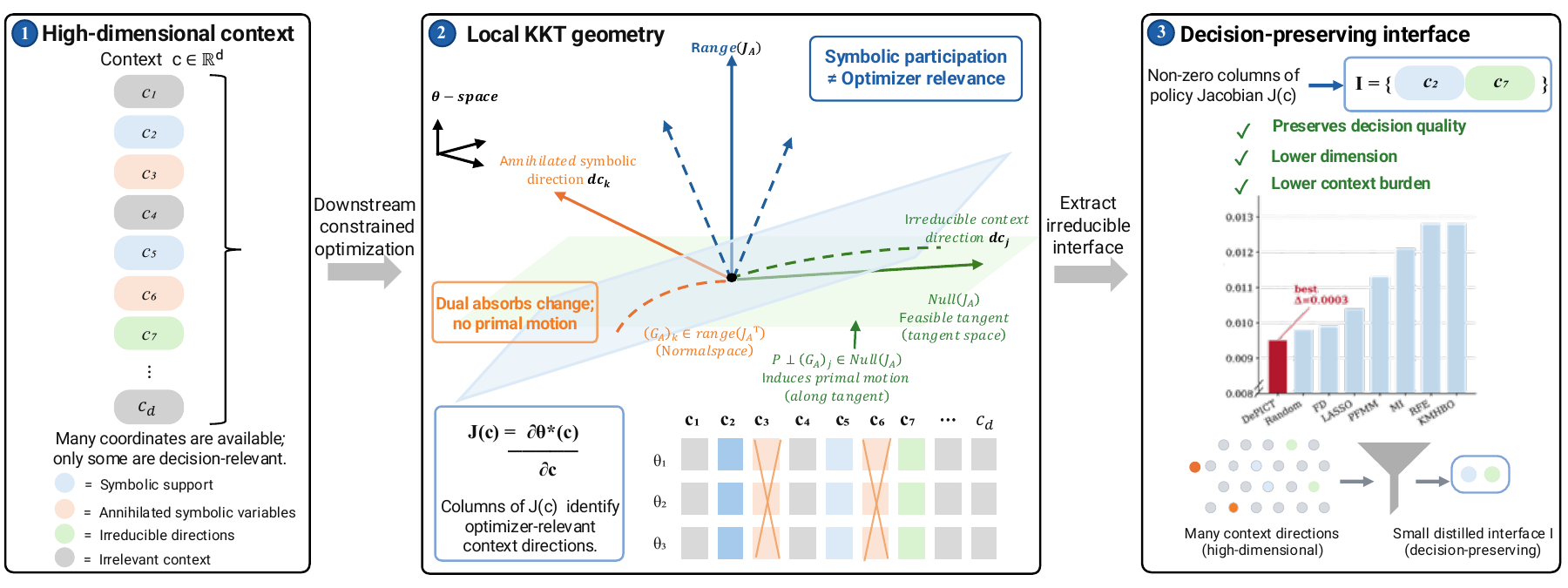}
    %\vspace{-0.3cm}
    \caption{DePICT constructs the information interface from the downstream decision.}
    \label{fig:Main}
    %\vspace{-0.4cm}
\end{figure*}

\section{Interface Construction Details}
\label{app:interface_generator}

Algorithm~\ref{alg:depict_detailed} and Fig~\ref{fig:Main} gives the full DePICT interface construction procedure. The input is an operating regime represented by
contexts \(\{c_t\}_{t=1}^T\), together with the downstream objective, constraints, and a solver or benchmark-specific continuous relaxation. For each context, DePICT computes how the oracle optimizer changes under primitive context perturbations. These local sensitivities are then aggregated across the regime to produce a reduced interface
\(\widehat{\mathcal I}\subseteq\{1,\ldots,n_1\}\).

The key computation is the solution-sensitivity Jacobian \(J(c_t)=\partial\theta^*(c_t)/\partial c_t\). DePICT computes this Jacobian by solving the linearized active-set KKT system. This is the
general formulation used by the method. When the stronger closed-form conditions hold, the same Jacobian can equivalently be computed through the geometric decomposition \(J=P^\perp G_{\mathcal A}-\Phi C_{\mathcal A}\). In implementation, the saddle-system form is preferred because it avoids explicitly forming
\(H_L^{-1}\) or \(S_{\mathcal A}^{-1}\).

For shortest-path benchmarks, DePICT does not differentiate the discrete argmin. Instead, it computes sensitivities through a strongly convex regularized flow QP. The same regularized flow formulation is used for downstream evaluation, so interface construction and decision evaluation operate on a common smooth optimization problem.

Each coordinate \(j\) receives two quantities at every representative context. The score \(s_t(j)=\|\operatorname{col}_j(J(c_t))\|_2\) measures the magnitude of the first-order optimizer response to perturbing coordinate \(j\). The indicator \(z_t(j)=\mathbf 1\{s_t(j)>\tau\}\) records whether the coordinate is numerically irreducible at context
\(c_t\). The threshold \(\tau\) avoids treating numerical noise as nonzero sensitivity.

After all contexts are processed, DePICT aggregates the scores \(\{s_t(j)\}_{t=1}^T\) into a regime-level score \(\bar s(j)\), and
computes the persistence frequency \(\pi(j)=T^{-1}\sum_t z_t(j)\). The final interface can then be selected in three modes. \textsc{Top-}\(m\) returns a fixed-budget interface by taking the \(m\) largest aggregated scores. \textsc{Freq-Filtered} returns coordinates whose irreducibility persists on at least an \(\alpha\)-fraction of representative contexts. \textsc{Strict-Union} returns every coordinate that is irreducible at least once in the sampled regime. These modes trade off compactness and coverage: \textsc{Top-}\(m\) is budget-controlled, \textsc{Freq-Filtered} suppresses unstable directions, and \textsc{Strict-Union} is conservative.
\begin{algorithm}[t]
\caption{Detailed DePICT Interface Construction}
\label{alg:depict_detailed}
\small
\begin{algorithmic}[1]
\Require Representative contexts $\{c_t\}_{t=1}^{T}$;
downstream solver or relaxation solver;
sensitivity routine;
mode in \{\textsc{Top-}$m$, \textsc{Freq-Filtered},
\textsc{Strict-Union}\};
budget $m$;
optional persistence thresholds $\alpha,\tau$
\Ensure Interface estimate $\widehat{\mathcal I}$

\For{$t=1,\dots,T$}
    \State Solve the downstream problem, or the benchmark-specific continuous relaxation, at $c_t$:
    \[
    \theta^*(c_t)\in
    \arg\min_{\theta} f(\theta;c_t)
    \quad
    \mathrm{s.t.}\quad
    m_i(\theta;c_t)\ge 0 .
    \]
    \State Recover the primal-dual point $(\theta^*(c_t),\lambda^*(c_t))$ and active set $\mathcal A_t$.
    \State Compute the solution-sensitivity Jacobian $J(c_t)$ by solving the linearized active-set KKT system:
    \[
    K_{\mathcal A_t}(c_t)
    \begin{bmatrix}
    J(c_t)\\
    \Lambda_c(c_t)
    \end{bmatrix}
    =
    \begin{bmatrix}
    G_{\mathcal A_t}(c_t)\\
    -C_{\mathcal A_t}(c_t)
    \end{bmatrix}.
    \]
    \Statex \hspace{\algorithmicindent}\textit{When the closed-form conditions hold, equivalently compute}
    \[
    J(c_t)
    =
    P^\perp(c_t)G_{\mathcal A_t}(c_t)
    -
    \Phi(c_t)C_{\mathcal A_t}(c_t).
    \]
    \For{$j=1,\dots,n_1$}
        \State $s_t(j)\gets \|\operatorname{col}_j(J(c_t))\|_2$
        \State $z_t(j)\gets \mathbf{1}\{s_t(j)>\tau\}$
    \EndFor
\EndFor

\For{$j=1,\dots,n_1$}
    \State Aggregate score:
    \[
    \bar s(j)\gets \mathsf{Agg}\bigl(\{s_t(j)\}_{t=1}^T\bigr)
    \]
    \State Estimate persistence:
    \[
    \pi(j)\gets \frac{1}{T}\sum_{t=1}^T z_t(j)
    \]
\EndFor

\If{mode = \textsc{Top-}$m$}
    \State $\widehat{\mathcal I}\gets$ indices of the $m$ largest values of $\bar s(j)$
\ElsIf{mode = \textsc{Freq-Filtered}}
    \State $\widehat{\mathcal I}\gets \{j:\pi(j)\ge \alpha\}$
\ElsIf{mode = \textsc{Strict-Union}}
    \State $\widehat{\mathcal I}\gets \{j:\pi(j)>0\}$
\EndIf

\State \Return $\widehat{\mathcal I}$
\end{algorithmic}
\end{algorithm}

\paragraph{Computational cost.}
For each representative context, exact DePICT scoring requires one linearized KKT solve with saddle matrix dimension \(n_\theta+|\mathcal A_t|\). A dense implementation may factor this
matrix once per context and reuse the factorization across multiple right-hand sides. The algorithm does not require explicit matrix
inverses, and the saddle-system form supports matrix-free KKT solves, Hessian-vector products, adjoint products, randomized score estimation,
and parallel computation across contexts. Thus the exact implementation used in our benchmarks is not the only route; large-scale deployments can trade exact Jacobian recovery for approximate or amortized sensitivity scores while preserving the same optimizer-relevance criterion.

\section{Derivations}
\subsection{Derivation of the solution-sensitivity Jacobian}
\label{app:jacobian-derivation}

Each coordinate $j$ receives the local sensitivity score
\[
s_t(j)=\|\operatorname{col}_j(J(c_t))\|_2,
\]
which measures the magnitude of the first-order optimizer response to
perturbing coordinate $j$. The main TOP-$m$ construction averages these
scores across representative contexts,
\[
\bar{s}(j)=\frac{1}{T}\sum_{t=1}^{T}s_t(j),
\]
and returns the $m$ coordinates with largest $\bar{s}(j)$.

For the optional persistence-based variants, we additionally define
$z_t(j)=\mathbf{1}\{s_t(j)>\tau\}$ and
$\pi(j)=T^{-1}\sum_t z_t(j)$, where $\tau$ is used only to suppress
numerical zero-level noise. FREQ-FILTERED retains coordinates with
$\pi(j)\geq\alpha$, while STRICT-UNION retains coordinates with
$\pi(j)>0$.

\paragraph{Regularity conditions.}
We assume the following standard local regularity conditions on the
operating regime. \emph{LICQ} requires the gradients of the active
constraints to be linearly independent, so the active set defines a
well-posed local tangent space. \emph{SOSC} requires the Lagrangian
Hessian to be positive definite on feasible tangent directions, ensuring
that the KKT point is a strict local minimum rather than a saddle or flat
direction. \emph{Strict complementarity} requires every active inequality
constraint to have a strictly positive multiplier, separating active and
inactive constraints and preventing degenerate boundary behavior.
Together, these conditions make the active-set KKT system nonsingular and
allow the optimizer map \(c\mapsto(\theta^\star(c),\lambda^\star(c))\)
to be locally differentiable.
\paragraph{LICQ.}
Let \(\mathcal A(c)=\{i:m_i(\theta^\star(c);c)=0\}\) be the active set.
LICQ requires
\[
A_{\mathcal A}(c)
=
\bigl[\nabla_\theta m_i(\theta^\star(c);c)\bigr]_{i\in\mathcal A}
\]
to have full row rank. Equivalently, the active constraint gradients are
linearly independent.

\paragraph{SOSC.}
Let
\[
H_L(c)=\nabla_{\theta\theta}^2 L(\theta^\star(c),\lambda^\star(c);c)
\]
be the Lagrangian Hessian. SOSC requires
\[
d^\top H_L(c)d>0
\qquad
\forall d\neq0
\ \text{such that}\
A_{\mathcal A}(c)d=0 .
\]
That is, curvature is positive along feasible tangent directions.

\paragraph{Strict complementarity.}
For every active inequality constraint,
\[
\lambda_i^\star(c)>0
\qquad \forall i\in\mathcal A(c).
\]
Inactive constraints have zero multipliers by complementary slackness.
Strict complementarity rules out active constraints with zero multiplier,
which are degenerate and can cause active-set instability.

\paragraph{Step 1: KKT residual system.}
At the oracle solution $(\theta^*(c),\lambda^*(c))$, the KKT
conditions on the active set are
\begin{align}
\nabla_\theta f(\theta^*(c);c)
-
\sum_{i\in\mathcal{A}(\mathcal{M})}
\lambda_i^*(c)\,\nabla_\theta m_i(\theta^*(c);c)
&=0,
\label{eq:kkt_stationarity_app}\\
m_i(\theta^*(c);c)&=0,
\qquad i\in\mathcal{A}(\mathcal{M}).
\label{eq:kkt_feasibility_app}
\end{align}
Define the residual
\begin{equation}
F(\theta,\lambda,c)
:=
\begin{bmatrix}
\nabla_\theta f(\theta;c)-A_{\mathcal A}(\theta,c)^\top\lambda\\[3pt]
m_{\mathcal A}(\theta;c)
\end{bmatrix},
\label{eq:kkt_residual_app}
\end{equation}
where
\[
A_{\mathcal A}(\theta,c)
:=
\bigl[\nabla_\theta m_i(\theta;c)\bigr]_{i\in\mathcal{A}(\mathcal{M})}
\in\mathbb{R}^{|\mathcal{A}(\mathcal{M})|\times n_\theta}.
\]
Then the KKT system is
\[
F(\theta^*(c),\lambda^*(c),c)=0.
\]

\paragraph{Step 2: Jacobian of the KKT residual with respect to $(\theta,\lambda)$.}
Differentiating $F$ with respect to $(\theta,\lambda)$ gives the saddle matrix
\begin{equation}
\frac{\partial F}{\partial(\theta,\lambda)}
=
\begin{bmatrix}
H_L(c) & -A_{\mathcal A}(c)^\top\\[3pt]
A_{\mathcal A}(c) & 0
\end{bmatrix}
=: \mathbf{M}(c),
\label{eq:saddle_matrix_app}
\end{equation}
where
\[
H_L(c)
=
\nabla_{\theta\theta}^2
\Bigl(
f(\theta^*(c);c)
-
\sum_{i\in\mathcal{A}(\mathcal{M})}
\lambda_i^*(c)m_i(\theta^*(c);c)
\Bigr)
\]
is the Lagrangian Hessian. Under LICQ and SOSC, $\mathbf{M}(c)$ is
invertible on $\mathcal{M}$.

When \(H_L\) and the Schur complement
\[
S_{\mathcal A}(c):=
A_{\mathcal A}(c)H_L(c)^{-1}A_{\mathcal A}(c)^\top
\]
are nonsingular, define
\[
\Phi(c):=H_L(c)^{-1}A_{\mathcal A}(c)^\top
S_{\mathcal A}(c)^{-1},\qquad
P^\perp(c):=H_L(c)^{-1}
-\Phi(c)A_{\mathcal A}(c)H_L(c)^{-1}.
\]
These objects are used only for the closed-form interpretation; DePICT
uses the saddle-system formulation and does not require forming
\(H_L^{-1}\), \(S_{\mathcal A}^{-1}\), or \(M(c)^{-1}\) explicitly.

\paragraph{Step 3: Jacobian of the KKT residual with respect to \(c\).}
To apply the implicit function theorem, differentiate \(F\) with respect
to \(c\), holding \((\theta,\lambda)\) fixed. For the upper block,
\[
\frac{\partial}{\partial c}
\left(\nabla_\theta f(\theta;c)-A_{\mathcal A}(\theta,c)^\top\lambda\right)
=
\nabla^2_{\theta c}f(\theta;c)
-
\sum_{i\in\mathcal A(M)}
\lambda_i
\frac{\partial^2 m_i}{\partial\theta\,\partial c}(\theta;c).
\]
Define
\[
G_{\mathcal A}(c)
:=
\sum_{i\in\mathcal A(M)}
\lambda_i^\star(c)
\frac{\partial^2 m_i}{\partial\theta\,\partial c}
(\theta^\star(c);c)
-
\nabla^2_{\theta c}f(\theta^\star(c);c).
\]
Then the upper block of \(\partial F/\partial c\) is
\(-G_{\mathcal A}(c)\). If \(f\) is independent of \(c\), the objective
mixed-curvature term vanishes and \(G_{\mathcal A}\) reduces to the
constraint-driven expression.
For the lower block,
\begin{equation}
\frac{\partial}{\partial c}m_{\mathcal A}(\theta;c)
=
C_{\mathcal A}(c),
\qquad
C_{\mathcal A}(c)
:=
\bigl[\partial m_i(\theta^*(c);c)/\partial c\bigr]_{i\in\mathcal{A}(\mathcal{M})}.
\label{eq:CA_appendix}
\end{equation}
Thus
\begin{equation}
\frac{\partial F}{\partial c}
=
\begin{bmatrix}
-\,G_{\mathcal A}(c)\\[3pt]
C_{\mathcal A}(c)
\end{bmatrix}.
\label{eq:F_c_appendix}
\end{equation}

\paragraph{Step 4: Apply the Implicit Function Theorem.}
Since
\[
F(\theta^*(c),\lambda^*(c),c)=0
\]
and $\partial F/\partial(\theta,\lambda)$ is invertible, the Implicit
Function Theorem yields
\begin{equation}
\frac{\partial(\theta^*,\lambda^*)}{\partial c}(c)
=
-
\left(
\frac{\partial F}{\partial(\theta,\lambda)}
\right)^{-1}
\frac{\partial F}{\partial c}.
\label{eq:IFT_appendix}
\end{equation}
Substituting \eqref{eq:saddle_matrix_app} and \eqref{eq:F_c_appendix},
the top block gives
\begin{align}
\frac{\partial\theta^*}{\partial c}(c)
&=
-
\begin{bmatrix}
P^\perp(c) & \Phi(c)
\end{bmatrix}
\begin{bmatrix}
-\,G_{\mathcal A}(c)\\[3pt]
C_{\mathcal A}(c)
\end{bmatrix}
\notag\\
&=
P^\perp(c)\,G_{\mathcal A}(c)
-
\Phi(c)\,C_{\mathcal A}(c).
\label{eq:policy_jacobian_appendix}
\end{align}
This is the solution-sensitivity Jacobian formula used throughout the paper.

\paragraph{Interpretation.}
The two terms in \eqref{eq:policy_jacobian_appendix} have distinct
meanings. The term $-\Phi(c)C_{\mathcal A}(c)$ is the
\emph{feasibility-tracking} component: it is the primal motion needed
to compensate for direct context-induced shifts in the active margins.
The term $P^\perp(c)G_{\mathcal A}(c)$ is the
\emph{geometry-rotation} component: it captures the part of active
constraint-gradient rotation that cannot be absorbed through multiplier
adjustment and therefore forces genuine motion of the optimum along the
feasible tangent space.

\paragraph{Objective dependence on context.}
The definition of $G_{\mathcal A}$ already includes both sources of stationarity
sensitivity:
\[
G_{\mathcal A}(c)
=
\sum_{i\in\mathcal A}
\lambda_i^\star(c)\nabla^2_{\theta c}m_i(\theta^\star(c);c)
-
\nabla^2_{\theta c}f(\theta^\star(c);c).
\]
When the objective is independent of $c$, the second term vanishes and
$G_{\mathcal A}$ reduces to the constraint-driven contribution.

\subsection{Symbolic Support vs.\ Optimizer Relevance: Full Statements and Proofs}
\label{app:separation}

\subsubsection{Symbolic support as a strict upper bound}

\begin{proposition}[Symbolic support as a strict upper bound]
\label{prop:symbolic-bound-app}
Under A1--A2:
\begin{enumerate}[label=\textup{(\roman*)},leftmargin=2em]
  \item \textup{(Necessary containment.)}
    $\Irr(\mathcal{M}) \subseteq \Supp(\mathcal{M})$.
    If $j \notin \Supp(\mathcal{M})$ then $v_j \equiv 0$ on $\mathcal{M}$.

  \item \textup{(Pointwise annihilation.)}
    $v_j(c) = 0$ if and only if
    \begin{equation}
      (C_{\mathcal A}(c))_{\cdot j} = 0
      \quad\text{and}\quad
      (G_{\mathcal A}(c))_{\cdot j} \in \range\bigl(A_{\mathcal A}(c)^\top\bigr).
      \label{eq:annihilation_app}
    \end{equation}
    Equivalently:
    \[
      (C_{\mathcal A}(c))_{\cdot j}=0
      \qquad\text{and}\qquad
      P^\perp(c)(G_{\mathcal A}(c))_{\cdot j}=0.
    \]

  \item \textup{(Regime-level annihilation.)}
    $j \notin \Irr(\mathcal{M})$ if and only if
    \eqref{eq:annihilation_app} holds for every $c\in\mathcal{M}$.
    In generic nondegenerate regimes, it is enough to verify this on an
    open dense subset.

  \item \textup{(Two-layer structure.)}
    $\Supp(\mathcal{M}) \setminus \Irr(\mathcal{M})$ decomposes into:
    \begin{itemize}[leftmargin=1.5em]
      \item \textbf{Layer~1:}
        $j \in \Supp_{\mathrm{full}}\setminus\Supp(\mathcal{M})$,
        i.e.\ $c_j$ appears only in inactive constraints and
        $(C_{\mathcal A})_{\cdot j}\equiv
        (G_{\mathcal A})_{\cdot j}\equiv 0$;
      \item \textbf{Layer~2:}
        $j \in \Supp(\mathcal{M})\setminus\Irr(\mathcal{M})$,
        i.e.\ $(C_{\mathcal A})_{\cdot j}\equiv 0$,
        $(G_{\mathcal A})_{\cdot j}\not\equiv 0$, and
        $(G_{\mathcal A})_{\cdot j}\in\range(A_{\mathcal A}^\top)$
        throughout $\mathcal{M}$.
    \end{itemize}
\end{enumerate}
\end{proposition}

\begin{proof}
\textbf{(i).}
If $j\notin\Supp(\mathcal{M})$, then by definition
\[
(C_{\mathcal A}(c))_{\cdot j}\equiv 0,
\qquad
(G_{\mathcal A}(c))_{\cdot j}\equiv 0
\qquad
\forall c\in\mathcal{M}.
\]
Hence
\[
v_j(c)
=
P^\perp(c)(G_{\mathcal A}(c))_{\cdot j}
-
\Phi(c)(C_{\mathcal A}(c))_{\cdot j}
=
0
\]
for all $c\in\mathcal{M}$.

\medskip
\textbf{(ii).}
Differentiate the KKT residual
\[
F(\theta^*(c),\lambda^*(c),c)
=
\bigl(
\nabla_\theta \mathcal{L}(\theta^*(c),\lambda^*(c),c),
m_{\mathcal A}(\theta^*(c);c)
\bigr)
=
0
\]
with respect to $c_j$.
Using the $C^1$ regularity of $(\theta^*,\lambda^*)$ on $\mathcal{M}$,
the stationarity block yields
\[
H_L(c)\,v_j(c)-A_{\mathcal A}(c)^\top \mu_j(c)
=
(G_{\mathcal A}(c))_{\cdot j},
\]
while the feasibility block yields
\[
A_{\mathcal A}(c)\,v_j(c)
=
-(C_{\mathcal A}(c))_{\cdot j},
\]
where $\mu_j(c):=\partial \lambda^*(c)/\partial c_j$.
Thus
\begin{equation}
\mathbf M(c)
\begin{pmatrix}
v_j(c)\\
\mu_j(c)
\end{pmatrix}
=
\begin{pmatrix}
(G_{\mathcal A}(c))_{\cdot j}\\[2pt]
-(C_{\mathcal A}(c))_{\cdot j}
\end{pmatrix},
\qquad
\mathbf M(c)
:=
\begin{pmatrix}
H_L(c) & -A_{\mathcal A}(c)^\top\\
A_{\mathcal A}(c) & 0
\end{pmatrix}.
\label{eq:kkt_saddle_app}
\end{equation}
By Lemma~\ref{lem:kkt-nonsingular}, \(M(c)\) is invertible, so the
solution is unique.

For the forward implication, if $v_j(c)=0$, then the lower block of
\eqref{eq:kkt_saddle_app} gives
\[
(C_{\mathcal A}(c))_{\cdot j}=0,
\]
and the upper block gives
\[
(G_{\mathcal A}(c))_{\cdot j}
=
-\,A_{\mathcal A}(c)^\top \mu_j(c)
\in
\range(A_{\mathcal A}(c)^\top).
\]

For the reverse implication, suppose
\[
(C_{\mathcal A}(c))_{\cdot j}=0,
\qquad
(G_{\mathcal A}(c))_{\cdot j}=A_{\mathcal A}(c)^\top \bar\mu
\]
for some $\bar\mu$.
Then $(v_j,\mu_j)=(0,-\bar\mu)$ satisfies both blocks of
\eqref{eq:kkt_saddle_app}, and uniqueness implies $v_j(c)=0$.

For the projector equivalence, let $u=(G_{\mathcal A}(c))_{\cdot j}$.
If $u=A_{\mathcal A}(c)^\top w$, then
\begin{align*}
P^\perp(c)u
&=
\Bigl(
H_L(c)^{-1}
-
\Phi(c)A_{\mathcal A}(c)H_L(c)^{-1}
\Bigr)
A_{\mathcal A}(c)^\top w \\
&=
H_L(c)^{-1}A_{\mathcal A}(c)^\top w
-
\Phi(c)\underbrace{A_{\mathcal A}(c)H_L(c)^{-1}A_{\mathcal A}(c)^\top}_{S_{\mathcal A}(c)}w \\
&=
H_L(c)^{-1}A_{\mathcal A}(c)^\top w
-
H_L(c)^{-1}A_{\mathcal A}(c)^\top S_{\mathcal A}(c)^{-1}S_{\mathcal A}(c)w \\
&=0.
\end{align*}
Conversely, if $P^\perp(c)u=0$, then
\[
H_L(c)^{-1}u
=
\Phi(c)A_{\mathcal A}(c)H_L(c)^{-1}u.
\]
Premultiplying by $H_L(c)$ gives
\[
u
=
H_L(c)\Phi(c)A_{\mathcal A}(c)H_L(c)^{-1}u
=
A_{\mathcal A}(c)^\top
S_{\mathcal A}(c)^{-1}
A_{\mathcal A}(c)H_L(c)^{-1}u
\in
\range(A_{\mathcal A}(c)^\top).
\]

\medskip
\textbf{(iii).}
By definition,
$j\notin\Irr(\mathcal{M})$ iff $v_j(c)=0$ for all $c\in\mathcal M$.
Applying part~(ii) pointwise gives the claim.
The open-dense formulation follows from continuity of $v_j(c)$ and the
fact that, in generic nondegenerate regimes, the zero/nonzero pattern of
a continuous field is determined away from lower-dimensional
degeneracies.

\medskip
\textbf{(iv).}
Layer~1 follows immediately from
\[
(C_{\mathcal A})_{\cdot j}\equiv 0,
\qquad
(G_{\mathcal A})_{\cdot j}\equiv 0.
\]
Layer~2 follows from part~(ii): if
$j\in \Supp(\mathcal M)\setminus \Irr(\mathcal M)$, then
$(G_{\mathcal A})_{\cdot j}\not\equiv 0$ but $v_j\equiv 0$, which
forces
\[
(C_{\mathcal A})_{\cdot j}\equiv 0,
\qquad
(G_{\mathcal A})_{\cdot j}\in\range(A_{\mathcal A}^\top)
\quad\forall c\in\mathcal M.
\]
\end{proof}

\begin{remark}[Geometry of Layer-2 annihilation]
\label{rem:geometry-app}
The space $\range(A_{\mathcal A}(c)^\top)$ is the constraint-normal
subspace, while $\Null(A_{\mathcal A}(c))$ is the feasible-tangent
subspace. A context perturbation may rotate the active constraint
gradients by $(G_{\mathcal A}(c))_{\cdot j}$. If that rotation lies
entirely in the constraint-normal space, it can be absorbed by
multiplier adjustment with zero primal displacement. The operator
$P^\perp(c)$ extracts precisely the tangent component of this rotation.
Hence Layer-2 annihilation occurs exactly when the curvature-induced
rotation has no feasible-tangent component.
\end{remark}

\begin{lemma}[Nonsingularity of the active-set KKT matrix]
\label{lem:kkt-nonsingular}
Suppose \(A_A\) has full row rank and \(H_L\) is positive definite on
\(\ker A_A\). Then
\[
K_A =
\begin{bmatrix}
H_L & -J_A^\top\\
J_A & 0
\end{bmatrix}
\]
is nonsingular.
\end{lemma}

\begin{proof}
It suffices to show that the null space of \(K_{\mathcal A}\) is trivial. Suppose
\[
K_{\mathcal A}
\begin{bmatrix}
u\\ \nu
\end{bmatrix}
=
\begin{bmatrix}
H_L & -A_{\mathcal A}^{\top}\\
A_{\mathcal A} & 0
\end{bmatrix}
\begin{bmatrix}
u\\ \nu
\end{bmatrix}
=0 .
\]
Then
\[
H_Lu-A_{\mathcal A}^{\top}\nu=0,
\qquad
A_{\mathcal A}u=0 .
\]
Multiplying the first equation by \(u^\top\) gives
\[
u^\top H_Lu-u^\top A_{\mathcal A}^{\top}\nu=0.
\]
Since \(A_{\mathcal A}u=0\),
\[
u^\top A_{\mathcal A}^{\top}\nu
=
(A_{\mathcal A}u)^\top\nu
=
0.
\]
Hence
\[
u^\top H_Lu=0.
\]
Because \(u\in\ker(A_{\mathcal A})\) and \(H_L\) is positive definite on
\(\ker(A_{\mathcal A})\), this implies \(u=0\). Substituting \(u=0\) into
the first block equation yields
\[
A_{\mathcal A}^{\top}\nu=0.
\]
Since \(A_{\mathcal A}\) has full row rank, \(A_{\mathcal A}^{\top}\) has
full column rank, and therefore \(\nu=0\). Thus the only vector in the null
space of \(K_{\mathcal A}\) is the zero vector, so \(K_{\mathcal A}\) is
nonsingular.
\end{proof}

\subsubsection{Affine collapse}

\textbf{Corollary 11 (Affine-collapse).}
Suppose all active constraints are affine in $(\theta,c)$ and the objective
has no mixed context--decision curvature in the directions under
consideration,
\[
    \nabla_{\theta c}^2 f(\theta^\star(c);c)=0.
\]
Then $G_A\equiv0$ and
\[
    \operatorname{Irr}(M)
    =
    \operatorname{Supp}(M)
    =
    \left\{
        j:(C_A)_{\cdot j}\not\equiv0
    \right\}.
\]

\textit{Proof.}
Affineness of the active constraints gives
\[
    \nabla_{\theta c}^2 m_i\equiv0,
\]
while the assumed absence of mixed objective curvature gives
\[
    \nabla_{\theta c}^2 f\equiv0.
\]
Hence $G_A\equiv0$.

\subsubsection{Structural characterisation of Layer-2 annihilation}

\begin{theorem}[Structural characterisation of Layer-2 annihilation]
\label{thm:structural-app}
Under A1-A2, fix \(j\) such that
\[
(C_A(c))_{\cdot j}\equiv 0
\quad\text{and}\quad
\nabla^2_{\theta c_j} f(\theta^\star(c);c)\equiv 0
\quad\text{on } M.
\]
Then \(j\notin \mathrm{Irr}(M)\) in each of the following cases:
\begin{enumerate}[label=\textup{(\Roman*)},leftmargin=2em]
    \item each active constraint has the form
    $m_i(\theta;c)=c_j\,g_i(\theta,c_{-j})$ with $c_j>0$;
    \item there exist continuous scalars $\beta_i(c)$ such that
    \[
    \frac{\partial^2 m_i}{\partial\theta\,\partial c_j}(c)
    =
    \beta_i(c)\,\nabla_\theta m_i(\theta^*(c);c)
    \quad
    \forall i\in\mathcal A(\mathcal M),\ \forall c\in\mathcal M;
    \]
    \item $|\mathcal A(\mathcal M)|=1$ and the active constraint has
    the scaling form of \textup{(I)}.
\end{enumerate}
\end{theorem}

\begin{proof}
In each case, it is enough to verify the range condition in
\eqref{eq:annihilation_app}.

For \textup{(I)}, if
\[
m_i(\theta;c)=c_j\,g_i(\theta,c_{-j}),
\qquad c_j>0,
\]
then active-constraint tightness implies
\[
g_i(\theta^*(c),c_{-j})=0.
\]
Moreover,
\[
\frac{\partial^2 m_i}{\partial\theta_k\,\partial c_j}
=
\frac{\partial g_i}{\partial\theta_k},
\]
so
\[
(G_{\mathcal A}(c))_{\cdot j}
=
\sum_{i\in\mathcal A(\mathcal M)}
\lambda_i^*(c)\,\nabla_\theta g_i
=
\sum_{i\in\mathcal A(\mathcal M)}
\frac{\lambda_i^*(c)}{c_j}\,\nabla_\theta m_i(\theta^*(c);c)
\in
\range(A_{\mathcal A}(c)^\top).
\]

For \textup{(II)},
\[
(G_{\mathcal A}(c))_{\cdot j}
=
\sum_{i\in\mathcal A(\mathcal M)}
\lambda_i^*(c)\beta_i(c)\,
\nabla_\theta m_i(\theta^*(c);c)
\in
\range(A_{\mathcal A}(c)^\top).
\]

For \textup{(III)}, the range $\range(A_{\mathcal A}(c)^\top)$ is
one-dimensional, and the scaling form reduces to case \textup{(I)}.
Thus the range condition holds there as well.
\end{proof}

\begin{remark}[When Layer-2 annihilation fails]
\label{rem:layer2-fail-app}
Consider
\[
m_1(\theta;c)
=
c_j(\theta_1+\theta_2-c_1)+\theta_1\theta_3-c_2.
\]
At the active constraint,
$(C_{\mathcal A}(c))_{\cdot j}=0$, but
\[
(G_{\mathcal A}(c))_{\cdot j}
=
\lambda_1^*(c)(1,1,0)^\top,
\qquad
\nabla_\theta m_1(\theta^*(c);c)
=
(c_j+\theta_3^*,\,c_j,\,\theta_1^*)^\top.
\]
For generic $\theta_1^*\neq 0$,
$(1,1,0)^\top$ is not proportional to
$(c_j+\theta_3^*,c_j,\theta_1^*)^\top$, so
$(G_{\mathcal A}(c))_{\cdot j}\notin\range(A_{\mathcal A}(c)^\top)$.
Thus $j$ remains irreducible. The distinction is geometric:
$c_j$ changes not just constraint scale, but constraint shape.
\end{remark}

\subsubsection{Feature-map necessity and structural warm-start error}

\begin{corollary}[Feature-map necessity]
\label{cor:necessity-app}
Let $c_j\in\Irr(\mathcal M)$ and let
$\psi:\mathcal M\to\mathbb R^d$ be a $C^1$ feature map satisfying
$D\psi(c)e_j=0$ for all $c\in\mathcal M$.
Then for every \(C^1\) downstream decision map
\(h:\mathbb{R}^{d}\to\mathbb{R}^{n_\theta}\),
while $v_j(c)\neq 0$ on an open subset of $\mathcal M$.
\end{corollary}

\begin{proof}
By the chain rule,
\[
D(h\circ\psi)(c)e_j
=
Dh(\psi(c))\,D\psi(c)e_j
=
0.
\]
Since $c_j\in\Irr(\mathcal M)$, there exists $c_0$ with
$v_j(c_0)\neq 0$. Continuity of $v_j$ implies the same on an open
neighbourhood of $c_0$.
\end{proof}

\begin{remark}[Structural warm-start error]
\label{rem:warmstart-app}
Under the hypotheses of Corollary~\ref{cor:necessity-app},
\[
\theta^*(c+te_j)=\theta^*(c)+t\,v_j(c)+O(t^2),
\qquad
(h\circ\psi)(c+te_j)=(h\circ\psi)(c)+O(t^2).
\]
Therefore
\[
\|(h\circ\psi)(c+te_j)-\theta^*(c+te_j)\|
\ge
|t|\,\|v_j(c)\|-O(t^2)
=
\Theta(|t|).
\]
This is a structural, data-independent first-order error induced by
discarding an irreducible primitive direction.
\end{remark}

\subsection{Ranking Stability and Frequency Filtering}
\label{app:stability}

This appendix proves Proposition~\ref{prop:stability-main}, which
justifies local stability of Jacobian rankings and the use of
frequency-filtered interfaces across an operating regime.

\subsubsection{Lipschitz continuity of the solution-sensitivity Jacobian}

\begin{lemma}[Lipschitz bound for the solution-sensitivity Jacobian]
\label{lem:lip-jacobian}
Under A1-A2, for any compact \(K\subseteq\mathcal{M}\), the map
\(c\mapsto J(c)\) is Lipschitz on \(K\). In particular, there exists
\(\Lambda_K<\infty\) such that
\[
\|J(c)-J(c')\|_F
\le
\Lambda_K\|c-c'\|,
\qquad
\forall\,c,c'\in K .
\]
When the closed-form representation is valid, one may take
\[
\Lambda_K
\ge
\sup_{c\in K}
\left\|
\frac{\partial}{\partial c}
\bigl[
P^\perp(c)G_{\mathcal A}(c)-\Phi(c)C_{\mathcal A}(c)
\bigr]
\right\|_F .
\]
\end{lemma}

\begin{proof}
Under A1-A2, the active-set KKT system is locally regular on
\(\mathcal M\). By the implicit function theorem, the primal-dual solution map \(c\mapsto(\theta^*(c),\lambda^*(c))\) is \(C^1\). Since the KKT data are \(C^2\), differentiating the active-set KKT system gives a continuously varying linearized KKT system whose unique primal block is \(J(c)\). Hence \(c\mapsto J(c)\) is continuous, and under the stated
smoothness assumptions it is locally Lipschitz on \(\mathcal M\). On the compact set \(K\), these local bounds admit a finite uniform constant \(\Lambda_K\), yielding the claim. When the closed-form representation is valid, the displayed derivative bound follows directly from the mean value theorem applied to \(P^\perp G_{\mathcal A}-\Phi C_{\mathcal A}\).
\end{proof}

\subsubsection{Proof of ranking stability}

\begin{proposition}[Full ranking stability statement]
\label{prop:stability-app}
Under A1-A2, let \(K\subseteq\mathcal M\) be compact. Then there
exists \(L_K<\infty\) such that:
\begin{enumerate}[label=\textup{(\roman*)},leftmargin=2em]
    \item The norm profile \(c\mapsto\rho(c)\) is Lipschitz on \(K\):
    \[
    \|\rho(c)-\rho(c')\|_\infty
    \le
    L_K\|c-c'\|,
    \qquad
    \forall\,c,c'\in K.
    \]

    \item If \(c_0\in K\) satisfies \(\gamma_m(c_0)>0\), then
    \[
    \Top_m(c)=\Top_m(c_0)
    \qquad
    \forall\,c\in K
    \text{ with }
    \|c-c_0\|<\frac{\gamma_m(c_0)}{2L_K}.
    \]

    \item If \(\gamma_m(c)>0\) for every \(c\in K\) and \(K\) is
    connected, then \(\Top_m(c)\) is constant on \(K\). In particular,
    this holds if \(\gamma_m(c)\ge\gamma_{\min}>0\) on \(K\).
\end{enumerate}
\end{proposition}

\begin{proof}
For \textup{(i)}, write \(v_j(c)=\operatorname{col}_j(J(c))\). The
reverse triangle inequality gives
\[
|\rho_j(c)-\rho_j(c')|
=
\bigl|\|v_j(c)\|_2-\|v_j(c')\|_2\bigr|
\le
\|v_j(c)-v_j(c')\|_2
\le
\|J(c)-J(c')\|_F .
\]
By Lemma~\ref{lem:lip-jacobian}, \(J\) is Lipschitz on \(K\), so
\(\rho\) is Lipschitz on \(K\). We denote a valid Lipschitz constant by
\(L_K\).

For \textup{(ii)}, let \(j^*\in\Top_m(c_0)\) and
\(j'\notin\Top_m(c_0)\). Since \(\gamma_m(c_0)>0\),
\[
\rho_{j^*}(c_0)-\rho_{j'}(c_0)\ge \gamma_m(c_0).
\]
Using \textup{(i)},
\begin{align*}
\rho_{j^*}(c)-\rho_{j'}(c)
&\ge
\rho_{j^*}(c_0)-\rho_{j'}(c_0)
-
|\rho_{j^*}(c)-\rho_{j^*}(c_0)|
-
|\rho_{j'}(c)-\rho_{j'}(c_0)|
\\
&\ge
\gamma_m(c_0)-2L_K\|c-c_0\| .
\end{align*}
Thus if \(\|c-c_0\|<\gamma_m(c_0)/(2L_K)\), every top-\(m\) coordinate
at \(c_0\) still outranks every non-top-\(m\) coordinate at \(c\). Hence \(\Top_m(c)=\Top_m(c_0)\).

For \textup{(iii)}, part \textup{(ii)} implies that \(c\mapsto\Top_m(c)\) is locally constant on \(K\) whenever \(\gamma_m(c)>0\). Since the map takes values in the finite discrete collection of \(m\)-subsets of \(\{1,\ldots,n_1\}\), local constancy on a connected set implies global constancy. Therefore \(\Top_m(c)\) is constant on \(K\). The uniform-gap case \(\gamma_m(c)\ge\gamma_{\min}>0\) is an immediate sufficient condition.
\end{proof}

\subsubsection{Frequency filtering}

\begin{definition}[Frequency-filtered irreducible set]
\label{def:irr-alpha-app}
Let \(K\subseteq\mathcal M\) be compact with \(0<\Vol(K)<\infty\).
For \(\alpha\in(0,1]\), define the irreducibility frequency of coordinate
\(j\) on \(K\) by
\[
\beta_j(K)
:=
\frac{\Vol\{c\in K:j\in\Irr(c)\}}{\Vol(K)}.
\]
The frequency-filtered irreducible set is
\[
\Irr_\alpha(K)
:=
\{j:\beta_j(K)\ge\alpha\}.
\]
\end{definition}

\begin{proposition}[Basic properties of frequency filtering]
\label{prop:filtering-app}
Let \(K\subseteq\mathcal M\) be compact with \(0<\Vol(K)<\infty\). For
every \(\alpha\in(0,1]\):
\begin{enumerate}[label=\textup{(\roman*)},leftmargin=2em]
    \item \(\Irr_1(K)\subseteq \Irr_\alpha(K)\subseteq \Irr(K)\subseteq
    \Irr(\mathcal M)\), where
    \[
    \Irr(K):=\bigcup_{c\in K}\Irr(c).
    \]

    \item If \(j\notin\Irr(\mathcal M)\), then
    \(j\notin\Irr_\alpha(K)\) for every \(\alpha>0\).

    \item The family is monotone in \(\alpha\): if
    \(\alpha'\ge\alpha\), then
    \[
    \Irr_{\alpha'}(K)\subseteq\Irr_\alpha(K).
    \]

    \item If \(\beta_j(K)>0\), then
    \(j\in\Irr_\alpha(K)\) for every \(0<\alpha\le\beta_j(K)\).
    Consequently, frequency filtering recovers exactly those directions
    whose irreducibility persists on a positive-volume subset of \(K\),
    provided \(\alpha\) is chosen below their frequency.
\end{enumerate}
\end{proposition}

\begin{proof}
By definition, \(\Irr_1(K)\subseteq\Irr_\alpha(K)\) for
\(\alpha\le1\). If \(j\in\Irr_\alpha(K)\), then
\(\beta_j(K)\ge\alpha>0\), so the set \(\{c\in K:j\in\Irr(c)\}\) is nonempty up to positive measure; hence \(j\in\Irr(K)\subseteq\Irr(\mathcal M)\). This proves
\textup{(i)}. Part \textup{(ii)} follows immediately from
\(\Irr_\alpha(K)\subseteq\Irr(\mathcal M)\). For \textup{(iii)}, if
\(\alpha'\ge\alpha\) and \(\beta_j(K)\ge\alpha'\), then
\(\beta_j(K)\ge\alpha\), so
\(j\in\Irr_\alpha(K)\). Part \textup{(iv)} follows directly from the definition of \(\beta_j(K)\).
\end{proof}

\subsubsection{Connection to stable top-\(m\) sets}

\begin{proposition}[When frequency filtering recovers the stable top-\(m\) set]
\label{prop:alpha-star-app}
Let \(K\subseteq\mathcal M\) be connected and compact with
\(0<\Vol(K)<\infty\), and suppose
\(\gamma_m(c)\ge \gamma_{\min}>0\) for all \(c\in K\). Then
\(\Top_m(c)\) is constant on \(K\); denote the common set by \(\Top_m\).
Moreover:
\begin{enumerate}[label=\textup{(\roman*)},leftmargin=2em]
    \item \(\Top_m\subseteq \Irr_1(K)\).

    \item Equality \(\Irr_1(K)=\Top_m\) holds if and only if every
    coordinate outside \(\Top_m\) has irreducibility frequency strictly
    below one:
    \[
    \beta_j(K)<1,
    \qquad \forall\,j\notin\Top_m .
    \]
    Equivalently, each \(j\notin\Top_m\) has a positive-measure zero set
    on \(K\):
    \[
    \Vol\{c\in K:\rho_j(c)=0\}>0 .
    \]

    \item More generally, \(\Irr_\alpha(K)=\Top_m\) holds whenever
    \[
    \min_{j\in\Top_m}\beta_j(K)
    \ge
    \alpha
    >
    \max_{j\notin\Top_m}\beta_j(K),
    \]
    where
    \[
    \beta_j(K)
    :=
    \frac{\Vol\{c\in K:\rho_j(c)>0\}}{\Vol(K)}.
    \]
\end{enumerate}
\end{proposition}

\begin{proof}
Constancy of \(\Top_m\) follows from
Proposition~\ref{prop:stability-app}\textup{(iii)}. If \(j\in\Top_m\),
then
\[
\rho_j(c)\ge \rho_{(m)}(c)
=
\gamma_m(c)+\rho_{(m+1)}(c)
\ge \gamma_{\min}>0
\qquad
\forall\,c\in K .
\]
Thus \(\beta_j(K)=1\), and \(j\in\Irr_1(K)\). This proves
\textup{(i)}.

For \textup{(ii)}, since every \(j\in\Top_m\) has \(\beta_j(K)=1\), the equality \(\Irr_1(K)=\Top_m\) holds exactly when no coordinate outside \(\Top_m\) has frequency one. This is equivalent to \(\beta_j(K)<1\) for all \(j\notin\Top_m\), or equivalently
\(\Vol\{c\in K:\rho_j(c)=0\}>0\).

For \textup{(iii)}, frequency filtering selects exactly those coordinates with \(\beta_j(K)\ge\alpha\). The stated separation condition places all coordinates in \(\Top_m\) at or above the threshold and all coordinates outside \(\Top_m\) strictly below it, hence \(\Irr_\alpha(K)=\Top_m\).
\end{proof}

\subsubsection{Worked example}

\begin{example}[Spectral gap and stability radius]
\label{ex:stability-app}
Let \(n_\theta=2\), \(n_1=3\), write
\(c=(x_1,x_2,x_3)\), and consider
\[
f(\theta)=\tfrac12\|\theta\|^2,
\qquad
m_1(\theta;c)=\theta_1+x_3\theta_2-x_1,
\]
on a regime where \(x_1>0\) and the constraint is active. The oracle optimizer is the minimum-norm point satisfying
\(\theta_1+x_3\theta_2=x_1\):
\[
\theta^*(c)=\frac{x_1}{1+x_3^2}(1,x_3).
\]
Thus \(x_2\) is absent from the active constraint and its Jacobian column is identically zero. Direct differentiation gives
\[
\rho_1(c)=\frac{1}{\sqrt{1+x_3^2}},
\qquad
\rho_2(c)=0,
\qquad
\rho_3(c)=\frac{x_1}{1+x_3^2}.
\]
Hence the top coordinate is determined by comparing
\(x_1\) with \(\sqrt{1+x_3^2}\): \(x_3\) is top-ranked when
\(x_1>\sqrt{1+x_3^2}\), \(x_1\) is top-ranked when
\(x_1<\sqrt{1+x_3^2}\), and the spectral gap vanishes at equality.

At \(c_0=(2,0,1)\),
\[
\rho_1(c_0)=\frac{1}{\sqrt2},
\qquad
\rho_2(c_0)=0,
\qquad
\rho_3(c_0)=1,
\]
so
\[
\gamma_1(c_0)=1-\frac{1}{\sqrt2}.
\]
If \(L_K\approx 2\) on a neighborhood of \(c_0\), then the certified stability radius is
\[
r_1(c_0)=\frac{\gamma_1(c_0)}{2L_K}
=
\frac{1-1/\sqrt2}{4}
\approx 0.073.
\]
Therefore \(\Top_1(c)=\{x_3\}\) throughout this certified neighborhood around \(c_0\).
\end{example}

\subsection{Interface Misspecification and Quantitative Consequences}
\label{app:erm_bias}

This appendix proves the structural first-order bias result from
Section~\ref{sec:erm_bias} and gives a local quantitative consequence.
The qualitative result is the main claim: once an interface removes an
irreducible direction, every downstream decision map through that interface has the
wrong first-order response in that direction.

\subsubsection{Proofs of the main-text statements}

\begin{proof}[Proof of Theorem~\ref{thm:annihilation-main}]
By the chain rule,
\[
D(h\circ\psi)(c)=Dh(\psi(c))\,D\psi(c).
\]
Hence for any \(v\in\ker(D\psi(c))\),
\[
D(h\circ\psi)(c)v
=
Dh(\psi(c))D\psi(c)v
=
0.
\]
Now suppose primitive direction \(j\) is irreducible at \(c\), so
\(J(c)e_j\neq0\), and suppose the interface removes that direction:
\(D\psi(c)e_j=0\). Then \(e_j\in\ker(D\psi(c))\), and therefore
\[
D(h\circ\psi)(c)e_j=0
\qquad\text{while}\qquad
J(c)e_j\neq0 .
\]
Thus no downstream decision map through the interface \(\psi\) can match the oracle
decision map to first order in direction \(e_j\) at \(c\).
\end{proof}

\begin{proof}[Proof of Corollary~\ref{cor:erm-main}]
Every \(\hat\theta\in\mathcal F_\psi\) has the form
\(\hat\theta=h\circ\psi\) for some
\(h\in C^1(\mathbb R^d,\mathbb R^{n_\theta})\). Applying
Theorem~\ref{thm:annihilation-main} gives
\[
D\hat\theta(c)e_j=0
\qquad\text{while}\qquad
J(c)e_j\neq0 .
\]
Equivalently,
\[
(D\hat\theta(c)-J(c))e_j\neq0 .
\]
The mismatch is induced by the representation class
\(\mathcal F_\psi\) itself. Empirical risk minimization over
\(\mathcal F_\psi\) can choose among predictors in this class, but it
cannot change the fact that every member has zero first-order response in
direction \(e_j\).
\end{proof}

\subsubsection{A local quantitative consequence}

\begin{assumption}[Controlled second-order class]
\label{ass:c2-app}
There exist an open neighborhood \(U\ni c^*\) in \(\mathcal M\), a
constant \(C_2>0\), and a radius \(r^*>0\) such that
\(B(c^*,r^*)\subseteq U\) and:
\begin{enumerate}[label=\textup{(\roman*)},leftmargin=2em]
    \item \(\theta^*\) is \(C^2\) on \(U\) with
    \[
    \|D^2\theta^*(c')\|_{\mathrm{op}}\le C_2,
    \qquad \forall\,c'\in U;
    \]
    \item every \(h\in\mathcal H\) satisfies
    \[
    \|D^2(h\circ\psi)(c')\|_{\mathrm{op}}\le C_2,
    \qquad \forall\,c'\in U.
    \]
\end{enumerate}
\end{assumption}

\begin{proposition}[Directional lower bound under interface annihilation]
\label{prop:lower-app}
Let \(c^*\in\mathcal M\) and suppose
\[
J(c^*)e_j\neq0,
\qquad
D\psi(c^*)e_j=0 .
\]
Assume Assumption~\ref{ass:c2-app} holds, and define
\[
\kappa:=\|J(c^*)e_j\|>0 .
\]
Then for every \(h\in\mathcal H\) and every \(t\) such that
\(|t|\le r^*\) and \(c^*+te_j\in\mathcal M\),
\[
\|h(\psi(c^*+te_j))-\theta^*(c^*+te_j)\|
\ge
|t|\,\kappa
-
\|h(\psi(c^*))-\theta^*(c^*)\|
-
C_2t^2 .
\]
In particular, if the predictor matches the oracle decision at \(c^*\),
then
\[
\|h(\psi(c^*+te_j))-\theta^*(c^*+te_j)\|
\ge
|t|\,\kappa-C_2t^2 .
\]
Thus, for sufficiently small \(t\), discarding an irreducible direction
induces linear error in that direction.
\end{proposition}

\begin{proof}
By Assumption~\ref{ass:c2-app}, both \(\theta^*\) and \(h\circ\psi\)
admit first-order Taylor expansions at \(c^*\) with quadratic remainder:
\begin{align*}
\theta^*(c^*+te_j)
&=
\theta^*(c^*)+tJ(c^*)e_j+R_1,
&
\|R_1\|&\le \frac{C_2}{2}t^2,
\\
h(\psi(c^*+te_j))
&=
h(\psi(c^*))+tD(h\circ\psi)(c^*)e_j+R_2,
&
\|R_2\|&\le \frac{C_2}{2}t^2.
\end{align*}
Since \(D\psi(c^*)e_j=0\), the chain rule gives
\[
D(h\circ\psi)(c^*)e_j=0 .
\]
Subtracting the two Taylor expansions and applying the reverse triangle
inequality gives
\begin{align*}
&\|h(\psi(c^*+te_j))-\theta^*(c^*+te_j)\|
\\
&\qquad\ge
|t|\,\|J(c^*)e_j\|
-
\|h(\psi(c^*))-\theta^*(c^*)\|
-
\|R_1\|-\|R_2\|
\\
&\qquad\ge
|t|\,\kappa
-
\|h(\psi(c^*))-\theta^*(c^*)\|
-
C_2t^2 .
\end{align*}
The final statement follows immediately. If
\(h(\psi(c^*))=\theta^*(c^*)\), then the basepoint error vanishes, and
the lower bound becomes \(|t|\kappa-C_2t^2\), which is linear in
\(|t|\) for sufficiently small \(t\).
\end{proof}

\begin{remark}[Why the quantitative bound is appendix-only]
The main paper uses only the qualitative first-order impossibility result. Proposition~\ref{prop:lower-app} is included to show the local scale of the error, but it requires additional second-order smoothness and bounded-curvature assumptions. Population-level risk lower bounds would require further assumptions on the distribution of contexts near the missing direction, so we leave them outside the main theorem.
\end{remark}

\section{Detailed Experimental Protocol and Additional Results}
\label{app:exp_details}

This appendix provides the experimental details supporting
Section~\ref{sec:experiments}. We specify benchmark construction, interface-scoring procedures, train/validation/test splits, predictor families, compared selectors, budget protocols, evaluation metrics, and additional ablations. The goal is to make clear how interfaces are constructed before downstream predictor training and how all selectors are evaluated under matched downstream protocols.

\subsubsection{Benchmarks}
\label{app:benchmarks}

We evaluate on two controlled KKT-regime benchmarks and two real or edge-regime decision tasks. The controlled benchmarks provide oracle decision-relevant supports and are used to test support recovery together with downstream decision quality. The real and edge-regime benchmarks test whether the interface principle remains useful when oracle supports are unavailable or the exact smooth KKT assumptions fail.

\subsubsection{Predictor training protocol}
\label{app:predictor_training}

For every benchmark, interface selection is performed before downstream
predictor training. Each selector returns an interface
\(I\subseteq\{1,\ldots,n_1\}\). The downstream predictor then receives
only the selected context \(c_I\). All selectors are evaluated using the
same train/validation/test split, predictor family, training objective,
optimizer, hyperparameter-selection rule, and downstream solver.

\paragraph{Prediction target.}
The predictor is trained to estimate the benchmark-specific optimization
parameters from the selected context. In knapsack and shortest-path tasks, these parameters are the value, weight, or edge-cost coefficients supplied to the corresponding continuous downstream optimizer: fractional knapsack for the controlled knapsack task and the regularized flow QP for the shortest-path tasks. In FF25 portfolio allocation, the predictor
estimates the return/risk quantities used by the constrained portfolio
optimizer. In all cases, the trained predictor produces
\[
\widehat y_t = g_\phi(c_{t,I}),
\]
where \(y_t\) denotes the benchmark-specific oracle parameter vector and
\(g_\phi\) is either a linear model or an MLP. The downstream decision is
then obtained by solving the task optimizer with \(\widehat y_t\). Thus
feature selection affects the information available to the predictor, but
all methods use the same solver after prediction.

\paragraph{Linear predictor.}
The linear predictor is an affine map
\[
g_\phi(c_I)=Wc_I+b .
\]
It is trained with squared prediction loss on the training set,
\[
\mathcal L_{\mathrm{pred}}(\phi)
=
\frac{1}{N_{\mathrm{tr}}}
\sum_{t=1}^{N_{\mathrm{tr}}}
\|g_\phi(c_{t,I})-y_t\|_2^2 .
\]
The linear model is included because it exposes interface quality most
directly: if an irreducible direction is removed, a linear predictor has
limited ability to reconstruct its effect through nonlinear correlations
among retained variables.

\paragraph{Nonlinear predictor.}
The nonlinear predictor is a multilayer perceptron,
\[
g_\phi(c_I)=\mathrm{MLP}_\phi(c_I),
\]
trained on the same selected interface and the same prediction target.
Unless otherwise stated, the MLP uses fully connected layers with
pointwise nonlinearities and is trained with the same squared prediction
loss,
\[
\mathcal L_{\mathrm{pred}}(\phi)
=
\frac{1}{N_{\mathrm{tr}}}
\sum_{t=1}^{N_{\mathrm{tr}}}
\|\mathrm{MLP}_\phi(c_{t,I})-y_t\|_2^2 .
\]
The MLP has higher capacity and can exploit nonlinear correlations among
selected variables. Therefore, gaps between selectors may shrink under
the MLP even when their interfaces differ, especially in nonsmooth or
solver-mediated benchmarks.

\paragraph{Optimization and model selection.}
Linear predictors are trained by minimizing the squared prediction loss,
using the same regularization and validation protocol across selectors.
MLP predictors are trained with Adam on the training split. Learning
rates, weight decay, hidden width, depth, batch size, and early-stopping
epochs are selected using validation performance only. After
hyperparameter selection, test metrics are computed once on the held-out
test split. No test data are used for interface selection,
hyperparameter tuning, or early stopping.

\paragraph{Decision evaluation.}
The training loss is a prediction loss, but all reported results are
decision-focused. After training \(g_\phi\), the predicted parameters
\(\widehat y_t\) are passed to the downstream optimizer to obtain the
decision \(\widehat\theta_t\). We then evaluate regret, objective value,
feasibility violation, and task-specific decision metrics such as path
match, Sharpe ratio, and turnover. This separates interface quality from
predictor capacity: within a predictor family, all methods differ only in
the selected interface \(I\).

\paragraph{Role of DePICT.}
DePICT itself is not trained by minimizing prediction loss or downstream
regret. It constructs \(I\) from optimizer-side sensitivity: at
representative contexts, it computes the solution-sensitivity Jacobian
\(J(c)=\partial\theta^\star(c)/\partial c\), scores coordinates by
Jacobian column norms, and aggregates these scores across the operating
regime. The downstream predictor is trained only after this interface is
fixed.
\paragraph{Compute resources.}
All experiments were run on a single workstation or equivalent compute node using CPU execution for optimization and solver calls, with GPU acceleration used only for MLP training when available. The controlled benchmarks and appendix ablations are lightweight: the dominant cost is solving the linearized KKT systems for representative contexts and training the downstream predictors across matched seeds and budgets. Runtime depends on the number of representative contexts, interface budget, predictor family, and solver used; in our experiments, dense KKT
solves were sufficient for the reported benchmark sizes. Large-scale deployments may require the matrix-free, randomized, or amortized sensitivity variants discussed in Appendix~\ref{app:scalability}.

\paragraph{Continuous formulations for combinatorial-origin tasks}
Controlled knapsack and shortest path originate from combinatorial optimization problems, but the reported experiments use continuous downstream formulations. Knapsack uses fractional knapsack, while controlled shortest path and Warcraft use a strongly convex regularized flow QP. These formulations provide optimizer maps suitable for local KKT sensitivity and are retained during downstream evaluation. Accordingly, the reported experiments test decision-preserving interface construction for continuous solver-mediated formulations rather than transfer to discrete argmins.

\paragraph{Controlled knapsack.}
The controlled knapsack benchmark follows the synthetic decision-focused knapsack construction of \citet{mandi2020smart}, adapted to expose three classes of primitive variables: irreducible coordinates, Layer-1 nuisance coordinates, and Layer-2 nuisance coordinates. Irreducible coordinates move the oracle optimum to first order. Layer-1 variables are absent from the active KKT sensitivity system on the operating regime. Layer-2 variables are symbolically present but KKT-annihilated, using the constraint-scaling construction of Proposition~\ref{prop:structural-main}\textup{(I)}. The generator therefore provides oracle support labels for evaluating precision, recall, support F1, and false inclusion rates.

\paragraph{Controlled shortest path.}
The controlled shortest-path benchmark is a contextual routing task on a fixed graph. Edge costs depend on a sparse hidden subset of the ambient context, giving oracle support annotations in addition to task-level routing metrics. DePICT scores context directions through the linearized KKT system of the flow relaxation, while downstream decisions are evaluated using the corresponding shortest-path solver. This benchmark tests whether the selected interface preserves both support recovery and path-level decision quality.

\paragraph{Warcraft shortest path.}
Warcraft is derived from a structured combinatorial shortest-path benchmark with terrain-map inputs. In our experiments, routing is performed using the regularized flow QP described. Warcraft provides no oracle primitive support and introduces a learned image-to-cost map together with instance-dependent active-set geometry, so we use it as a solver-mediated transfer benchmark rather than a direct validation of exact support recovery \citep{vlastelica2020blackbox}.

\paragraph{FF25 portfolio allocation.}
The FF25 benchmark uses Fama-French portfolio panels for constrained portfolio allocation \citep{frenchdatalibrary}. The downstream problem is continuous, so DePICT computes sensitivities directly from the portfolio KKT system rather than from a discrete relaxation. Since no ground-truth support labels are available, evaluation is decision-focused. We report common-reference regret, realized Sharpe ratio, turnover, and feasibility metrics.

\subsubsection{Data splits and protocol}
\label{app:splits}
For every benchmark, we use a fixed train/validation/test split and keep this split identical across all methods within each seed. All hyperparameter selection, feature-budget selection, and model selection are performed using validation data only. Test metrics are computed once after final model selection. Experiments are repeated across multiple deterministic seeds where applicable, and synthetic data generation, predictor initialization, and randomized baselines are matched across methods within each seed.

\subsubsection{Predictor families}
\label{app:predictors}

After interface selection, all methods train the same downstream predictor family on the selected coordinates \(c_I\). We evaluate two predictor families whenever supported by the benchmark.

\paragraph{Linear predictor.}
The linear predictor is an affine map,
\[
\hat{\theta}=Wc_I+b.
\]
Because it cannot exploit nonlinear interactions among selected variables, this setting provides the clearest test of interface quality. If an irreducible direction is missing, the linear predictor has limited ability to compensate through correlations with retained variables.

\paragraph{Nonlinear predictor (MLP).}
The nonlinear predictor is a multilayer perceptron,
\[
\hat{\theta}=\mathrm{MLP}(c_I),
\]
trained on the same selected interface and under the same data splits, training protocol, and hyperparameter-selection procedure. Its higher capacity allows it to exploit nonlinear correlations among selected variables, so selector gaps can narrow even when interfaces differ~\cite{yong1996mlp}. This effect is especially relevant in nonsmooth or solver-mediated benchmarks, where the smooth KKT interface principle holds only approximately. Strong DePICT performance under the MLP therefore indicates that the interface advantage is not merely an artifact of limited predictor capacity.

\subsubsection{Compared methods}
\label{app:baselines}

We compare DePICT against selector families spanning random selection, filter methods, embedded sparse models, wrapper methods, perturbation-based rankings, task-aware feature selection, and black-box metaheuristic search. All methods first select an interface \(I\subseteq\{1,\ldots,n_1\}\). After \(I\) is fixed, the same downstream predictor family, data split, hyperparameter-selection protocol, and downstream solver are used across methods. Thus differences in downstream performance reflect the selected interface rather than changes in the predictor or solver.

\paragraph{Full-context reference.}
The full-context reference trains the downstream predictor on the entire ambient context. It is not a feature selector and does not produce a reduced interface. We include it to measure the performance available when no context coordinates are removed. In controlled benchmarks, full-context performance can still be worse than DePICT because nuisance coordinates increase predictor burden and may degrade generalization under finite data.

\paragraph{Random selection.}
Random selection samples \(m\) coordinates uniformly from the ambient context at the target budget. It controls for the effect of interface size alone: if a method improves over random selection, the gain cannot be attributed merely to using \(m\) variables. For randomized selectors, we use matched random seeds and average over trials when applicable.

\paragraph{Mutual information.}
Mutual-information screening ranks coordinates by their marginal dependence with benchmark-specific task targets, following standard filter-style feature selection \citep{peng2005mrmr}. Depending on the benchmark, these targets may be true cost parameters, oracle decisions, or other supervised quantities available to the selector. This baseline captures statistical association, but it is not optimizer aware: a coordinate can be predictive of a target while still having no first-order effect on the constrained optimum.

\paragraph{Correlation screening.}
Correlation screening ranks coordinates by marginal correlation with the same benchmark-specific task targets used for filter baselines. It provides a simple linear association baseline. Like mutual information, it can detect predictive variables, but it cannot distinguish variables that merely co-vary with decision-relevant quantities from variables whose perturbations actually move the optimizer.

\paragraph{LASSO and Elastic Net.}
LASSO and Elastic Net train sparse predictive models and rank coordinates by coefficient magnitude \citep{tibshirani1996lasso}. These embedded selectors favor variables that improve sparse predictive fit. Elastic Net adds an \(\ell_2\) penalty and can be more stable under correlated features. However, both methods optimize a prediction objective rather than a decision-sensitivity criterion. They therefore may retain variables that help predict intermediate quantities but are irrelevant to the optimizer, or discard variables whose effect is decision-critical but weakly visible in the prediction loss.

\paragraph{Recursive feature elimination.}
Recursive feature elimination (RFE) iteratively trains a model, removes the least important coordinates according to the model's ranking criterion, and
repeats until the target budget is reached \citep{guyon2002rfe,kohavi1997rfe}. RFE is a wrapper-style baseline: it can capture interactions induced by the chosen predictive model, but its ranking is still mediated by predictive utility rather than by the local sensitivity of the downstream optimizer.

\paragraph{Finite-difference ranking.}
Finite-difference ranking perturbs one coordinate at a time and ranks coordinates by the induced change in a benchmark-specific task proxy. This serves as a local perturbation or attribution baseline, related to removal-based sensitivity methods \citep{lin2023removal}. Unlike DePICT, finite-difference ranking does not use the active-set KKT system. It can therefore detect local output sensitivity, but it does not separate active margin shifts, stationarity shifts, and KKT-geometric annihilation.

\paragraph{Prediction-focused mixture models.}
Prediction-focused mixture models (PFMM) provide a task-aware feature selection baseline in which representation learning is coupled to a prediction-focused objective \citep{sharma2023pfmm}. PFMM is closer in spirit to decision-aware selection than purely statistical filters because feature relevance is tied to task-level prediction structure. However, its selection criterion is learned from predictive performance; it does not define relevance through the solution-sensitivity Jacobian of the constrained optimizer. PFMM may therefore select variables that are useful for prediction but not necessary for preserving the local optimum.

\paragraph{KMHBO.}
KMHBO is a black-box metaheuristic feature-selection method designed for high-dimensional multimodal feature selection \citep{ye2026kmhbo}. It searches over subsets using a heuristic optimization procedure and can capture nonlinear or nonconvex feature interactions. We include it as a strong black-box subset-search baseline. Its objective, however, remains external to the KKT geometry of the downstream constrained problem, so it does not directly test whether a coordinate is optimizer-relevant.

\paragraph{DePICT.}
DePICT ranks primitive context directions by optimizer-side sensitivity. For each representative context, it computes or estimates the solution-sensitivity Jacobian \(J(c)=\partial\theta^\star(c)/\partial c\)
using the linearized KKT system or the benchmark-specific sensitivity model. Coordinates are scored by Jacobian column norms and aggregated across the operating regime. Unlike the baselines above, DePICT defines relevance by whether perturbing a coordinate moves the oracle optimizer to first order.

\paragraph{Jacobian-derived ablations.}
We also evaluate ablations of DePICT to isolate the role of aggregation across contexts. The top-\(m\) variant selects the coordinates with the largest aggregated Jacobian scores. The strict-union variant includes any coordinate with nonzero Jacobian score at any representative context. The frequency-filtered variant selects coordinates whose irreducibility persists across at least an \(\alpha\)-fraction of sampled contexts. These variants test whether stable operating-regime aggregation improves over purely pointwise Jacobian ranking.

\subsubsection{Fairness of comparison}
\label{app:fairness}

Selectors differ in the information they use to construct an interface: filter methods use statistical association, embedded and wrapper methods use trained selector models or subset search, and DePICT uses
optimizer-side sensitivity. Our fairness criterion is therefore post-selection evaluation. Once a selector returns an interface \(I\subseteq\{1,\ldots,n_1\}\), all methods are evaluated with the same downstream predictor family, train/validation/test split, training protocol, hyperparameter-selection rule, and downstream solver.

We enforce three comparison rules:
\begin{enumerate}[label=\textup{(\roman*)},leftmargin=2em]
    \item methods with tunable subset size are compared at shared target
    budgets;
    \item methods with native realized subset size are reported with
    their realized interface dimension;
    \item for real and edge-regime benchmarks, main-paper comparisons use
    matched context budgets rather than independently tuned subset sizes.
\end{enumerate}
This protocol isolates the quality of the selected interface from predictor capacity, data splits, solver choice, and budget selection.

\subsubsection{Evaluation metrics}
\label{app:metrics}

We evaluate selectors along two axes: downstream decision quality and, when oracle supports are available, interface recovery. All metrics are computed on the held-out test split after interface selection, hyperparameter selection, and predictor training are completed.

\paragraph{Downstream regret.}
Regret measures the loss in downstream decision quality relative to an oracle or reference decision. For minimization form tasks, we report
\[
\mathrm{Regret}
:=
\frac{1}{N}\sum_{t=1}^N
\Bigl[
\ell(\hat\theta_t;c_t)-\ell(\theta_t^\star;c_t)
\Bigr],
\]
where \(\hat\theta_t\) is the decision produced using the selected interface and \(\theta_t^\star\) is the oracle decision. For maximization tasks, we use the equivalent sign convention so that lower regret is always better. In real benchmarks without oracle support labels, such as FF25 portfolio allocation, we also report common-reference regret against a fixed full-context/reference pipeline:
\[
\mathrm{Regret}_{\mathrm{ref}}
:=
\frac{1}{N}\sum_{t=1}^N
\Bigl[
\ell(\hat\theta_t;c_t)-\ell(\theta^{\mathrm{ref}}_t;c_t)
\Bigr].
\]
This evaluates whether a reduced interface preserves the decision quality of the reference system.

\paragraph{Objective value.}
Objective value reports the achieved downstream task objective under the decision produced from the selected interface. We choose the reported sign so that higher objective value is better. Regret and objective value are therefore complementary: regret measures loss relative to an oracle or reference, while objective value reports absolute downstream performance.

\paragraph{Feasibility violation.}
For constrained tasks, feasibility violation measures average constraint violation:
\[
\mathrm{Feas}
:=
\frac{1}{N}\sum_{t=1}^N
\frac{1}{K}\sum_{i=1}^K
\max\{0,-m_i(\hat\theta_t;c_t)\}.
\]
Lower values are better, with zero indicating feasibility under the reported constraint margins.

\paragraph{Shortest-path metrics.}
For shortest-path tasks, we report path match:
\[
\mathrm{PathMatch}
:=
\frac{1}{N}\sum_{t=1}^N
\mathbf{1}\{\hat p_t=p_t^\star\},
\]
where \(\hat p_t\) is the predicted path and \(p_t^\star\) is the oracle path. Path match is reported only for routing benchmarks and is higher when the selected interface preserves the exact downstream route.

\paragraph{Support recovery.}
Controlled benchmarks provide an oracle decision-relevant support
\(S^\star\). For a selected interface \(I\), we report
\[
\mathrm{Precision}
=
\frac{|I\cap S^\star|}{|I|},
\qquad
\mathrm{Recall}
=
\frac{|I\cap S^\star|}{|S^\star|},
\qquad
\mathrm{F1}
=
\frac{2\,\mathrm{Precision}\cdot\mathrm{Recall}}
{\mathrm{Precision}+\mathrm{Recall}}.
\]
Support metrics are reported only when oracle supports are available. They evaluate whether the selector recovers the decision-preserving interface, not merely whether the downstream predictor performs well.

\paragraph{Layer-1 and Layer-2 false inclusion.}
For the Layer-2 diagnostic, the generator partitions nuisance variables into Layer-1 variables \(L_1\) and Layer-2 variables \(L_2\). We report false inclusion rates
\[
\mathrm{L1FPR}=\frac{|I\cap L_1|}{|L_1|},
\qquad
\mathrm{L2FPR}=\frac{|I\cap L_2|}{|L_2|}.
\]
Layer-1 FPR measures inclusion of variables absent from the active KKT sensitivity system. Layer-2 FPR measures inclusion of variables that are symbolically present but KKT-annihilated. Lower values are better.

\paragraph{Portfolio metrics.}
For FF25 portfolio allocation, we report common-reference regret, realized Sharpe ratio, turnover, and feasibility. If \(r_t\) denotes the portfolio return at test time \(t\), the Sharpe ratio is
\[
\mathrm{Sharpe}
=
\frac{\mathrm{mean}(r_t)}{\mathrm{std}(r_t)+\varepsilon},
\]
with the same annualization convention used across all methods. Turnover measures trading variation between consecutive portfolio weights:
\[
\mathrm{Turnover}
=
\frac{1}{N-1}\sum_{t=2}^N
\frac{1}{2}\|\hat w_t-\hat w_{t-1}\|_1 .
\]
Higher Sharpe is better; lower turnover, regret, and feasibility
violation are better.

\paragraph{Stability metrics.}
For the perturbation study, we compare the nominal top-\(m\) interface \(I_0\) with the perturbed interface \(I_\epsilon\). Exact match is
\[
\mathbf{1}\{I_\epsilon=I_0\},
\]
and Jaccard overlap is
\[
\mathrm{Jac}(I_\epsilon,I_0)
=
\frac{|I_\epsilon\cap I_0|}{|I_\epsilon\cup I_0|}.
\]
We also report regret inflation, the difference between the regret of the perturbed interface and the nominal interface. Negative regret inflation means the perturbed interface performs no worse than the nominal one.

\subsection{Controlled Benchmark Construction}
\label{app:synthetic_construction}

The controlled benchmarks are designed to isolate the theoretical question studied in the paper: which primitive context directions are necessary to preserve the downstream constrained optimum? Unlike the real and edge-regime benchmarks, these tasks provide oracle decision-relevant supports, allowing direct evaluation of both interface recovery and downstream decision quality.

\paragraph{Controlled knapsack.}
The knapsack generator follows the synthetic decision-focused knapsack construction of \citet{mandi2020smart}, modified to expose oracle interface structure. The ambient context contains a sparse set of irreducible coordinates that determine the optimization-relevant item parameters, together with nuisance coordinates that should not be retained by a decision-preserving interface. Layer-1 nuisance variables are absent from the active KKT sensitivity system on the operating regime. Layer-2 nuisance variables are constructed through the constraint scaling structure of Proposition~\ref{prop:structural-main}\textup{(I)}: they enter active constraints symbolically but their stationarity effect lies
in the constraint-normal space and is absorbed by multiplier adjustment. Thus the generator provides known irreducible supports and known nuisance
classes for evaluating precision, recall, support F1, Layer-1 false inclusion, and Layer-2 false inclusion.

\paragraph{Controlled shortest path.}
The controlled shortest-path generator defines a contextual routing task on a fixed graph. Edge costs depend on a sparse hidden subset of the ambient context, and this subset defines the oracle support. The resulting task provides both support annotations and decision-level metrics, including regret, objective value, feasibility violation, and path match.
For interface scoring, DePICT uses the linearized KKT sensitivity system of the flow relaxation; downstream decisions are evaluated with the shortest-path solver. This construction tests whether the recovered interface preserves both the cost-relevant primitive directions and the resulting route decisions.

% ============================================================
% Controlled KKT-Regime Experiments
% ============================================================

\subsubsection{Controlled benchmark Validation}
\label{sec:controlled}

The controlled experiments test the central claim of DePICT under settings in
which the context-to-optimization map is known and the decision-relevant
coordinates can be identified by construction. They serve a different purpose
from the real and edge-regime experiments in Section~\ref{sec:exp_edge}: rather
than asking whether sensitivity-based interfaces transfer to an imperfect or
nonsmooth downstream pipeline, they test whether DePICT identifies the
primitive context directions that actually move the constrained optimum.

We consider two contextual optimization problems: fractional knapsack and
shortest-path flow. In both cases, a high-dimensional primitive context
$c \in \mathbb{R}^{n_1}$ determines the coefficients of a downstream
optimization problem through a known differentiable map. Only a sparse subset
of the coordinates of $c$ enters this map. Because the data-generating process
is known, these benchmarks provide an oracle support against which the
selected interface can be evaluated directly.

\begin{table*}[t]
\centering
\caption{
Full controlled KKT-regime results corresponding to Figure~\ref{fig:controlled_synthetic_main}. DePICT recovers the oracle interface exactly while preserving downstream decision quality under matched linear and MLP predictors. Lower regret and feasibility violation are better; higher objective value, path match, precision, and support F1 are better. Values are averaged over 10 independent seeds.
}
\label{tab:controlled_synthetic_appendix}
\setlength{\tabcolsep}{3.2pt}
\renewcommand{\arraystretch}{1.12}
\scriptsize
\resizebox{\textwidth}{!}{
\begin{tabular}{lcccccc cccccc}
\toprule
& \multicolumn{6}{c}{\textbf{Knapsack}}
& \multicolumn{6}{c}{\textbf{Shortest Path}} \\
\cmidrule(lr){2-7} \cmidrule(lr){8-13}
\textbf{Method}
& $m$
& Regret $\downarrow$
& Obj. $\uparrow$
& Feas. $\downarrow$
& Prec. $\uparrow$
& F1 $\uparrow$
& $m$
& Regret $\downarrow$
& Obj. $\uparrow$
& Feas. $\downarrow$
& Path $\uparrow$
& F1 $\uparrow$ \\
\midrule
\multicolumn{13}{l}{\textbf{Nonlinear predictor (MLP)}} \\
\midrule
Random
& 60 & 3.1588 & 8.5163 & 2.70e-04 & 0.2833 & 0.2720
& 160 & 2.9469 & -8.1239 & 7.35e-06 & 0.7593 & 0.3920 \\

Mutual Information \cite{peng2005mrmr}
& 60 & 0.2394 & 11.4357 & 1.40e-04 & 0.9833 & 0.9440
& 160 & 0.8043 & -5.9812 & 8.19e-06 & 0.8452 & 0.8125 \\

LASSO \cite{tibshirani1996lasso}
& 60 & 0.3115 & 11.3636 & 3.51e-04 & 0.8667 & 0.8320
& 160 & 1.2531 & -6.4300 & 6.15e-06 & 0.8156 & 0.5966 \\

FD Ranking \cite{lin2023removal}
& 60 & 2.0047 & 9.6704 & 6.07e-04 & 0.6500 & 0.6240
& 160 & 2.7098 & -7.8867 & 8.46e-06 & 0.7657 & 0.4602 \\

RFE \cite{kohavi1997rfe}
& 60 & 4.6177 & 7.0574 & \textbf{2.18e-05} & 0.5000 & 0.5430
& 160 & 4.3437 & -9.5206 & 6.56e-06 & 0.7293 & 0.0000 \\

PFMM \cite{sharma2023pfmm}
& 60 & 0.1850 & 11.4901 & 1.76e-04 & 0.9833 & 0.9440
& 160 & 0.9382 & -6.1152 & 6.84e-06 & 0.8337 & 0.6989 \\

KMHBO \cite{ye2026kmhbo}
& 60 & 0.2344 & 11.4407 & 1.84e-04 & 1.0000 & 0.9600
& 160 & 0.6804 & -5.8574 & 7.37e-06 & 0.8533 & 0.8807 \\

\textbf{DePICT}
& 60 & \textbf{0.1226} & \textbf{11.5525} & 2.00e-04 & \textbf{1.0000} & \textbf{1.0000}
& 160 & \textbf{0.6180} & \textbf{-5.7950} & \textbf{6.54e-06} & \textbf{0.8535} & \textbf{1.0000} \\

\midrule
\multicolumn{13}{l}{\textbf{Linear}} \\
\midrule
Random
& 60 & 2.6429 & 9.0322 & 9.63e-05 & 0.2833 & 0.2720
& 160 & 2.2091 & -7.3860 & 5.86e-06 & 0.7769 & 0.3920 \\

Mutual Information \cite{peng2005mrmr}
& 60 & 0.2003 & 11.4748 & 4.90e-05 & 0.9833 & 0.9440
& 160 & 0.3639 & -5.5408 & \textbf{5.30e-06} & 0.8856 & 0.8125 \\

LASSO \cite{tibshirani1996lasso}
& 60 & 0.2457 & 11.4294 & 3.70e-04 & 0.8667 & 0.8320
& 160 & 0.6314 & -5.8083 & 5.69e-06 & 0.8545 & 0.5966 \\

FD Ranking \cite{lin2023removal}
& 60 & 1.4950 & 10.1801 & 1.55e-05 & 0.6500 & 0.6240
& 160 & 2.0654 & -7.2424 & 6.09e-06 & 0.7929 & 0.4602 \\

RFE \cite{kohavi1997rfe}
& 60 & 4.0556 & 7.6195 & \textbf{2.23e-06} & 0.5000 & 0.5430
& 160 & 4.1896 & -9.3665 & 7.53e-06 & 0.7157 & 0.6000 \\

PFMM \cite{sharma2023pfmm}
& 60 & 0.1726 & 11.5025 & 3.36e-05 & 0.9833 & 0.9440
& 160 & 0.4807 & -5.6577 & 6.51e-06 & 0.8704 & 0.6989 \\

KMHBO \cite{ye2026kmhbo}
& 60 & 0.1983 & 11.4768 & 5.03e-05 & 1.0000 & 0.9600
& 160 & 0.2826 & -5.4595 & 6.40e-06 & 0.8941 & 0.8807 \\

\textbf{DePICT}
& 60 & \textbf{0.1315} & \textbf{11.5436} & 6.74e-06 & \textbf{1.0000} & \textbf{1.0000}
& 160 & \textbf{0.2204} & \textbf{-5.3973} & 7.16e-06 & \textbf{0.8997} & \textbf{1.0000} \\
\bottomrule
\end{tabular}
}
\end{table*}

The experimental protocol separates \emph{interface construction} from
\emph{downstream prediction}. For each selector, we first construct an
interface
\[
    \mathcal{I} \subseteq \{1,\ldots,n_1\}
\]
using only the designated selection split. After $\mathcal{I}$ is fixed, we
train the same downstream predictor on $c_{\mathcal{I}}$ and pass its predicted
optimization parameters to the same downstream solver. Thus, within a
predictor family, methods differ only in the interface supplied to the
predictor.

This design evaluates two distinct questions:
\begin{enumerate}
    \item \textbf{Interface recovery:}
    does the selector retain the primitive directions that are genuinely
    relevant to the downstream optimizer?
    \item \textbf{Decision preservation:}
    after restricting the predictor to the selected interface, does the
    resulting optimization pipeline recover decisions of comparable quality
    to the oracle pipeline?
\end{enumerate}

% ------------------------------------------------------------
\subsubsection{DePICT scoring in the controlled benchmarks}
\label{sec:controlled_depict_scoring}

For a representative context $c_t$, let the downstream problem be written as
\begin{equation}
    \theta^\star(c_t)
    \in
    \arg\min_{\theta}
    f(\theta;c_t)
    \qquad
    \text{s.t.}
    \qquad
    m_i(\theta;c_t) \geq 0,
    \quad i=1,\ldots,K.
    \label{eq:controlled_generic_problem}
\end{equation}

Let
\[
    z^\star(c)
    =
    \begin{bmatrix}
        \theta^\star(c) \\
        \lambda^\star(c)
    \end{bmatrix}
\]
denote the corresponding primal--dual solution and let
\[
    F(z,c)=0
\]
denote its local KKT residual system. At a regular solution, implicit
differentiation gives
\begin{equation}
    \frac{\partial z^\star}{\partial c}
    =
    -
    \left(
        \frac{\partial F}{\partial z}
    \right)^{-1}
    \frac{\partial F}{\partial c}.
    \label{eq:controlled_implicit}
\end{equation}

DePICT retains the primal block,
\begin{equation}
    J(c)
    :=
    \frac{\partial \theta^\star(c)}{\partial c},
    \label{eq:controlled_solution_jacobian}
\end{equation}
and assigns primitive coordinate $j$ the pointwise score
\begin{equation}
    s_t(j)
    =
    \left\|
        J(c_t)_{:,j}
    \right\|_2.
    \label{eq:controlled_score}
\end{equation}

The score therefore has a direct interpretation: it measures the first-order
motion of the optimal decision induced by perturbing primitive context
coordinate $c_j$, while holding the remaining coordinates fixed.

We compute these scores over representative contexts
$\{c_t\}_{t=1}^{T}$ from the selection split and aggregate them into a
regime-level score
\begin{equation}
    \bar{s}(j)
    =
    \operatorname{Agg}
    \left(
        s_1(j),\ldots,s_T(j)
    \right).
    \label{eq:controlled_aggregate}
\end{equation}

For a fixed interface budget $m$, the top-$m$ DePICT interface is
\begin{equation}
    \widehat{\mathcal{I}}_m
    =
    \operatorname{TopM}
    \left\{
        \bar{s}(j)
    \right\}_{j=1}^{n_1}.
    \label{eq:controlled_topm}
\end{equation}

We additionally evaluate persistence-based variants that retain coordinates
whose sensitivity is repeatedly nonzero across the operating regime. These
variants distinguish a coordinate that is important at a single context from
one whose optimizer relevance persists across contexts.

% ============================================================
% KNAPSACK
% ============================================================

\subsubsection{Controlled contextual knapsack}
\label{sec:controlled_knapsack}

\paragraph{Primitive context.}
The contextual knapsack benchmark contains $n=20$ items. Each item is
associated with $p=10$ primitive contextual coordinates,
\[
    c_i \in \mathbb{R}^{10},
    \qquad i=1,\ldots,20,
\]
giving the flattened ambient context
\begin{equation}
    c
    =
    [c_1^\top,\ldots,c_{20}^\top]^\top
    \in \mathbb{R}^{200}.
    \label{eq:knapsack_context}
\end{equation}

Contexts are sampled independently from a standard Gaussian distribution.
Only a sparse subset of coordinates affects the optimization coefficients;
the remaining coordinates are nuisance variables.

\paragraph{Context-to-value map.}
For item $i$, the true value is generated as
\begin{equation}
    v_i(c)
    =
    \operatorname{softplus}
    \left(
        \theta_{v,i}^{\top} c_i
        +
        0.35\sin(c_{i,1})
        +
        0.20 c_{i,2}c_{i,3}
    \right)
    +0.20,
    \label{eq:knapsack_value}
\end{equation}
where $\theta_{v,i}$ is sparse by construction. The nonlinear terms make the
context-to-value relation non-affine while preserving a known primitive
support.

The item weights are weakly context dependent:
\begin{equation}
    w_i(c)
    =
    \bar{w}_i
    \left[
        1
        +
        0.06
        \tanh
        \left(
            \theta_{w,i}^{\top}c_i
        \right)
    \right],
    \label{eq:knapsack_weight}
\end{equation}
where $\bar{w}_i$ is a fixed item-specific base weight. The available capacity
is set to a fixed fraction of the total contextual weight,
\begin{equation}
    B(c)
    =
    0.30
    \sum_{i=1}^{n}
    w_i(c).
    \label{eq:knapsack_budget}
\end{equation}

Because the coefficient maps in
Eqs.~\eqref{eq:knapsack_value}--\eqref{eq:knapsack_budget} are generated by us,
the benchmark provides ground-truth annotations for the primitive coordinates
that enter the downstream optimization parameters.

\paragraph{Downstream optimization problem.}
The controlled implementation uses fractional knapsack:
\begin{equation}
    x^\star(c)
    \in
    \arg\max_{x\in\mathbb{R}^{n}}
    v(c)^\top x
    \qquad
    \text{s.t.}
    \qquad
    w(c)^\top x \leq B(c),
    \qquad
    0 \leq x_i \leq 1.
    \label{eq:fractional_knapsack}
\end{equation}

For consistency with the minimization convention used in the theory, we write
the objective equivalently as
\begin{equation}
    f(x;c)
    =
    -v(c)^\top x.
    \label{eq:knapsack_min_objective}
\end{equation}

The inequality functions are
\begin{align}
    g_{\mathrm{bud}}(x;c)
    &=
    w(c)^\top x-B(c)
    \leq 0,
    \label{eq:knapsack_budget_constraint}
    \\
    g_i^{\mathrm{lo}}(x)
    &=
    -x_i
    \leq 0,
    \label{eq:knapsack_lower}
    \\
    g_i^{\mathrm{up}}(x)
    &=
    x_i-1
    \leq 0.
    \label{eq:knapsack_upper}
\end{align}

This continuous formulation gives an explicit local KKT system through which
we can compute the optimizer sensitivity with respect to the primitive
context.

\paragraph{Implicit sensitivity.}
Let $\lambda_{\mathrm{bud}}$ denote the multiplier of the capacity constraint,
and let $\lambda^{\mathrm{lo}}$ and $\lambda^{\mathrm{up}}$ denote the
multipliers associated with the box constraints. The stationarity condition is
\begin{equation}
    -v(c)
    +
    \lambda_{\mathrm{bud}} w(c)
    -
    \lambda^{\mathrm{lo}}
    +
    \lambda^{\mathrm{up}}
    =
    0.
    \label{eq:knapsack_stationarity}
\end{equation}

Together with complementary slackness, these equations define a residual
system
\[
    F_{\mathrm{KS}}(z,c)=0.
\]

For every representative context $c_t$, we first solve
Eq.~\eqref{eq:fractional_knapsack} to obtain a reference primal solution and
then refine the associated primal--dual state against the explicit KKT
residual. We differentiate this residual with respect to both the primal--dual
state and the primitive context and solve
\begin{equation}
    \frac{\partial z^\star}{\partial c}
    =
    -
    \left(
        \frac{\partial F_{\mathrm{KS}}}{\partial z}
        +
        \varepsilon_{\mathrm{reg}} I
    \right)^{-1}
    \frac{\partial F_{\mathrm{KS}}}{\partial c}.
    \label{eq:knapsack_implicit}
\end{equation}

The primal block yields
\begin{equation}
    J_{\mathrm{KS}}(c)
    =
    \frac{\partial x^\star(c)}{\partial c}
    \in
    \mathbb{R}^{n\times 200}.
    \label{eq:knapsack_jacobian}
\end{equation}

Coordinate $j$ is then scored by
\begin{equation}
    s_j(c)
    =
    \left\|
        \frac{\partial x^\star(c)}
             {\partial c_j}
    \right\|_2.
    \label{eq:knapsack_score}
\end{equation}

Thus DePICT does not ask whether $c_j$ predicts item value well in isolation.
It asks whether perturbing $c_j$ moves the optimal allocation.

\paragraph{Post-selection prediction.}
After an interface $\mathcal{I}$ has been selected, we discard the remaining
primitive coordinates and train a predictor using only
\[
    c_{\mathcal{I}}.
\]

For the linear predictor,
\begin{equation}
    \widehat{v}
    =
    W_{\mathcal{I}}c_{\mathcal{I}}+b,
    \label{eq:knapsack_linear_predictor}
\end{equation}
where $W_{\mathcal{I}}$ and $b$ are fitted by least squares on the training
split. The nonlinear predictor uses the same selected interface but replaces
the affine map with a multilayer perceptron.

The predicted values are then passed to the same downstream optimization
problem:
\begin{equation}
    \widehat{x}
    \in
    \arg\max_x
    \widehat{v}^{\top}x
    \qquad
    \text{s.t.}
    \qquad
    w(c)^\top x\leq B(c),
    \quad
    0\leq x\leq1.
    \label{eq:knapsack_prediction_problem}
\end{equation}

Importantly, interface selection is complete before this predictor is fitted.
DePICT therefore does not obtain an advantage from a different predictor,
training objective, or downstream solver.

% ============================================================
% SHORTEST PATH
% ============================================================

\subsubsection{Controlled contextual shortest path}
\label{sec:controlled_shortest_path}

\paragraph{Graph and primitive context.}
The controlled shortest-path benchmark uses a fixed directed $6\times6$ grid.
Edges connect each node to its right and lower neighbors, producing an acyclic
directed graph
\[
    G=(V,E).
\]

Each edge $e\in E$ has $p=8$ primitive contextual coordinates,
\[
    c_e\in\mathbb{R}^{8},
\]
and the full primitive context is
\begin{equation}
    c
    =
    [c_1^\top,\ldots,c_{|E|}^\top]^\top
    \in
    \mathbb{R}^{8|E|}.
    \label{eq:sp_context}
\end{equation}

For the $6\times6$ right/down grid, $|E|=60$, so the ambient context dimension
is $480$.

\paragraph{Context-to-cost map.}
The true cost of edge $e$ is generated by
\begin{equation}
    a_e(c)
    =
    \operatorname{softplus}
    \left(
        \theta_e^\top c_e
        +
        0.30\sin(c_{e,1})
        +
        0.15c_{e,2}c_{e,3}
    \right)
    +0.05,
    \label{eq:sp_cost}
\end{equation}
where $\theta_e$ is sparse. Hence only a subset of the primitive coordinates
can affect the true edge-cost vector
\[
    a(c)\in\mathbb{R}^{|E|}.
\]

The known generator provides oracle support labels for evaluating interface
recovery.

\paragraph{Flow formulation of shortest path.}
Let
\[
    x_e
\]
denote the amount of source-to-destination flow assigned to edge $e$. Let
\[
    A_{\mathrm{flow}}
    \in
    \mathbb{R}^{|V|\times|E|}
\]
be the node--edge incidence matrix and let
$b_{\mathrm{flow}}\in\mathbb{R}^{|V|}$ contain $+1$ at the source, $-1$ at
the destination, and zero elsewhere. Feasible source-to-destination flows
satisfy
\begin{equation}
    A_{\mathrm{flow}}x
    =
    b_{\mathrm{flow}},
    \qquad
    0\leq x\leq1.
    \label{eq:sp_flow_constraints}
\end{equation}

Without regularization, the usual shortest-path linear program is
\begin{equation}
    \min_x
    a(c)^\top x
    \qquad
    \text{s.t.}
    \qquad
    A_{\mathrm{flow}}x=b_{\mathrm{flow}},
    \quad
    0\leq x\leq1.
    \label{eq:sp_lp}
\end{equation}

The network-flow constraint matrix is totally unimodular, so with integral
flow demand the vertices of this polytope are integral. This is the classical
connection between the combinatorial shortest-path problem and its continuous
flow formulation.

\paragraph{Differentiable local sensitivity model.}
The discrete shortest-path map is unsuitable for ordinary local
differentiation. As the edge costs vary, the optimal discrete path remains
unchanged over open regions and switches discontinuously at path boundaries.
Consequently, differentiating the discrete argmin directly does not provide a
useful local sensitivity signal.

For interface construction, we therefore use the same flow polytope with a
small strongly convex regularizer:
\begin{equation}
    \widetilde{x}^{\star}(c)
    =
    \arg\min_x
    \left[
        a(c)^\top x
        +
        \varepsilon_{\mathrm{path}}
        \|x\|_2^2
    \right]
    \label{eq:sp_regularized_qp_obj}
\end{equation}
subject to
\begin{equation}
    A_{\mathrm{flow}}x=b_{\mathrm{flow}},
    \qquad
    0\leq x\leq1,
    \label{eq:sp_regularized_qp_constraints}
\end{equation}
with
\[
    \varepsilon_{\mathrm{path}}
    =
    5\times10^{-3}
\]
in the checked-in implementation.

The quadratic term is introduced to provide a well-conditioned local
sensitivity model. We use this regularized flow problem to define the
Jacobian employed for interface scoring; we do \emph{not} rely on the
integrality of its optimizer. The integrality statement applies to the
unregularized linear program in Eq.~\eqref{eq:sp_lp}, whereas
Eq.~\eqref{eq:sp_regularized_qp_obj} is a strongly convex quadratic program.

This distinction is important: the continuous problem is an instrument for
estimating which context directions influence the downstream route geometry,
rather than a claim that the discrete shortest-path argmin itself is
differentiable.

\paragraph{KKT system.}
Let $\nu$ denote the multiplier associated with the flow-conservation
equalities and let
$\lambda^{\mathrm{lo}},\lambda^{\mathrm{up}}\geq0$ denote the lower- and
upper-bound multipliers. The stationarity equation is
\begin{equation}
    a(c)
    +
    2\varepsilon_{\mathrm{path}}x
    +
    A_{\mathrm{flow}}^\top\nu
    -
    \lambda^{\mathrm{lo}}
    +
    \lambda^{\mathrm{up}}
    =
    0.
    \label{eq:sp_stationarity}
\end{equation}

Together with
\begin{align}
    A_{\mathrm{flow}}x-b_{\mathrm{flow}} &= 0,
    \label{eq:sp_kkt_flow}
    \\
    \lambda^{\mathrm{lo}}\odot(-x) &= 0,
    \label{eq:sp_comp_lower}
    \\
    \lambda^{\mathrm{up}}\odot(x-1) &= 0,
    \label{eq:sp_comp_upper}
\end{align}
this defines the residual
\[
    F_{\mathrm{SP}}(z,c)=0.
\]

We initialize the primal and dual variables from a numerical constrained solve,
refine them against the explicit KKT residual, and then compute
\begin{equation}
    \frac{\partial z^\star}{\partial c}
    =
    -
    \left(
        \frac{\partial F_{\mathrm{SP}}}{\partial z}
        +
        \varepsilon_{\mathrm{reg}}I
    \right)^{-1}
    \frac{\partial F_{\mathrm{SP}}}{\partial c}.
    \label{eq:sp_implicit}
\end{equation}

The primal block is
\begin{equation}
    J_{\mathrm{SP}}(c)
    =
    \frac{\partial\widetilde{x}^{\star}(c)}
         {\partial c}
    \in
    \mathbb{R}^{|E|\times 8|E|}.
    \label{eq:sp_jacobian}
\end{equation}

Primitive coordinate $j$ receives score
\begin{equation}
    s_j(c)
    =
    \left\|
        J_{\mathrm{SP}}(c)_{:,j}
    \right\|_2.
    \label{eq:sp_feature_score}
\end{equation}

A high score means that a small perturbation of $c_j$ produces substantial
first-order motion in the optimal regularized flow; a zero score means that
the local KKT geometry annihilates that perturbation.

\paragraph{Why the flow relaxation is used.}
The role of the continuous optimization problem can be summarized as
\begin{equation}
    \underbrace{c}_{\text{primitive context}}
    \longrightarrow
    \underbrace{a(c)}_{\text{edge costs}}
    \longrightarrow
    \underbrace{\widetilde{x}^{\star}(c)}_{\text{regularized flow}}
    \longrightarrow
    \underbrace{
        \frac{\partial\widetilde{x}^{\star}}{\partial c}
    }_{\text{interface signal}}.
    \label{eq:sp_scoring_pipeline}
\end{equation}

We use the regularized flow QP as the downstream routing model throughout the reported shortest-path experiments. The quadratic term removes local ambiguity, yields a unique optimizer map suitable for KKT sensitivity, and is retained during downstream evaluation. We do not claim that this QP is identical to the discrete shortest-path argmin.

\paragraph{Post-selection prediction.}
After selecting an interface $\mathcal{I}$, the downstream predictor receives
only $c_{\mathcal{I}}$ and predicts the full edge-cost vector. For the linear
model,
\begin{equation}
    \widehat{a}
    =
    W_{\mathcal{I}}c_{\mathcal{I}}+b,
    \label{eq:sp_linear_predictor}
\end{equation}
with parameters fitted by least squares on the training split. The MLP
predictor receives the same selected coordinates and predicts the same target
$a$.

The predicted costs are then passed to the downstream flow optimization:
\begin{equation}
    \widehat{x}
    =
    \arg\min_x
    \left[
        \widehat{a}^{\top}x
        +
        \varepsilon_{\mathrm{path}}\|x\|_2^2
    \right]
    \qquad
    \text{s.t.}
    \qquad
    A_{\mathrm{flow}}x=b_{\mathrm{flow}},
    \quad
    0\leq x\leq1.
    \label{eq:sp_predicted_decision}
\end{equation}

The oracle reference uses the true edge costs:
\begin{equation}
    x^{\mathrm{ref}}
    =
    \arg\min_x
    \left[
        a(c)^\top x
        +
        \varepsilon_{\mathrm{path}}\|x\|_2^2
    \right]
    \qquad
    \text{s.t.}
    \qquad
    A_{\mathrm{flow}}x=b_{\mathrm{flow}},
    \quad
    0\leq x\leq1.
    \label{eq:sp_reference_decision}
\end{equation}

% ============================================================
% MATCHED EVALUATION
% ============================================================

\subsubsection{Matched post-selection evaluation}
\label{sec:controlled_eval}

The selector is evaluated only through the interface it returns. Once
$\mathcal{I}$ is fixed, DePICT and all baselines use the same:
\begin{itemize}
    \item training, validation, and test splits;
    \item interface budget;
    \item downstream predictor family;
    \item prediction target;
    \item predictor-training objective; and
    \item downstream optimization solver.
\end{itemize}

This matched protocol isolates interface quality from predictor capacity and
solver choice.

\paragraph{Prediction targets.}
For knapsack, the predictor estimates the item-value vector,
\[
    c_{\mathcal{I}}
    \longmapsto
    \widehat{v}.
\]
For shortest path, it estimates the edge-cost vector,
\[
    c_{\mathcal{I}}
    \longmapsto
    \widehat{a}.
\]

The linear predictor provides a deliberately restricted diagnostic of
interface sufficiency: because it cannot reconstruct omitted information
through nonlinear interactions, differences between interfaces are exposed
directly. The MLP experiment asks whether those differences persist when the
downstream model has additional capacity to exploit correlations among the
retained coordinates.

% ============================================================
% METRICS
% ============================================================

\subsubsection{Evaluation metrics}
\label{sec:controlled_metrics}

\paragraph{Support recovery.}
Because the controlled generators provide oracle primitive supports
$\mathcal{S}^{\star}$, we evaluate
\begin{align}
    \operatorname{Precision}
    &=
    \frac{
        |\widehat{\mathcal{I}}\cap\mathcal{S}^{\star}|
    }{
        |\widehat{\mathcal{I}}|
    },
    \\
    \operatorname{Recall}
    &=
    \frac{
        |\widehat{\mathcal{I}}\cap\mathcal{S}^{\star}|
    }{
        |\mathcal{S}^{\star}|
    },
    \\
    F_1
    &=
    \frac{
        2\,\operatorname{Precision}\operatorname{Recall}
    }{
        \operatorname{Precision}+\operatorname{Recall}
    }.
    \label{eq:support_metrics}
\end{align}

These metrics test whether the selected interface agrees with the
construction-level notion of relevant primitive coordinates.

\paragraph{Decision regret.}
For a test context $c_t$, let $\theta_t^{\mathrm{ref}}$ denote the decision
obtained from the true optimization parameters and let
$\widehat{\theta}_t$ denote the decision obtained after predicting those
parameters from the selected interface. We report
\begin{equation}
    \operatorname{Regret}
    =
    \frac{1}{N_{\mathrm{test}}}
    \sum_{t=1}^{N_{\mathrm{test}}}
    \left[
        \ell(\widehat{\theta}_t;c_t)
        -
        \ell(\theta_t^{\mathrm{ref}};c_t)
    \right],
    \label{eq:controlled_regret}
\end{equation}
with signs chosen so that lower regret is better.

For knapsack, decisions are evaluated using the true item values:
\begin{equation}
    \operatorname{Regret}_{\mathrm{KS},t}
    =
    v(c_t)^\top x_t^{\mathrm{ref}}
    -
    v(c_t)^\top\widehat{x}_t.
    \label{eq:knapsack_regret}
\end{equation}

For the regularized shortest-path experiment,
\begin{equation}
    \operatorname{Regret}_{\mathrm{SP},t}
    =
    \left[
        a(c_t)^\top\widehat{x}_t
        +
        \varepsilon_{\mathrm{path}}
        \|\widehat{x}_t\|_2^2
    \right]
    -
    \left[
        a(c_t)^\top x_t^{\mathrm{ref}}
        +
        \varepsilon_{\mathrm{path}}
        \|x_t^{\mathrm{ref}}\|_2^2
    \right].
    \label{eq:sp_regret}
\end{equation}

\paragraph{Feasibility.}
We additionally report downstream constraint violation. For knapsack,
\begin{equation}
    \operatorname{Feas}_{\mathrm{KS}}
    =
    \max
    \left\{
        0,\,
        w^\top\widehat{x}-B
    \right\}
    +
    \|[-\widehat{x}]_+\|_1
    +
    \|[\widehat{x}-1]_+\|_1.
    \label{eq:knapsack_feas}
\end{equation}

For shortest path,
\begin{equation}
    \operatorname{Feas}_{\mathrm{SP}}
    =
    \left\|
        A_{\mathrm{flow}}\widehat{x}
        -
        b_{\mathrm{flow}}
    \right\|_1
    +
    \|[-\widehat{x}]_+\|_1
    +
    \|[\widehat{x}-1]_+\|_1.
    \label{eq:sp_feas}
\end{equation}

\paragraph{Flow agreement.}
The checked-in implementation measures shortest-path agreement as the
fraction of edge coordinates whose optimized flows agree within numerical
tolerance:
\begin{equation}
    \operatorname{FlowMatch}
    =
    \frac{1}{|E|}
    \sum_{e\in E}
    \mathbf{1}
    \left\{
        \left|
            \widehat{x}_e-x_e^{\mathrm{ref}}
        \right|
        \leq\tau_{\mathrm{match}}
    \right\}.
    \label{eq:flow_match}
\end{equation}

We call this quantity \emph{flow agreement} rather than exact path match to
avoid conflating coordinatewise agreement of continuous flow solutions with
equality of two discrete paths.

% ============================================================
% WHAT EACH CONTROLLED EXPERIMENT ESTABLISHES
% ============================================================

\subsubsection{What the controlled experiments establish}
\label{sec:controlled_interpretation}

The controlled experiments separate three levels of evidence.

First, \emph{support recovery} tests whether optimizer sensitivity identifies
the primitive coordinates known to parameterize the downstream task.

Second, \emph{post-selection regret} tests whether preserving those
coordinates is sufficient to retain downstream decision quality after the
interface is passed to a separately trained predictor.

Third, the shortest-path benchmark tests whether the same interface principle
can be used when the underlying task has a combinatorial interpretation. In
this case, DePICT does not differentiate the discrete argmin. It differentiates
a regularized continuous flow problem during interface construction and uses
the resulting sensitivity only to rank primitive context coordinates.

This separation is essential to the interpretation of DePICT: the Jacobian is
used to construct the interface, not to train the downstream predictor or to
replace the downstream optimizer.

% ============================================================
% Real-Data and Edge-Regime Stress Tests
% ============================================================

\subsubsection{Real-Data and Edge-Regime Stress Tests}
\label{sec:real_edge}

The controlled experiments test DePICT where the context-to-optimizer map is
known and the relevant primitive coordinates are available by construction.
We next ask a different question: whether the same interface principle remains
useful when these favorable conditions are removed.

We consider two real-data benchmarks: Warcraft terrain routing and FF25
portfolio allocation. Neither benchmark provides an oracle
decision-preserving support, and neither is used as a direct validation of the
support-recovery theory. Instead, they test whether local optimizer
sensitivity can identify compact interfaces that preserve downstream decision
quality in solver-mediated pipelines with learned context-to-task mappings,
nonsmooth behavior, changing active constraints, or real observational data.

For both benchmarks, we first construct a differentiable map from primitive
context to the parameters of the downstream optimization problem. DePICT then
differentiates the resulting optimizer map,
\begin{equation}
    J(c)
    =
    \frac{\partial \theta^\star(c)}{\partial c},
\end{equation}
aggregates its column norms across representative contexts, and fixes an
interface $\mathcal I$. A new downstream predictor is subsequently fitted
using only $c_{\mathcal I}$. Thus the stress tests preserve the central
experimental separation:
\begin{equation}
    \text{task model}
    \longrightarrow
    \text{optimizer sensitivity}
    \longrightarrow
    \mathcal I
    \longrightarrow
    \text{reduced-interface predictor}
    \longrightarrow
    \text{downstream decision}.
    \label{eq:real_general_pipeline}
\end{equation}

% ============================================================
% WARCRAFT
% ============================================================

\subsubsection{Warcraft terrain routing}
\label{sec:warcraft}

\paragraph{Primitive image context.}
Warcraft provides RGB terrain maps together with terrain costs for a
shortest-path task. We use the $12\times12$ grid version. The original
high-resolution RGB image is block-averaged so that each logical grid cell is
represented by one RGB vector. The resulting primitive context is
\begin{equation}
    c
    \in
    \mathbb{R}^{12\times12\times3}
    \equiv
    \mathbb{R}^{432}.
    \label{eq:warcraft_context}
\end{equation}

Each primitive coordinate therefore corresponds to one color channel at one
grid-aligned terrain cell. We deliberately retain this low-level
representation rather than introducing semantic image features: the purpose
of the benchmark is to test whether DePICT can identify decision-relevant
directions in a structured, high-dimensional input.

\paragraph{Reference edge costs.}
The benchmark supplies a cost for each grid vertex. For a directed edge
$e=(u,v)$, we define its reference traversal cost as
\begin{equation}
    a_e
    =
    \frac{1}{2}
    \left(
        q_u + q_v
    \right),
    \label{eq:warcraft_edge_cost}
\end{equation}
where $q_u$ and $q_v$ are the benchmark-provided costs of the incident cells.
This produces a reference edge-cost vector
\[
    a \in \mathbb{R}^{|E|}.
\]

\paragraph{Differentiable context-to-cost model.}
The benchmark does not provide an analytical map from RGB values to edge
costs. We therefore fit this map using the training split.

For each directed edge $e=(u,v)$, the model receives the RGB vectors of its
two incident cells,
\begin{equation}
    \phi_e(c)
    =
    \begin{bmatrix}
        \operatorname{RGB}_u \\
        \operatorname{RGB}_v
    \end{bmatrix}
    \in\mathbb{R}^{6},
\end{equation}
and predicts
\begin{equation}
    \widetilde a_e(c)
    =
    \operatorname{softplus}
    \left(
        \beta_e^\top\phi_e(c)+b_e
    \right).
    \label{eq:warcraft_cost_model}
\end{equation}

The edge-specific coefficients are fitted from training maps and their
reference edge costs. This model is used only to provide a differentiable
context-to-task map for sensitivity computation; it does not define the
interface itself.

The resulting dependency is
\begin{equation}
    c
    \longrightarrow
    \widetilde a(c)
    \longrightarrow
    x^\star(c).
    \label{eq:warcraft_context_cost_path}
\end{equation}

\paragraph{Regularized flow optimization.}
Let $x_e$ denote flow on edge $e$ and let
$A_{\mathrm{flow}}x=b_{\mathrm{flow}}$ impose unit flow from the source to
the destination. For sensitivity computation, we solve
\begin{equation}
    \widetilde{x}^{\star}(c)
    =
    \arg\min_x
    \left[
        \widetilde a(c)^\top x
        +
        \varepsilon_{\mathrm{path}}\|x\|_2^2
    \right]
    \label{eq:warcraft_flow_objective}
\end{equation}
subject to
\begin{equation}
    A_{\mathrm{flow}}x=b_{\mathrm{flow}},
    \qquad
    0\leq x\leq1.
    \label{eq:warcraft_flow_constraints}
\end{equation}

We use
\[
    \varepsilon_{\mathrm{path}}=5\times10^{-3}.
\]

The quadratic term removes local ambiguity and yields a stable optimizer map
for implicit differentiation. Accordingly, we treat
Eqs.~\eqref{eq:warcraft_flow_objective}--\eqref{eq:warcraft_flow_constraints}
as a continuous sensitivity model of the routing problem; we do not claim
that the discrete shortest-path map itself is differentiable.

\paragraph{DePICT sensitivity.}
For each representative map, we solve the primal--dual KKT system associated
with the regularized flow problem and compute
\begin{equation}
    J_{\mathrm{WC}}(c)
    =
    \frac{
        \partial\widetilde{x}^{\star}(c)
    }{
        \partial c
    }
    \in
    \mathbb{R}^{|E|\times432}.
    \label{eq:warcraft_jacobian}
\end{equation}

Equivalently, the sensitivity propagates through both components of the
pipeline,
\begin{equation}
    \frac{\partial\widetilde{x}^{\star}}{\partial c}
    =
    \frac{\partial\widetilde{x}^{\star}}
         {\partial\widetilde a}
    \frac{\partial\widetilde a}{\partial c}.
    \label{eq:warcraft_chain}
\end{equation}

Primitive coordinate $j$ receives score
\begin{equation}
    s_j(c)
    =
    \left\|
        J_{\mathrm{WC}}(c)_{:,j}
    \right\|_2.
    \label{eq:warcraft_score}
\end{equation}

We aggregate these scores across the designated interface-selection maps and
select the highest-scoring coordinates at each target budget.

The interpretation is direct: a coordinate receives a high score when a
small perturbation of that cell-channel value induces substantial local motion
of the optimized flow.

\paragraph{Reduced-interface evaluation.}
After fixing $\mathcal I$, we train a new edge-cost predictor using only
$c_{\mathcal I}$:
\begin{equation}
    \widehat a
    =
    W_{\mathcal I}c_{\mathcal I}+b.
    \label{eq:warcraft_reduced_predictor}
\end{equation}

The predicted edge costs are passed to the same regularized flow optimization
problem to obtain $\widehat x$. The reference decision is obtained from the
benchmark-derived edge costs $a$.

We evaluate the selected interface using downstream regret, objective value,
feasibility, and agreement between the predicted-cost and reference-cost flow
solutions. Because Warcraft provides no oracle primitive support, we do not
report support precision, recall, or F1.

Thus Warcraft evaluates
\begin{equation}
\boxed{
    \text{terrain image}
    \rightarrow
    \text{optimizer sensitivity}
    \rightarrow
    \text{reduced RGB interface}
    \rightarrow
    \text{edge-cost prediction}
    \rightarrow
    \text{routing decision}.
}
\end{equation}

% ============================================================
% FF25
% ============================================================

\subsubsection{FF25 constrained portfolio allocation}
\label{sec:ff25}

\paragraph{Data and primitive context.}
The FF25 benchmark uses monthly returns from the Kenneth R.\ French
$5\times5$ size--book-to-market portfolio panel. At each decision month $t$,
we construct the primitive context exclusively from information available
before the realized return at month $t$.

For each of the $25$ portfolios, the context contains trailing return means
over $1$, $3$, $6$, and $12$ months; trailing volatilities over $3$, $6$, and
$12$ months; and return momentum over $1$, $6$, and $12$ months. We
add trailing means of the Fama--French factor series over $3$, $6$, and
$12$ months. The resulting context has dimension
\begin{equation}
    c_t\in\mathbb{R}^{262}.
    \label{eq:ff25_context}
\end{equation}

All features are backward-looking, preventing future-return information from
entering the decision at month $t$. We split the observations
chronologically into training, validation, and test periods.

\paragraph{Context-to-return model.}
The downstream optimizer requires an estimate of expected returns. We fit a
multivariate linear model on the training period,
\begin{equation}
    \mu(c)
    =
    W_\mu^\top c+b_\mu,
    \qquad
    \mu(c)\in\mathbb{R}^{25},
    \label{eq:ff25_mu}
\end{equation}
where the target is the vector of realized portfolio returns at the
corresponding month.

This model supplies the differentiable context dependence required by the
optimizer:
\begin{equation}
    c_t
    \longrightarrow
    \mu(c_t)
    \longrightarrow
    w^\star(c_t).
    \label{eq:ff25_context_optimizer}
\end{equation}

The rolling covariance matrix
$\Sigma_t$ is estimated from the preceding $60$ months and is treated as
fixed with respect to the local context perturbation used for DePICT scoring.

\paragraph{Constrained portfolio problem.}
At each decision month, we solve the convex program
\begin{equation}
    w_t^\star
    =
    \arg\min_w
    \left[
        -\mu(c_t)^\top w
        +
        \gamma w^\top\Sigma_t w
        +
        \eta
        \|w-w_{t-1}\|_2^2
    \right],
    \label{eq:ff25_objective}
\end{equation}
subject to
\begin{align}
    \mathbf 1^\top w &= 1,
    \label{eq:ff25_budget}
    \\
    0\leq w_i &\leq u,
    \qquad i=1,\ldots,25,
    \label{eq:ff25_box}
    \\
    \ell
    \leq
    F_t^\top w
    &\leq
    u_F.
    \label{eq:ff25_exposure}
\end{align}

The implementation uses
\[
    \gamma=5,
    \qquad
    \eta=1,
    \qquad
    u=0.10.
\]

The matrix $F_t$ contains rolling factor exposures to the market, SMB, and
HML factors. Exposure bounds are determined from the training period and
remain fixed during test evaluation. The turnover term couples consecutive
portfolio decisions while remaining differentiable with respect to the
current portfolio weights.

\paragraph{DePICT sensitivity.}
Let
\[
    z_t^\star
    =
    \left(
        w_t^\star,
        \lambda_t^\star,
        \nu_t^\star
    \right)
\]
denote the local primal--dual solution. From the portfolio KKT residual
\[
    F_{\mathrm{PF}}(z_t,c_t)=0,
\]
implicit differentiation gives
\begin{equation}
    \frac{\partial z_t^\star}{\partial c_t}
    =
    -
    \left(
        \frac{\partial F_{\mathrm{PF}}}{\partial z_t}
        +
        \varepsilon_{\mathrm{reg}}I
    \right)^{-1}
    \frac{\partial F_{\mathrm{PF}}}{\partial c_t}.
    \label{eq:ff25_implicit}
\end{equation}

The primal block yields
\begin{equation}
    J_{\mathrm{FF25}}(c_t)
    =
    \frac{\partial w_t^\star}{\partial c_t}
    \in
    \mathbb{R}^{25\times262}.
    \label{eq:ff25_jacobian}
\end{equation}

The pointwise importance of primitive coordinate $j$ is
\begin{equation}
    s_{t,j}
    =
    \left\|
        J_{\mathrm{FF25}}(c_t)_{:,j}
    \right\|_2.
    \label{eq:ff25_score}
\end{equation}

For the reported implementation, Jacobian evidence used for feature selection
is restricted to training-period dates, with at most $120$ representative
dates entering the aggregation. Test observations never contribute to
interface construction.

After aggregating the scores,
\begin{equation}
    \bar{s}_j
    =
    \operatorname{Agg}_t
    \left(
        s_{t,j}
    \right),
\end{equation}
we construct fixed-budget interfaces by retaining the coordinates with the
largest aggregated sensitivities.

\paragraph{Reduced-interface evaluation.}
For each selected interface $\mathcal I$, we refit the expected-return model
using only the retained coordinates:
\begin{equation}
    \widehat{\mu}_{\mathcal I}(c_t)
    =
    W_{\mathcal I}^\top
    c_{t,\mathcal I}
    +
    b_{\mathcal I}.
    \label{eq:ff25_reduced_mu}
\end{equation}

The reduced-interface and full-context models are then passed independently
through the same constrained portfolio optimizer over the chronological test
period. Each path carries forward its own previous portfolio weights, so the
turnover penalty is evaluated consistently with the decisions generated by
that path.

We report realized out-of-sample Sharpe ratio, mean return, volatility,
turnover, feasibility, and the mean optimization-objective gap between the
reduced- and full-context pipelines. Because the data provide no oracle
decision-relevant feature set, these metrics evaluate preservation of
downstream behavior rather than support recovery.

% ============================================================
% INTERPRETATION
% ============================================================

\subsubsection{Interpretation of the real-data experiments}
\label{sec:real_interpretation}

These experiments test transfer of the DePICT principle, not the exactness of
the local theory. The controlled benchmarks establish whether the
solution-sensitivity criterion recovers known decision-relevant directions.
Warcraft and FF25 instead ask whether the same criterion remains useful when
the context-to-task relationship must be estimated from data and the
deployment pipeline departs from the idealized fixed-regime setting.

Accordingly, we make no claim that the interfaces selected on these datasets
are globally minimal or equal to an unknown oracle support. The relevant
empirical question is narrower: at a matched interface budget and under the
same downstream prediction and optimization pipeline, does selecting context
according to optimizer sensitivity preserve decision quality better than
alternative interfaces?

The two benchmarks exercise complementary departures from the controlled
setting. Warcraft introduces structured image inputs and a learned
context-to-cost map before a constrained routing problem. FF25 retains a
smooth convex optimizer but introduces real temporal data, estimated expected
returns and covariances, changing portfolio states, and time-varying
constraint geometry. Positive results on these tasks therefore provide
evidence that decision-preserving interface selection can transfer beyond the
settings in which exact support recovery is available, without extending the
formal guarantees beyond their stated assumptions.

\subsubsection{Additional Results and Ablations}
\label{app:ablations}

This section reports extended results underlying the main-paper tables and figures: full budget sweeps for the controlled benchmarks, selector-isolated results under linear and nonlinear predictors, portfolio-allocation tables, and ablations over Jacobian aggregation and instance filtering. These results test whether the main conclusions are robust to budget choice, aggregation rule, representative-context filtering, and predictor family.
\subsubsection{Seed variability across benchmarks}
\label{app:seed_variability}

To quantify run-to-run variability, we repeat the representative linear-predictor regret comparisons over 10 independent seeds. Seeds vary
synthetic data generation, predictor initialization, and randomized selectors where applicable; data splits and evaluation protocols are matched across methods within each seed. Figure~\ref{fig:seed_variance_app} reports mean performance with \(\pm 1\) standard deviation.

Across controlled knapsack, controlled shortest path, Warcraft, and FF25, DePICT attains the lowest mean regret among the compared reduced-interface methods. The controlled KKT-regime benchmarks show the clearest separation,
consistent with the theory, while FF25 exhibits larger overlap among error bars, consistent with its role as a real-data tradeoff benchmark rather
than a direct validation of the smooth KKT theorem.

\begin{figure*}[t]
\centering
\includegraphics[width=\textwidth]{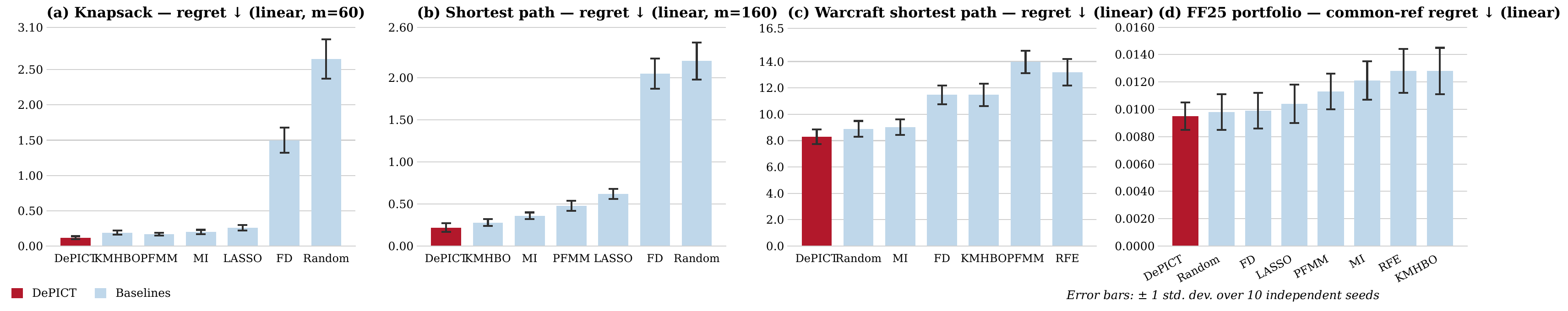}
\vspace{-2mm}
\caption{
\textbf{Seed variability across representative regret comparisons.}
Bars report mean regret and error bars report \(\pm 1\) standard deviation
over 10 independent seeds. Panels cover controlled knapsack, controlled
shortest path, Warcraft shortest path, and FF25 portfolio allocation under
representative linear-predictor settings. Lower is better.
}
\label{fig:seed_variance_app}
\vspace{-3mm}
\end{figure*}

\paragraph{Edge Regime.}
Table~\ref{tab:warcraft_ff25_main} reports solver-mediated performance on Warcraft uses the regularized flow QP derived from the shortest-path problem. Its challenge relative to the controlled setting comes from the learned image-to-cost map and changing instance geometry rather than from a discrete downstream argmin.

On Warcraft, DePICT achieves the lowest regret under both predictor families. Under the linear predictor, where nonlinear compensation is unavailable, DePICT obtains regret \(8.2974\), improving over the nearest
baseline at \(8.8871\), and also gives the highest path match (\(0.8286\)). Under the nonlinear predictor (MLP), selector gaps are much smaller, but DePICT remains best in regret (\(12.0559\)) and path match
(\(0.8368\)). This pattern is consistent with the edge-regime interpretation: the smooth sensitivity signal is not guaranteed for the discrete decision map, but it remains useful, especially when predictor capacity cannot mask missing interface directions.

On FF25, the comparison is decision-focused because no oracle support is available. At matched budget \(m=40\), DePICT achieves the lowest common-reference regret in the linear setting (\(0.0095\)) and the lowest turnover (\(0.0619\)), while maintaining competitive Sharpe. This indicates that the selected interface preserves portfolio quality while
reducing trading variation. Under the MLP predictor, DePICT attains the best reported common-reference objective and the lowest or tied-lowest regret among reduced-interface methods, although RFE and PFMM obtain the best Sharpe and turnover respectively. Overall, the FF25 results support a tradeoff view: DePICT is not uniformly best on every financial metric,
but it lies on a favorable regret-Sharpe-turnover frontier.

Taken together, Warcraft and FF25 show that the interface selected from optimizer-side sensitivity can transfer beyond the assumptions of the main theorem. The evidence is strongest in the linear-predictor setting, where differences between interfaces are least likely to be hidden by downstream model capacity.

\begin{table*}[t]
\centering
\caption{
Solver-mediated decision-task evaluation on Warcraft shortest path and FF25 portfolio optimization. 
Warcraft reports shortest-path regret, objective value, feasibility violation, and path-match rate. 
FF25 reports common-reference portfolio regret, common-reference objective, realized Sharpe, and turnover at matched budget $m=40$ from the provided runs. 
Lower regret, feasibility violation, and turnover are better; higher objective, path match, and Sharpe are better.
}
\label{tab:warcraft_ff25_main}
\setlength{\tabcolsep}{3.2pt}
\renewcommand{\arraystretch}{1.05}
\scriptsize
\resizebox{\textwidth}{!}{
\begin{tabular}{lcccc ccccc}
\toprule
& \multicolumn{4}{c}{\textbf{Warcraft Shortest Path}}
& \multicolumn{5}{c}{\textbf{FF25 Portfolio}} \\
\cmidrule(lr){2-5} \cmidrule(lr){6-10}
\textbf{Method}
& Regret $\downarrow$
& Obj. $\uparrow$
& Feas. $\downarrow$
& Path $\uparrow$
& $m$
& Regret $\downarrow$
& Obj. $\uparrow$
& Sharpe $\uparrow$
& Turnover $\downarrow$ \\
\midrule
\multicolumn{10}{l}{\textbf{Nonlinear predictor (MLP)}} \\
\midrule
Random
& 12.7335 & -50.6898 & 3.15e-05 & 0.8278
& 40 & 0.0040 & -0.0040 & 0.4539 & 0.1042 \\

Mutual Information
& 12.0701 & -49.0297 & \textbf{2.44e-05} & 0.8348
& 40 & 0.0041 & -0.0029 & 0.4743 & 0.1030  \\

LASSO
& 12.7771 & -50.7333 & 3.41e-05 & 0.8272
& 40 & 0.0033 & -0.0018 & 0.4537 & 0.1001 \\

FD Ranking
& 12.3340 & -50.1533 & 3.09e-05 & 0.8298
& 40 & 0.0033 & -0.0037 & 0.4515 & 0.0899 \\

RFE
& 13.3152 & -51.2714 & 2.78e-05 & 0.8283
& 40 & 0.0034 & -0.0036 & \textbf{0.4925} & 0.0909 \\

PFMM
& 13.4992 & -50.8269 & 2.97e-05 & 0.8365
& 40 & 0.0037 & -0.0033 & 0.4663 & \textbf{0.0868} \\

KMHBO
& 14.9038 & -52.2316 & 3.06e-05 & 0.8358
& 40 & 0.0035 & -0.0035 & 0.4784 & 0.1052 \\

\textbf{DePICT}
& \textbf{12.0559} & \textbf{-49.0029} & 3.14e-05 & \textbf{0.8368}
& 40 & \textbf{0.0033} & \textbf{0.0047} & 0.4776 & 0.0970 \\

\midrule
\multicolumn{10}{l}{\textbf{Linear}} \\
\midrule
Random
& 8.8871 & -45.9737 & 3.54e-05 & 0.8155
& 40 & 0.0098 & -0.0160 & 0.4687 & 0.0700 \\

Mutual Information
& 9.0221 & -46.1086 & 3.54e-05 & 0.8160
& 40 & 0.0121 & -0.0132 & 0.4901 & 0.0747\\

LASSO
& 27.9684 & -65.0550 & 3.46e-05 & 0.8090
& 40 & 0.0104 & -0.0149 & 0.4673 & 0.0725 \\

FD Ranking
& 11.4664 & -48.5529 & 3.73e-05 & 0.8167
& 40 & 0.0099 & -0.0154 & 0.4771 & 0.0654 \\

RFE
& 13.1856 & -51.7321 & 3.41e-05 & 0.8159
& 40 & 0.0128 & -0.0126 & 0.5059 & 0.0632 \\

PFMM
& 13.9751 & -50.6988 & \textbf{3.58e-06} & 0.8197
& 40 & 0.0113 & -0.0141 & 0.4636 & 0.0657 \\

KMHBO
& 11.4721 & -49.2138 & 3.56e-06 & 0.8029
& 40 & 0.0128 & \textbf{-0.0125} & \textbf{0.5138} & 0.0706 \\

\textbf{DePICT}
& \textbf{8.2974} & \textbf{-45.6251} & 3.91e-05 & \textbf{0.8286}
& 40 & \textbf{0.0095} & -0.0158 & 0.4892 & \textbf{0.0619} \\
\bottomrule
\end{tabular}
}
\end{table*}

\paragraph{Supplementary energy-scheduling results.}
Table~\ref{tab:energy_scheduling_main} reports an additional solver-mediated scheduling benchmark. This task is discrete and outside the fixed-active-set KKT theory, so it is used only as supplementary evidence for interface transfer beyond the smooth regime.
\begin{table*}[t]
\centering
\caption{
Energy-cost-aware scheduling edge-regime evaluation. We report all comparison
methods at budget $m=5$ across task sizes and predictor families. Lower cost
and regret are better. All methods have zero feasibility violation and zero
solver failure rate. The benchmark is solver-mediated and discrete, so it is
used as an edge-regime stress test rather than a direct validation of the
fixed-active-set KKT theory.
}
\label{tab:energy_scheduling_main}
\setlength{\tabcolsep}{3.5pt}
\renewcommand{\arraystretch}{1.05}
\scriptsize
\resizebox{\textwidth}{!}{
\begin{tabular}{lcc cc cc}
\toprule
& \multicolumn{2}{c}{\textbf{10 Tasks}}
& \multicolumn{2}{c}{\textbf{15 Tasks}}
& \multicolumn{2}{c}{\textbf{20 Tasks}} \\
\cmidrule(lr){2-3} \cmidrule(lr){4-5} \cmidrule(lr){6-7}
\textbf{Method}
& Cost $\downarrow$ & Regret $\downarrow$
& Cost $\downarrow$ & Regret $\downarrow$
& Cost $\downarrow$ & Regret $\downarrow$ \\
\midrule
\multicolumn{7}{l}{\textbf{Linear}} \\
\midrule
Correlation Screening
& 187.794 & 34.294
& 324.866 & 42.349
& 395.948 & 46.654 \\

Elastic Net
& 188.453 & 34.954
& 324.417 & 41.900
& 407.354 & 58.060 \\

FD Ranking
& 188.453 & 34.954
& 324.417 & 41.900
& 407.354 & 58.060 \\

Full
& 203.783 & 50.284
& 359.248 & 76.731
& 452.739 & 103.445 \\

KMHBO 
& 177.597 & 24.098
& 311.039 & 28.522
& 397.140 & 47.846 \\

LASSO
& 176.193 & 22.693
& 311.848 & 29.332
& 387.582 & 38.288 \\

Random
& 181.291  & 27.791 
& 317.592  & 35.075 
& 394.576  & 45.282  \\

RFE
& 177.574 & 24.075
& 310.625 & 28.108
& 383.914 & 34.620 \\

PFMM
& 179.727 & 26.228
& 312.993 & 30.476
& 387.216 & 37.922 \\

DePICT
& 180.114 & 26.615
& \textbf{310.210} & \textbf{27.693}
& \textbf{365.204} & \textbf{15.910} \\

\midrule
\multicolumn{7}{l}{\textbf{MLP}} \\
\midrule
Correlation Screening
& 175.802  & 22.303 
& 309.802  & 27.286 
& 380.270  & 30.976 \\

Elastic Net
& 175.161  & 21.661 
& 308.597  & 26.080 
& 378.403  & 29.109  \\

FD Ranking
& 174.128  & 20.629 
& 307.505  & 24.988 
& 377.615  & 28.320  \\

Full
& 182.703  & 29.204 
& 321.755  & 39.239 
& 396.682  & 47.388  \\

Random
& 175.831  & 22.332 
& 309.725  & 27.209 
& 379.776  & 30.482  \\

KMHBO 
& 175.484  & 21.985 
& 309.480  & 26.963 
& 380.691  & 31.396  \\

LASSO
& 171.810  & 18.311 
& 305.230  & 22.713 
& 376.220  & 26.925  \\

RFE
& 175.394  & 21.894 
& 308.820  & 26.303 
& 379.258  & 29.964  \\

PFMM
& 176.293  & 22.793 
& 310.350  & 27.833 
& 381.105  & 31.811  \\

DePICT
& 174.418  & \textbf{15.918} 
& \textbf{303.439}  & \textbf{20.922} 
& 378.030  & 28.236  \\
\bottomrule
\end{tabular}
}
\end{table*}
\subsubsection{Implementation Details}
\label{app:implementation}

All experiments use deterministic seeds. Within each seed, data splits, synthetic data generation, predictor initialization, and randomized selectors are matched across methods whenever applicable. Interface selection is performed using training data or representative operating-regime contexts only. Hyperparameters, feature budgets, and model-selection choices are selected using validation data. Test metrics are computed once after final selection. For methods with tunable subset size, we evaluate shared target budgets; for methods with native realized subset sizes, we report the realized interface dimension.

\subsubsection{Computational cost and scalability}
\label{app:scalability}

DePICT requires second-order sensitivity information-Lagrangian Hessians, active constraint Jacobians, and a linearized KKT solve computed entirely from the task objective and constraints. This is a one-time cost incurred at interface-design time, before data collection or predictor training.

\paragraph{Comparability with baselines.}
A natural concern is whether this cost is prohibitive relative to baseline methods. The methods operate at different stages of the pipeline. Statistical selectors such as mutual information and LASSO construct interfaces from labeled training data, while wrapper methods such as RFE and KMHBO additionally depend on trained predictive models or repeated evaluations. DePICT instead constructs the interface from optimizer-side sensitivity before the downstream predictor is trained and requires no labeled context--decision pairs. Consequently, wall-clock interface-selection costs should be interpreted together with the different information and training requirements of each method.

\paragraph{Exact implementation cost.}
DePICT computes solution sensitivities by solving the linearized active-set KKT system at representative contexts. Let
\[
d_K = n_\theta + |\mathcal A|.
\]
A dense implementation requires factorizing a $d_K \times d_K$ saddle matrix per context and then solving multiple right-hand sides for the desired context directions. This is appropriate for the controlled and medium-scale benchmarks considered here, but it is not the only implementation route.

\paragraph{Scalable variants.}
Several scalable alternatives follow directly from the saddle-system formulation, preserving the same relevance criterion while trading exactness for scale.

\begin{itemize}
    \item \textbf{Matrix-free solves.}
    One can avoid explicit inverses and solve KKT systems with iterative methods using only Hessian-vector and Jacobian-vector products, avoiding explicit matrix formation.

    \item \textbf{Randomized adjoint probing.}
    If only rankings are needed, one need not form the full Jacobian $J(c)$. Randomized adjoint probing can estimate column-norm scores from products involving $J(c)^\top$, requiring far fewer solves than the full interface dimension.

    \item \textbf{Symbolic pre-screening.}
    Variables that do not enter $C_{\mathcal A}$ or $G_{\mathcal A}$ have zero first-order effect by as shown in Section \ref{sec:separation}. They can therefore be screened out before numerical computation, restricting sensitivity estimation to the active KKT block.

    \item \textbf{Factorization reuse.}
    The saddle-system factorization or preconditioner can be shared across multiple right-hand sides and across nearby representative contexts within the same operating regime, amortizing the dominant factorization cost.

    \item \textbf{Surrogate amortization.}
    Sensitivities can be amortized further by computing exact KKT sensitivities on a small set of representative contexts and training a lightweight surrogate score model, enabling fast interface scoring at new contexts without repeated KKT solves.
\end{itemize}

Together, these variants make DePICT's sensitivity computation tractable beyond the benchmark scale considered here. We leave their systematic empirical evaluation to future work.

\subsubsection{Layer-2 annihilation diagnostic}
\label{app:layer2_diagnostic}

The Layer-2 diagnostic is designed to isolate the distinction between symbolic support and optimizer relevance. The benchmark partitions primitive context coordinates into three groups: irreducible variables, Layer-1 nuisance variables, and Layer-2 nuisance variables. Irreducible variables move the oracle optimizer to first order and therefore belong to the decision-preserving interface. Layer-1 variables may appear in the full problem description but are absent from the active KKT sensitivity blocks on the operating regime. Layer-2 variables are more subtle: they enter the active optimization structure but are annihilated by the local KKT geometry.

\paragraph{Controlled-knapsack construction.}
For controlled knapsack, we construct Layer-2 variables using the
constraint-scaling structure of Proposition~\ref{prop:structural-main}\textup{(I)}.
Each Layer-2 coordinate enters an active constraint through a positive scaling
factor, making it symbolically present in the active constrained problem and
therefore visible to methods based on statistical association, predictive utility,
perturbation scores, or symbolic occurrence. Despite this participation, its
stationarity perturbation lies entirely in the active constraint-normal space,
\(\range(A_{\mathcal A}^\top)\), and can therefore be absorbed by a corresponding
change in the Lagrange multipliers without inducing first-order primal motion:
\[
C_{\mathcal A,\cdot j}=0,
\qquad
G_{\mathcal A,\cdot j}\in\range(A_{\mathcal A}^\top),
\qquad
J(c)e_j=0 .
\]
Thus, the planted Layer-2 coordinates are symbolically active but not
decision-preserving.

\paragraph{FF25 construction.}
For FF25, we start from the 262-dimensional backward-looking portfolio
context and append 20 planted Layer-2 and 20 planted Layer-1 nuisance
coordinates, yielding an augmented context of dimension $d=302$.
The Layer-2 coordinates are constructed through the portfolio objective so
that their induced perturbation is absorbed by the budget-constraint
geometry and does not move the optimal portfolio locally, while the
Layer-1 coordinates do not enter the active optimizer sensitivity.
We evaluate the diagnostic over 40 portfolio dates and report results over
10 independent seeds.

\begin{table*}[t]
\centering
\caption{
\textbf{Full Layer-2 diagnostic on controlled knapsack and FF25 under the linear predictor.}
For knapsack, higher irreducible recall and support F1 are better, while lower
Layer-2/Layer-1 false-positive rate and regret are better.
For FF25, lower Layer-2/Layer-1 false-positive rate, regret, and turnover are
better, while higher Sharpe is better.
}
\label{tab:knapsack_layer2}
\setlength{\tabcolsep}{5.0pt}
\renewcommand{\arraystretch}{1.12}
\scriptsize

% ============================================================
% PANEL A — KNAPSACK
% ============================================================

\textbf{(a) Controlled knapsack}

\vspace{2pt}

\resizebox{\textwidth}{!}{
\begin{tabular}{lcccccc}
\toprule
Method
& $m$
& Irr.\ Recall $\uparrow$
& L2 FPR $\downarrow$
& L1 FPR $\downarrow$
& F1 $\uparrow$
& Regret $\downarrow$ \\
\midrule

Full
& 200 & 1.000 & 1.000 & 1.000 & 0.182 & 1.788 \\

Random
& 20 & 0.100 & 0.100 & 0.100 & 0.100 & 1.954 \\

Mutual Information \cite{peng2005mrmr}
& 20 & 0.300 & 0.700 & 0.000 & 0.300 & 0.478 \\

LASSO \cite{tibshirani1996lasso}
& 20 & 0.100 & 0.000 & 0.113 & 0.100 & 2.667 \\

FD Ranking \cite{lin2023removal}
& 20 & 0.400 & 0.500 & 0.013 & 0.400 & 0.565 \\

PFMM \cite{sharma2023pfmm}
& 20 & 0.400 & 0.500 & 0.013 & 0.400 & 0.475 \\

KMHBO \cite{ye2026kmhbo}
& 20 & 0.350 & 0.650 & 0.000 & 0.350 & 0.535 \\

RFE \cite{kohavi1997rfe}
& 20 & 0.000 & 0.000 & 0.125 & 0.000 & 2.790 \\

\textbf{DePICT}
& 20
& \textbf{1.000}
& \textbf{0.000}
& \textbf{0.000}
& \textbf{1.000}
& \textbf{0.009} \\

\bottomrule
\end{tabular}
}

\vspace{8pt}

% ============================================================
% PANEL B — FF25
% ============================================================

\textbf{(b) FF25 portfolio allocation}

\vspace{2pt}

\resizebox{\textwidth}{!}{
\begin{tabular}{lcccccc}
\toprule
Method
& $m$
& L2 FPR $\downarrow$
& L1 FPR $\downarrow$
& Regret $\downarrow$
& Sharpe $\uparrow$
& Turnover $\downarrow$ \\
\midrule

Full
& 30 & 1.000 & 1.000 & 0.0127 & 0.609 & 0.1378 \\

Random
& 20 & 0.047 & 0.047 & 0.0213 & 0.620 & 0.0737 \\

Mutual Information \cite{peng2005mrmr}
& 20 & 1.000 & 0.000 & 0.0235 & \textbf{0.648} & 0.0778 \\

LASSO \cite{tibshirani1996lasso}
& 20 & 0.967 & 0.000 & 0.0235 & 0.645 & 0.0779 \\

PFMM$^\ast$ \cite{sharma2023pfmm}
& 20 & 1.000 & 0.000 & 0.0235 & \textbf{0.648} & 0.0778 \\

KMHBO$^\ast$ \cite{ye2026kmhbo}
& 20 & 0.767 & 0.000 & 0.0233 & 0.638 & 0.0787 \\

RFE \cite{kohavi1997rfe}
& 20 & 0.100 & 0.000 & 0.0225 & 0.629 & 0.0834 \\

\textbf{DePICT}
& 20
& \textbf{0.000}
& \textbf{0.000}
& \textbf{0.0054}
& 0.550
& \textbf{0.0468} \\

\bottomrule
\end{tabular}
}

\end{table*}

This construction tests the paper's central claim more directly than downstream
regret alone. A selector may perform well by retaining statistically associated
or predictive variables, but the Layer-2 diagnostic asks a more specific
question: can it distinguish variables that enter the active optimization
problem from those that actually move the optimizer? We evaluate this
distinction first on controlled knapsack, where the irreducible interface is
known exactly, and then on FF25 portfolio allocation, where we additionally
measure its downstream effect on regret, Sharpe ratio, and turnover.

Table~\ref{tab:knapsack_layer2}(a) shows the controlled-knapsack separation.
DePICT recovers every irreducible coordinate while assigning zero false
inclusion to both Layer-1 and Layer-2 nuisance variables, yielding support
F1 \(=1.000\). Competing selectors often retain variables that are active
symbolically but irrelevant to the optimizer: Layer-2 FPR is \(0.700\) for
mutual information, \(0.500\) for FD ranking and PFMM, and \(0.650\) for
KMHBO. This distinction is reflected downstream. Under the linear predictor,
DePICT achieves regret \(0.009\), compared with \(0.475\) for PFMM, the
strongest non-DePICT baseline; under the MLP, DePICT remains lowest at
\(0.058\).

Table~\ref{tab:knapsack_layer2}(b) shows the same diagnostic on FF25.
The separation is again concentrated in the Layer-2 variables: DePICT has
zero Layer-2 and Layer-1 false inclusion, whereas Layer-2 FPR reaches
\(1.000\) for mutual information and PFMM, \(0.967\) for LASSO,
\(0.767\) for KMHBO, and \(0.100\) for RFE. DePICT also obtains the lowest
regret, \(0.0054\), compared with \(0.0127\) for the full interface and
\(0.0213\) for the best reduced-interface baseline, while reducing turnover
to \(0.0468\). The portfolio metrics also expose a trade-off: DePICT does not
maximize Sharpe, attaining \(0.550\) compared with \(0.648\) for mutual
information and PFMM. Thus, the FF25 result does not indicate uniform
improvement across every downstream metric; rather, it shows that excluding
Layer-2 variables can produce a substantially smaller-regret, lower-turnover
portfolio interface while preserving a clearly measurable portfolio solution.

Together, the two diagnostics support the same structural distinction in
different settings. The controlled benchmark verifies exact recovery of the
known optimizer-relevant interface, while FF25 shows that the same
Layer-2 filtering behavior remains consequential in a solver-mediated
portfolio task.

Figure~\ref{fig:controlled_seed_variance_app} further shows that the
Layer-2 and controlled-knapsack conclusions persist across 10 independent
seeds.

\begin{figure*}[t]
\centering
\includegraphics[width=0.90\textwidth]{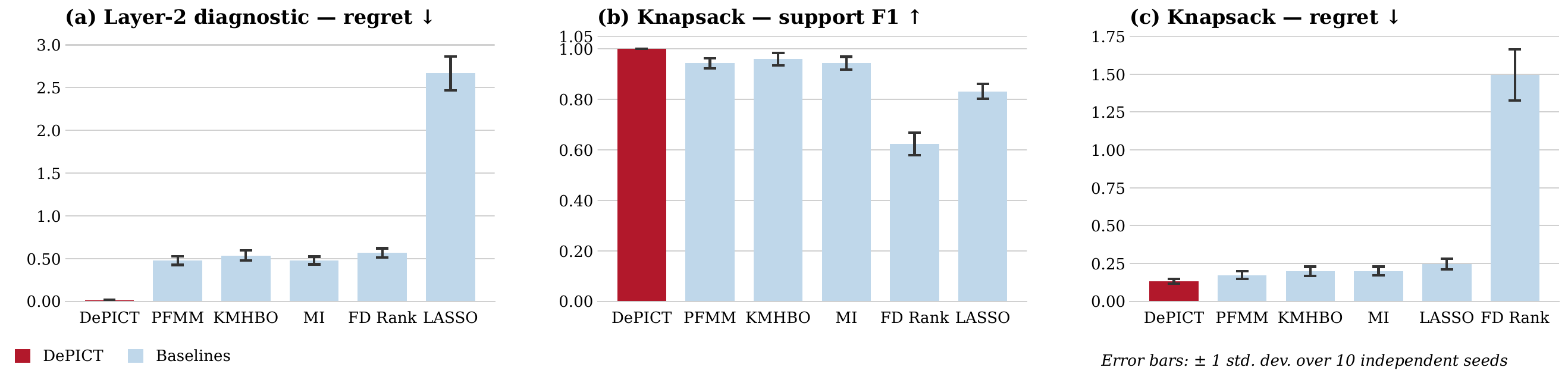}
\vspace{-2mm}
\caption{
\textbf{Seed variability for the controlled diagnostic.}
Bars report means and error bars report \(\pm 1\) standard deviation over
10 independent seeds. DePICT maintains near-zero Layer-2 diagnostic
regret, exact or near-exact knapsack support recovery, and the lowest
knapsack regret under the linear predictor. Lower regret is better; higher
support F1 is better.
}
\label{fig:controlled_seed_variance_app}
\vspace{-3mm}
\end{figure*}

\subsubsection{Stability perturbation protocol}
\label{app:stability_protocol}

The stability experiment tests whether DePICT rankings are robust to small changes in the operating context. We fix a nominal context    
\(c_0\) representative of the controlled knapsack operating regime and sample perturbed contexts
\[
c_\epsilon = c_0+\delta,
\qquad
\delta\sim \mathrm{Unif}\{\delta:\|\delta\|_\infty\le \epsilon\}.
\]
For each \(c_\epsilon\), we recompute the Jacobian scores and select the top-\(m\) interface. We compare this perturbed interface with the nominal top-\(m\) interface using exact match, Jaccard overlap, and regret inflation.

Exact match is a strict set-identity metric:
\[
\mathrm{Match}
=
\mathbf{1}\{I_\epsilon=I_0\}.
\]
It can decrease even when only one coordinate changes. Jaccard overlap is less brittle:
\[
\mathrm{Jac}(I_\epsilon,I_0)
=
\frac{|I_\epsilon\cap I_0|}{|I_\epsilon\cup I_0|}.
\]
It measures whether the perturbed interface retains most of the same directions as the nominal interface. Regret inflation measures the decision impact of the perturbation:
\[
\Delta\mathrm{Regret}
=
\mathrm{Regret}(I_\epsilon)-\mathrm{Regret}(I_0).
\]
Negative values mean the perturbed interface performs no worse than the nominal interface on the downstream evaluation.

This experiment connects directly to Proposition~\ref{prop:stability-main}. The proposition guarantees exact top-\(m\) stability only inside a radius controlled by the local spectral gap. Once perturbations approach directions where Jacobian column norms are close, exact set identity can fail. The empirical question is whether these failures materially affect downstream decisions. The results show that they usually do not: even when exact match decreases, Jaccard overlap remains high and regret inflation stays near zero. This indicates that DePICT's ranking is decision-stable under local context perturbations, even when exact top-\(m\) identity is too strict a requirement.

\begin{figure*}[t]
\centering
\setlength{\tabcolsep}{2.0pt}
\renewcommand{\arraystretch}{1.00}
\scriptsize

\begin{minipage}[t]{0.62\textwidth}
\centering
\vspace{0pt}
\textbf{(a) Stability metrics under perturbations}

\vspace{2pt}
\resizebox{\linewidth}{!}{
\begin{tabular}{lccc ccc ccc}
\toprule
& \multicolumn{3}{c}{$m=20$}
& \multicolumn{3}{c}{$m=40$}
& \multicolumn{3}{c}{$m=60$} \\
\cmidrule(lr){2-4} \cmidrule(lr){5-7} \cmidrule(lr){8-10}
$\epsilon$
& Match $\uparrow$ & Jac. $\uparrow$ & Reg. $\downarrow$
& Match $\uparrow$ & Jac. $\uparrow$ & Reg. $\downarrow$
& Match $\uparrow$ & Jac. $\uparrow$ & Reg. $\downarrow$ \\
\midrule
$10^{-4}$          
& .739 & .934 & .018  
& .653 & .963 & .002  
& .797 & .989 & .001 \\

$10^{-3}$          
& .665 & .913 & .027  
& .578 & .955 & $-.003$ 
& .750 & .987 & .000 \\

$10^{-2}$          
& .577 & .877 & .006  
& .472 & .946 & .002  
& .706 & .985 & $-.001$ \\

$5{\times}10^{-2}$ 
& .456 & .828 & .006  
& .251 & .925 & .001  
& .561 & .979 & $-.002$ \\
\bottomrule
\end{tabular}
}

\vspace{2pt}
{\footnotesize Match/Jac. higher is better; Reg. lower is better.}
\end{minipage}
\hfill
\begin{minipage}[t]{0.30\textwidth}
\centering
\vspace{0pt}
\textbf{(b) Stability profile}

\vspace{2pt}
\includegraphics[
  width=\linewidth,
  height=3.25cm,
  keepaspectratio
]{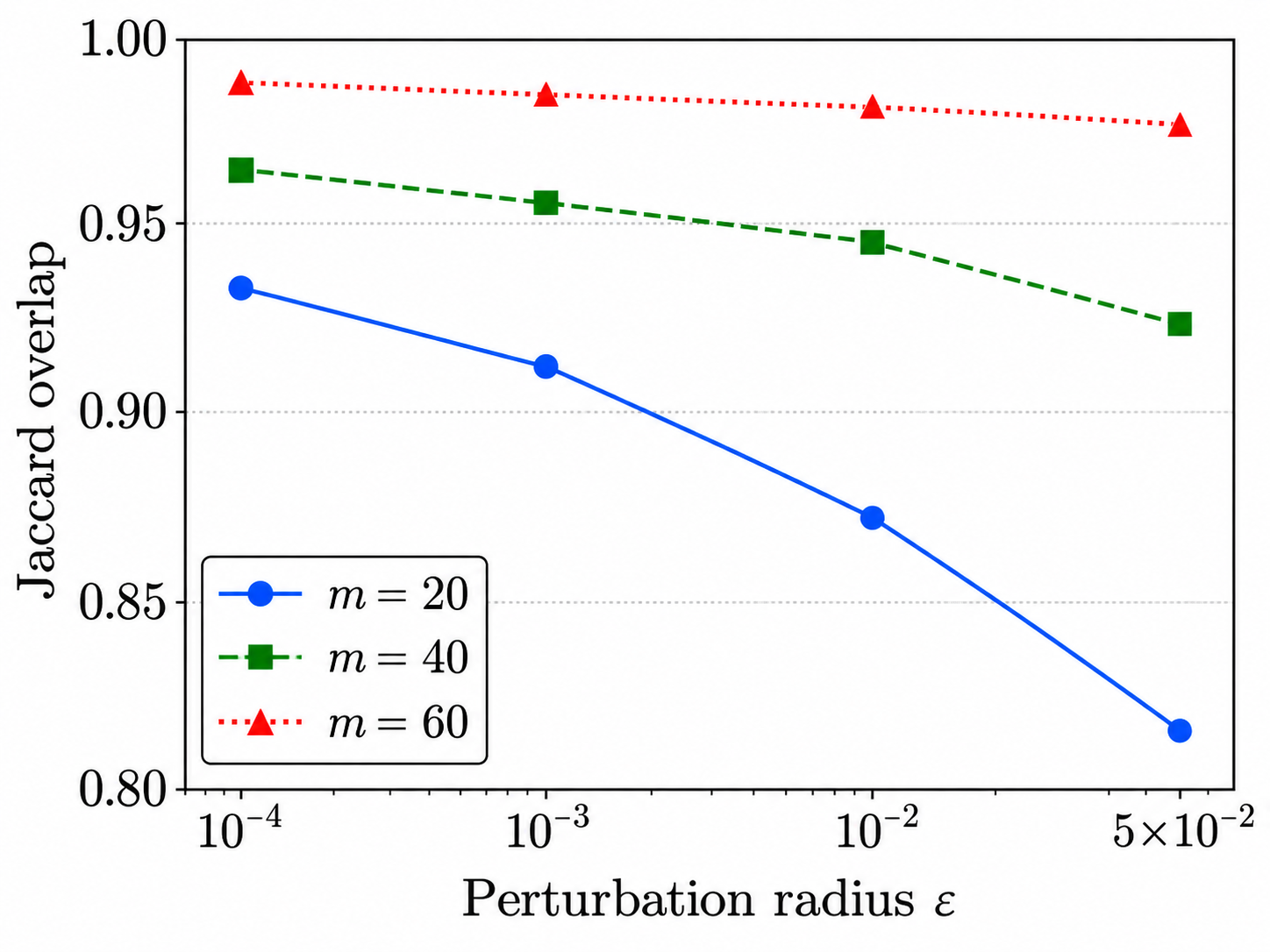}
\end{minipage}

\vspace{-2pt}
\caption{
Jacobian-ranking stability under context perturbations on controlled knapsack. Exact top-\(m\) identity decreases with \(\epsilon\), but Jaccard overlap remains high and regret inflation stays near zero, supporting the spectral-gap view that high-overlap interfaces can remain decision-stable even when exact rankings change.
}
\label{app:knapsack_stability_ablation}
\vspace{-.4cm}
\end{figure*}

\subsection{Discussion: Interface Selection, Adaptation, and Unlearning}
\label{sec:discussion_unlearning}

DePICT identifies which primitive context directions are structurally necessary for preserving a downstream constrained decision, but the same perspective also suggests a useful direction for model editing and unlearning. Recent work on personalized adaptation, multimodal unlearning, and layer-local editing shows that model behavior can often be changed by targeting the components most responsible for a task-specific response \cite{ranjan2026persareinforcementlearningprofessorstyle,ranjan2026razor}. In this view, DePICT can act as a decision-level diagnostic before or after unlearning: if removing or editing a representation changes only nuisance directions while preserving the DePICT interface, the downstream decision should remain stable; if an irreducible direction is removed, the task should exhibit structural degradation. This also connects to privacy and membership auditing, where gradient- or feature-drift signals can indicate whether a model still retains evidence of trained examples \cite{ranjan2026g, ranjan2026listening}. More broadly, combining decision-preserving interfaces with retrieval- or structure-aware mitigation methods may help separate task-relevant information from biased, unsafe, or unnecessary context \cite{ranjan2026catrag,ranjan2026position}. For embodied and edge-deployed systems, where context bandwidth, safety, and adaptation constraints are central \cite{grover2026embodied}, future work can test whether targeted unlearning on downstream tasks removes undesirable behavior while DePICT verifies whether the model still preserves the necessary decision-relevant interface, thereby helping determine whether the system has truly forgotten or merely hidden the trained behavior.

%\input{Sections/Tables}
%\input{Sections/Appendix}

% ICLR 2027 requires an AI use statement before submission. It is intentionally
% not added here because this conversion preserves the manuscript's content.

\end{document}